%% file: paper.tex
\documentclass[10pt]{article}
\usepackage[margin=0.9in]{geometry}
\usepackage[T1]{fontenc}
\usepackage[utf8]{inputenc}
\usepackage[numbers,sort&compress]{natbib}
\usepackage{caption}
\usepackage{graphicx}
\usepackage{amsmath}
\usepackage{xcolor}
\usepackage{titletoc}

\usepackage{amssymb,mathtools,bm,microtype,booktabs,enumitem}
\usepackage{amsthm}
\usepackage{flafter}
\usepackage{placeins}

\newcommand{\E}{\mathbb E}
\newcommand{\R}{\mathbb R}
\newcommand{\N}{\mathcal N}
\newcommand{\dd}{\mathrm d}
\newcommand{\iid}{\mathrel{\overset{\mathrm{iid}}{\sim}}}
\newcommand{\dto}{\xrightarrow[d\to\infty]{}}
\newcommand{\one}{\mathbf 1}
\newcommand{\id}{\mathbf I}
\newcommand{\sig}{\sigma}
\newcommand{\calD}{\mathcal D}
\newcommand{\calZ}{\mathcal Z}
\newcommand{\bx}{\bm x}
\newcommand{\bw}{\bm w}
\newcommand{\bu}{\bm u}
\newcommand{\bz}{\bm z}
\newcommand{\br}{\bm r}
\newcommand{\bchi}{\bm\chi}
\newcommand{\bv}{\bm v}
\newcommand{\bdelta}{\bm\delta}
\newcommand{\bfield}{\bm f}
\newcommand{\bkappa}{\bm\kappa}
\newcommand{\laR}{\lambda_R}
\newcommand{\Cov}{\operatorname{Cov}}
\newcommand{\Var}{\operatorname{Var}}
\newcommand{\rank}{\operatorname{rank}}
\newcommand{\vis}{\mathrm{vis}}
\newcommand{\mask}{\mathrm{mask}}
\newcommand{\out}{\mathrm{out}}
\newcommand{\Extr}{\mathop{\mathrm{extr}}}
\newcommand{\argmin}{\operatorname*{arg\,min}}

\definecolor{dynamicColor}{HTML}{1F4E8C}
\definecolor{staticColor}{HTML}{9A7300}
\definecolor{augmentationColor}{HTML}{6F4A8E}
\definecolor{fullrecColor}{HTML}{9E2F2F}
\definecolor{bayesColor}{HTML}{2E6E3F}

\newif\ifconceptcolors
\conceptcolorstrue

\newcommand{\dynamicmask}[1]{\ifconceptcolors\textcolor{dynamicColor}{#1}\else#1\fi}
\newcommand{\staticmask}[1]{\ifconceptcolors\textcolor{staticColor}{#1}\else#1\fi}
\newcommand{\standardaug}[1]{\ifconceptcolors\textcolor{augmentationColor}{#1}\else#1\fi}
\newcommand{\fullrecon}[1]{\ifconceptcolors\textcolor{fullrecColor}{#1}\else#1\fi}
\newcommand{\bayesopt}[1]{\ifconceptcolors\textcolor{bayesColor}{#1}\else#1\fi}

\usepackage{hyperref}
\definecolor{referenceColor}{HTML}{405A73}
\hypersetup{
  colorlinks=true,
  allcolors=referenceColor
}
\usepackage{url}
\usepackage[nameinlink,capitalise]{cleveref}
\newif\ifabbreviatedreferences
\abbreviatedreferencestrue
\ifabbreviatedreferences
  \crefname{figure}{Fig.}{Figs.}
  \Crefname{figure}{Fig.}{Figs.}
  \crefname{equation}{Eq.}{Eqs.}
  \Crefname{equation}{Eq.}{Eqs.}
  \crefname{section}{Sec.}{Secs.}
  \Crefname{section}{Sec.}{Secs.}
  \crefname{appendix}{App.}{Apps.}
  \Crefname{appendix}{App.}{Apps.}
\else
  \crefname{figure}{Figure}{Figures}
  \Crefname{figure}{Figure}{Figures}
  \crefname{equation}{Equation}{Equations}
  \Crefname{equation}{Equation}{Equations}
  \crefname{section}{Section}{Sections}
  \Crefname{section}{Section}{Sections}
  \crefname{appendix}{Appendix}{Appendices}
  \Crefname{appendix}{Appendix}{Appendices}
\fi

\AddToHook{cmd/appendix/after}{%
  \crefalias{section}{appendix}%
  \crefalias{subsection}{appendix}%
  \crefalias{subsubsection}{appendix}%
  \crefalias{subappendix}{appendix}%
  \crefalias{subsubappendix}{appendix}%
}
\newtheorem{theorem}{Theorem}[section]
\newtheorem{proposition}[theorem]{Proposition}
\newtheorem{lemma}[theorem]{Lemma}

\newtheorem{assumption}[theorem]{Assumption}

\title{The Hidden Advantage of Mask Resampling:\\
A Theory of Masked Autoencoders}
\input{authors}
\date{}

\begin{document}

\setcounter{footnote}{1}
\maketitle

\begin{abstract}

Why can masked prediction learn useful representations that unmasked reconstruction misses? We study this question in a high-dimensional model of a masked autoencoder (MAE) trained on data with shared latent structure and heterogeneous noise. We prove that masked linear reconstruction can recover the latent feature at linear sample complexity in regimes where unmasked linear reconstruction, equivalent to PCA, fails. The analysis also quantifies the statistical advantage of mask resampling, an established ingredient of masked pretraining. By introducing a fixed collection of $K$ masks per sample, we characterize its effect on feature recovery and downstream performance, identifying regimes where greater mask diversity lowers sample complexity. Guided by this prediction, we find that random cropping and flipping in standard image-training pipelines can obscure the advantage of mask resampling by renewing the prediction task even when the patch mask is fixed. Removing these transformations reveals a downstream advantage for dynamic over static masking in CNN autoencoders and vision transformers. A complementary BERT pilot finds benefits from greater mask diversity on downstream language tasks. Our results separate the benefit of the masked prediction objective from that of mask diversity, and show how a tractable theory can guide experiments that uncover advantages hidden by standard training practices.
\end{abstract}
\begin{center}
\small
Code available at:
\url{https://github.com/SPOC-group/mask-resampling}
\end{center}

\section{Introduction}
\label{sec:introduction}
Why can predicting missing parts of an input produce more useful representations than reconstructing the entire input? Building on denoising autoencoders \citep{vincent2010stacked}, \citet{pathak2016context} predicted missing image regions from visible context and obtained more transferable representations than a standard autoencoder trained without input masking, a paradigm that now underlies masked language and image modeling \citep{devlin2019bert,he2022masked,bao2021beit,xie2022simmim}. Masked learning therefore couples two ingredients: a \textit{prediction objective} that specifies what structure is rewarded, and an \textit{augmentation procedure} that controls how many distinct prediction tasks a given sample ever presents. Despite its empirical success, our theoretical understanding of how these ingredients shape learned representations remains incomplete. We develop a high-dimensional theory that separates their effects and connects them to latent-feature recovery and downstream performance.

Motivated by the observation that unmasked, i.e., the standard autoencoder, reconstruction \footnote{By unmasked reconstruction, we mean the standard autoencoder objective that observes the entire input and minimizes squared reconstruction error over all coordinates.}
can favor high-variance nuisance directions over downstream-relevant structure \citep{balestriero2024learning}, we consider a data model with a shared latent direction and heterogeneous coordinate-wise noise. The latent direction carries the downstream signal, while heterogeneity allows nuisance variance to dominate the leading principal component.
Masking changes this competition by requiring hidden coordinates to be predicted from the visible ones, rather than rewarded for their own variance. We prove that a linear masked autoencoder recovers the latent direction with $n=O(d)$ samples where \fullrecon{unmasked linear reconstruction}, equivalent to PCA, fails (\Cref{sec:main-separation}).

\begin{figure}[!htbp]
\centering

\includegraphics[width=\textwidth]{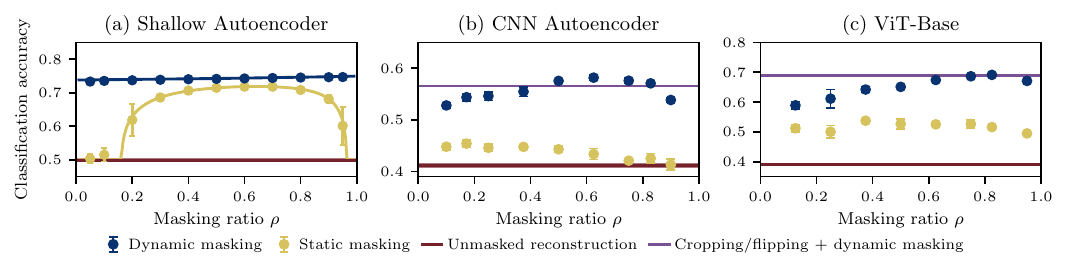}

\caption{\textbf{Isolating mask resampling reveals its contribution to representation learning.} Downstream accuracy versus masking ratio $\rho$ for \dynamicmask{dynamic masking} (blue), \staticmask{static masking} (yellow), and \fullrecon{unmasked reconstruction} (red). \textbf{(a)} Shallow autoencoder on spiked model (\Cref{sec:main-model}); curves are theory and markers are simulations. \textbf{(b)} CNN autoencoder on CIFAR-10 (test top-1). \textbf{(c)} ViT-Base on ImageNet-100 (validation top-5). In (b,c), blue and yellow curves use no stochastic cropping or flipping. The purple reference includes cropping and flipping with \emph{dynamic mask resampling} at $\rho=0.625$ \textbf{(b)} and $\rho=0.75$ \textbf{(c)}. At suitable masking ratios, resampling alone reaches comparable accuracy. Bars and shading show $\pm1$ standard deviation (SD). See \Cref{app:experimental-reporting} for details.}
\label{fig:static-vs-dynamic}

\end{figure}

Analyzing this model clearly reveals a second notable effect: the number of distinct masks available per sample controls how many underlying training samples are needed for latent-feature recovery. We distinguish \staticmask{static masking}, where each sample is assigned a single random mask that is reused throughout training, from \dynamicmask{dynamic masking}, commonly used in practice \citep{he2022masked,liu2019roberta}, in which the mask is independently resampled each time a sample is used during pretraining. More generally, assigning each sample $K$ independent masks and reusing them thereafter interpolates between these two regimes at fixed number of underlying samples. This allows us to isolate mask diversity from data diversity. Our analysis predicts that increasing $K$ lowers the recovery threshold, with corresponding improvements in downstream performance in the settings studied.

This prediction prompted us to revisit how the contribution of mask diversity is assessed empirically. In language modeling, dynamic masking has been reported to provide only modest gains over reusing a fixed collection of masks \citep{liu2019roberta}, but their static baseline already uses $K=10$ pre-generated masks per sequence rather than a single fixed mask. Data2Vec 2.0 shows substantial improvements in image classification accuracy by increasing the number of masked views per sample within each training update \citep{baevski2023efficient}. Furthermore, in vision, self-supervised image learning commonly uses random cropping, flipping, and color transformations to generate changing views of the same sample \citep{chen2020simple,tian2020makes}. In MAE specifically, standard pretraining combines random cropping and flipping with mask resampling \citep{he2022masked}. Under a newly drawn crop, for example, the same patch mask can hide different content from the original image. We therefore hypothesize that stochastic image augmentation adds another source of prediction-task diversity, obscuring the effect of resampling the mask itself.

Guided by this hypothesis, we compare static and dynamic masking \standardaug{with} and \dynamicmask{without} \standardaug{stochastic cropping and flipping}. Without these transformations, \dynamicmask{dynamic masking} improves downstream accuracy in CNN autoencoders on CIFAR-10 and ViT-Base on ImageNet-100 (\Cref{fig:static-vs-dynamic} (b,c)). In the ViT experiments, reintroducing cropping and flipping substantially narrows the observed static–dynamic gap, mainly because static masking improves (\Cref{fig:masking-augmentation}).
Finite-$K$ experiments further show progressive improvements as each image is exposed to more masks (\Cref{fig:K-interpolation}). A complementary BERT-Medium pilot finds improvements with increasing mask diversity on GLUE tasks, extending the qualitative evidence beyond vision (\Cref{sec:main-experiments}).

Together, these results distinguish two benefits of masked representation learning: the prediction objective can favor useful correlations over nuisance variance, while renewing the masks can lower the sample complexity required to learn those correlations. By identifying this distinction, the theory motivates a controlled test of its relevance beyond the solvable model.

The \textbf{main contributions} of this paper are the following:
\begin{itemize}
\item Under suitable heterogeneous noise, we prove that a linear masked autoencoder recovers the latent direction for all sufficiently large finite sample complexity $\alpha=n/d$, whereas unmasked linear reconstruction (PCA) fails for every fixed $\alpha>0$. Replica calculations complement this guarantee by predicting the recovery threshold (\Cref{sec:main-separation}).

\item We formulate a tractable high-dimensional model of a two-stage learning protocol: self-supervised learning (SSL) through masked reconstruction, followed by supervised linear probing with a frozen encoder (\Cref{sec:main-model}).
For linear and nonlinear autoencoders, we derive replica predictions for feature recovery and downstream accuracy as functions of the masking ratio and mask diversity, showing the benefit of resampling the masks (\Cref{sec:main-recovery}).

\item Guided by our theoretical predictions, we test the role of mask diversity in vision and language. Dynamic masking improves downstream performance without stochastic cropping and flipping; adding these transformations substantially narrows the observed static–dynamic gap in the ViT setting, consistent with overlapping sources of prediction-task diversity. A BERT-Medium pilot provides complementary evidence for the benefit of mask diversity in language modeling (\Cref{sec:main-experiments}).
\end{itemize}

\subsection*{Further related work}

We outline the broad lines of related work here; \Cref{app:positioning} provides an expanded comparison with the closest references and clarifies the differences in setting, analytical scope, and contribution.

\paragraph{Statistical-physics theory of neural networks.} Our replica analysis builds on high-dimensional empirical risk minimization \citep{aubin2020generalization,loureirogmm2021,cui2025high}, shallow-autoencoder dynamics \citep{refinetti2022dynamics}, and denoising autoencoders \citep{cui2023high}. \citet{conde2026solvable} prove an advantage of nonlinear autoencoders over PCA. Heterogeneous noise has also been studied in supervised learning \citep{adomaityte2023classification,adomaityte2024robust}.

\paragraph{Theory of masked learning.} Existing work connects masked autoencoding to contrastive learning and downstream guarantees \citep{zhang2022mask}, and establishes latent identifiability in hierarchical models \citep{kong2023understanding}. \citet{bisulco2025linearity} characterize linear MAEs with full reconstruction from masked inputs, including visible coordinates rather than only hidden targets. Their analysis relates patch size and masking ratio to the spatial correlations learned. \citet{dong2026understanding} analyze decoupled coordinate predictors through rigorous ridgeless-regression risk formulas, accounting for both corrupted covariates and the samples supplying a selected prediction target. Single-coordinate prediction also connects to factored attention \citep{rende2024mapping}. \citet{wortsman2026a} characterize a spiked-model recovery transition in the joint matrix of single-coordinate ridge predictors and establish prediction-risk advantages over PCA in structured covariance settings. \citet{pmlr-v202-cabannes23a} distinguish sample variability from augmentation variability in SSL generalization bounds; additional views can reduce the latter without removing the former.

\paragraph{Spiked models and heteroskedastic noise.} Our linear baseline connects to spiked covariance models \citep{johnstone2001distribution} and PCA recovery transitions \citep{baik2005bbp,benaych2012singular}. \citet{hong2018asymptotic} characterize PCA under heteroskedastic noise. Spectral methods addressing heteroskedasticity include weighted PCA \citep{hong2023optimally}, diagonal correction through HeteroPCA \citep{zhang2022heteroskedastic}, and optimal estimators for doubly heteroskedastic noise \citep{ZhangMondelli2024}. We study how masking changes the reconstruction objective to promote recovery, without supplying the coordinate variances or their distribution.

\section{A high-dimensional model of masked learning}
\label{sec:main-model}

\paragraph{Data and noise heterogeneity.} Our goal is to compare prediction objectives and masking protocols for a data distribution that enables sharp high-dimensional analysis. Inspired by insights into the advantage of masking from \cite{balestriero2024learning} and by work on unsupervised learning in high dimensions \citep{johnstone2001distribution,deshpande2016sparse,lesieur2017constrained}, we consider a data model that separates a shared informative feature in the form of a spike from coordinate-specific nuisance fluctuations. Concretely, we consider $n$ samples with $d$ coordinates generated as
\begin{equation}
    x_{\mu i} = \sqrt{\frac{\beta}{d}}\,\lambda_\mu u_i^\star + \sqrt{\gamma_i^\star}\,z_{\mu i}, \qquad \mu\in[n],\quad i\in[d].
    \label{eq:main-data-model}
\end{equation}
Here $\bu^\star$ is the shared latent direction (spike), $\lambda_\mu$ is the latent variable of sample $\mu$, and $\beta>0$ controls signal strength. The variance $\gamma_i^\star$ sets the strength of nuisance fluctuations in coordinate $i$: a large value increases input variance independently of the spiked direction. These variances are drawn once and shared by all samples, whereas the noise realizations $z_{\mu i}$ change across samples and coordinates.

The variables $\lambda_\mu$ and $u_i^\star$ are independently drawn from centered, unit-second-moment distributions $P_\lambda$ and $P_u$; $z_{\mu i}\iid\N(0,1)$, and the variances $\gamma_i^\star\iid P_\gamma$ are positive. We normalize $\E_\gamma[\gamma]=1$ and write $\kappa_2=\E_\gamma[\gamma^2]<\infty$; hence $\kappa_2\geq1$, with equality only for homogeneous noise. All random families are mutually independent. The case $P_\gamma=\delta_1$ is the standard spiked covariance model \citep{johnstone2001distribution}. Full assumptions and notation are in \Cref{sec:finite-K-model,app:notation}.

\paragraph{Architecture and masked prediction.} The learner is a single-hidden-unit autoencoder with tied encoder and decoder weights $\bw\in\R^d$. For each sample $\mu$, we draw $K$ mask vectors $\boldsymbol{\chi}_\mu^{(r)} =(\chi_{\mu i}^{(r)})_{i\in[d]}\in\{0,1\}^d$, with entries drawn independently from $\mathrm{Bernoulli}(\rho)$ and independently of the data. The mask collection $\Xi=(\boldsymbol{\chi}_\mu^{(r)})_{\mu\in[n],\,r\in[K]}$ is held fixed throughout pretraining. An entry of one marks a hidden coordinate; so $\rho$ is the \textit{expected masking ratio}. Mask $r$ of sample $\mu$ encodes only visible coordinates through
\begin{equation}
    h_{\mu r,\mathrm{vis}}(\bw) = \frac{1}{\sqrt d}\, \bw^\top\! \left[ \bigl(\mathbf{1}_d-\boldsymbol{\chi}_\mu^{(r)}\bigr) \odot \mathbf{x}_\mu \right],
    \label{eq:main-preactivation}
\end{equation}
where $\odot$ denotes the elementwise (Hadamard) product and $\mathbf{1}_d$ is the all-ones vector. For the fixed mask collection $\Xi$, pretraining minimizes the regularized empirical masked risk
\begin{equation}
    \mathcal R_{K,d}(\bw) = \frac{\laR}{2}\|\bw\|_2^2 + \frac{1}{K}\sum_{\mu=1}^n\sum_{r=1}^K \left\| \boldsymbol{\chi}_\mu^{(r)} \odot \left[ \mathbf{x}_\mu - \frac{\bw}{\sqrt d}\, \sig\!\left(h_{\mu r,\mathrm{vis}}(\bw)\right) \right] \right\|_2^2,
    \label{eq:main-finite-K-objective}
\end{equation}
where $\sig$ is the activation function and $\laR$ controls ridge regularization. The loss is evaluated only on hidden coordinates, so each target is predicted from the visible input. \fullrecon{Unmasked reconstruction} instead observes and predicts every coordinate; we define this separate objective in \Cref{sec:main-separation}.

\paragraph{Static, finite-$K$, and dynamic masking.} The choice $K=1$ gives \staticmask{static masking}. Increasing $K$ adds masked views without adding underlying samples. At a fixed dataset, the average over masks converges to the mask expectation as $K\to\infty$. We refer to this loss as the \dynamicmask{dynamic masking} objective; drawing a fresh mask at each encounter provides stochastic evaluations of it. The high-dimensional analysis takes $n,d\to\infty$ at fixed $K$ and obtains dynamic masking as $K\rightarrow \infty$ (\Cref{sec:dynamic-limit}).

Our \textit{two-stage protocol} first learns a representation from unlabeled data and then fits a supervised probe with the encoder frozen. We assess both whether pretraining recovers the shared latent direction and whether the resulting representation supports prediction from a small, independent labeled dataset.

\paragraph{Pretraining and feature recovery.} Given the pretraining dataset $\mathcal D_{\rm pre}=\{\bx_\mu\}_{\mu=1}^{n}$ with $n=\alpha d$ and the fixed mask collection $\Xi$, we study a global minimizer $\widehat\bw\in\operatorname*{arg\,min}_{\bw\in\R^d} \mathcal R_{K,d}(\bw)$ of the risk in \eqref{eq:main-finite-K-objective}. In the proportional limit $n,d\to\infty$ at fixed sample complexity $\alpha=n/d$, recovery is measured by the cosine similarity (alignment):
\begin{equation}
    \theta_{\bu^\star} = \frac{\lvert\langle\widehat\bw,\bu^\star\rangle\rvert} {\lVert\widehat\bw\rVert_2\,\lVert\bu^\star\rVert_2}
    \label{eq:main-cosine}
\end{equation}
Weak recovery means that there exists $\eta>0$, independent of $n$ and $d$ such that $\theta_{\bu^\star}>\eta$ has probability that tends to one as $n,d\to \infty$. A recovery threshold marks the onset of this alignment.

\paragraph{Downstream task and linear probing.} After pretraining, we freeze $\widehat\bw$ and draw a labeled dataset $\mathcal D_{\rm ds} =\{(\bx_\nu^{\rm ds},y_\nu^{\rm ds})\}_{\nu=1}^{n_{\rm ds}}$, where $y_\nu^{\rm ds} =\operatorname{sign}(\lambda_\nu^{\rm ds})$. These samples follow \eqref{eq:main-data-model} with the same latent direction and coordinate variances as $\mathcal D_{\rm pre}$, but independently drawn latent coefficients and noise. We fit a one-dimensional logistic classifier on the unmasked scalar features $d^{-1/2}{\widehat\bw}^\top\bx_\nu^{\rm ds}$ and evaluate it on independent test samples from the same model. Thus pretraining uses $n=\alpha d$ unlabeled samples, while the labeled probe size $n_{\rm ds}$ remains fixed as $d$ grows (\Cref{subsec:downstream-classification,app:downstream-label-count}).

\paragraph{Analytical predictions and optimality.} The replica method predicts feature recovery after pretraining and the accuracy of the subsequent frozen-representation probe (\Cref{sec:replica-derivation,sec:model-specializations}). Replica calculations predict weak recovery thresholds by locating where an uninformative solution becomes unstable to a small alignment. These predictions concern continuous recovery transitions and are validated with simulations; an earlier discontinuous transition requires a separate analysis (\Cref{subsec:recovery-threshold-scope}). Below, ``weak recovery threshold'' has this predictive meaning when referring to replica results. Independently, \Cref{sec:main-separation} proves a separation between masked linear autoencoder and PCA.

We also compare feature recovery with a matched \bayesopt{Bayes-optimal} reference that uses the same unlabeled samples but is supplied with the latent priors and realized coordinate variances. Unmasked reconstruction and finite-$K$ comparisons isolate the effects of the objective and mask reuse, respectively; the Bayes reference instead assesses the gap to matched inference with additional prior information. Following \citet{lesieur2017constrained}, \Cref{app:bayes-optimal} gives the state evolution and heteroskedastic approximate message passing (AMP) updates. Downstream evaluation uses the same scalar logistic probe for AMP and the autoencoders (\Cref{subsec:downstream-classification}).

\section{A provable advantage of masked reconstruction}
\label{sec:main-separation}

We first isolate the effect of the prediction objective. Unmasked autoencoders learn a compressed representation by reconstructing the full input through a low-dimensional bottleneck. The bottleneck forces the learner to select which variations in the data to preserve, with reconstruction error providing the criterion for this selection. Unmasked reconstruction is therefore a natural baseline for asking whether the features most useful for reconstructing the input also capture its shared latent structure.

\paragraph{Unmasked reconstruction.} The autoencoder observes and predicts every coordinate, minimizing
\begin{equation}
    \mathcal R_{\rm full,d}(\bw) =\sum_{\mu=1}^n \left\| \mathbf{x}_\mu - \frac{\bw}{\sqrt d}\, \sig\bigl(h_\mu^{\rm full}(\bw)\bigr)\right\|^2_2, \qquad h_\mu^{\rm full}(\bw)=\frac{\bw^\top\bx_\mu}{\sqrt d}.
    \label{eq:main-full-objective}
\end{equation}
For $\sig(h)=h$, the global minimizers of this tied single-unit autoencoder point along a direction of largest empirical variance, making it equivalent to ordinary PCA \citep{bourlard1988auto}; see \Cref{subsec:sep-pca} for the reduction.

\paragraph{What changes when the target is hidden?} For the masked objective \eqref{eq:main-finite-K-objective}, we likewise consider a linear activation, set $\laR=0$, and average over masks. Write $\bw=\sqrt{dQ}\,\br$, with $Q\geq0$ and $\|\br\|=1$, and let $\Sigma=\operatorname{diag}(\gamma_1^\star,\ldots,\gamma_d^\star) +\beta\bu^\star(\bu^\star)^\top/d$ be the population covariance conditional on the planted direction and coordinate variances. Both population losses per sample are quadratic in $Q$. Their terms linear in $Q$ are $-2Q$ times the following rewards:
\begin{equation}
\begin{aligned}
 \text{Unmasked:}\quad &\br^\top\Sigma\br \qquad \text{Masked:}\quad &\rho(1-\rho)\,\br^\top \bigl(\Sigma-\operatorname{diag}\Sigma\bigr)\br
\end{aligned}
\label{eq:main-masked-reward}
\end{equation}
Unmasked reconstruction rewards total variance, whereas masked prediction rewards cross-correlations between different coordinates. The diagonal is absent from the masked reward because masking prevents a coordinate from predicting itself, removing the reward for its own variance: only correlations with visible coordinates can help reconstruct it.

To illustrate this effect, consider a coordinate with large noise variance, preserving its fluctuations in the compressed representation can substantially reduce unmasked reconstruction error, even though those fluctuations carry no information about the shared structure (the spike $\bu^\star$ in our model). But when that coordinate is hidden, its independent noise cannot be predicted from the visible coordinates. The shared feature $\bu^\star$, by contrast, creates correlations between visible and hidden coordinates and can therefore improve masked prediction. Noise no longer earns a direct reward through its variance, and the quadratic prediction penalty discourages reliance on noisy inputs. Together, these effects can allow masked reconstruction to favor cleaner coordinates. The full loss decomposition is given in
\Cref{subsec:sep-exact-mask,subsec:sep-population-gap}.

\paragraph{Homogeneous noise provides a reference case.} When $P_\gamma=\delta_1$, noise contributes equally to the unmasked reward in every unit direction. For Gaussian latent priors, PCA recovers above $\alpha_c=\beta^{-2}$ \citep{baik2005bbp,benaych2012singular}, while the replica-symmetric prediction for linear dynamic masking gives the same threshold and asymptotic alignment (\Cref{subsec:sep-pca,subsec:dyn-bound-informative}). \Cref{fig:validation} (a) illustrates this agreement.

\paragraph{A separation at linear sample complexity.} We specialize the model of \Cref{sec:main-model} to $P_u=P_\lambda=\N(0,1)$ and a fixed, mean-one variance law supported in $[a,\infty)$ for some $a>0$, with finite second moment. We fix $\beta>0$ and $\rho\in(0,1)$ as $n,d\to\infty$ with $n/d\to\alpha$. The precise hypotheses, including additional edge conditions for bounded-noise PCA failure, are collected in Assumption~\ref{as:main-separation}.

\begin{theorem}[Recovery by masked reconstruction and separation from PCA]
\label{thm:main-separation}
Under \ref{as:model} and \ref{as:noise}, there exist finite $A_{\rm suff}(a,\beta,\rho)$ and $\eta(a,\beta,\rho)>0$ such that, for every fixed $\alpha>A_{\rm suff}$, with probability tending to one the linear dynamic masking objective, with exact averaging over masks, admits global minimizers and every such minimizer satisfies $\theta_{\bu^\star}^2\geq\eta$. If, additionally, the variance support is unbounded or the bounded-noise conditions \ref{as:pca-fail} hold, then unmasked linear reconstruction (PCA) has $\theta_{\bu^\star}^2\to0$ in probability for every fixed $\alpha>0$.
\end{theorem}
\begin{figure}[!htbp]
\centering

\includegraphics[width=\textwidth]{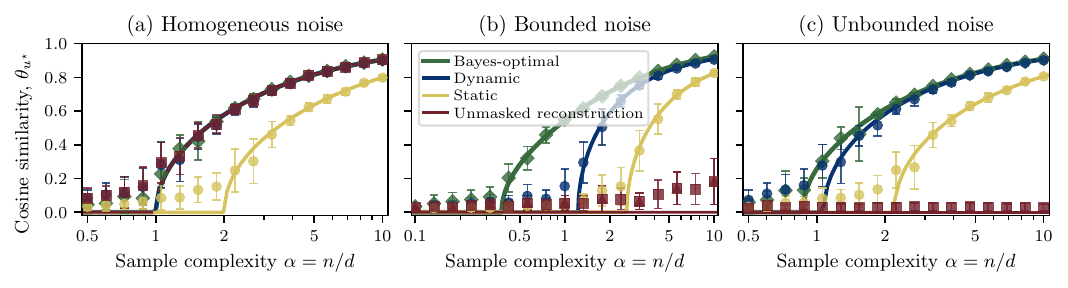}

\caption{\textbf{Masking recovers a feature that PCA misses.} Cosine similarity versus sample complexity $\alpha=n/d$, for \textbf{(a)} homogeneous noise, $\gamma=1$; \textbf{(b)} bounded heterogeneous noise, $\gamma=\frac14+\frac94 B$, $B\sim\operatorname{Beta}(1,2)$; and \textbf{(c)} unbounded heterogeneous noise, $\gamma=\frac34+\frac14 e^{Z-1/2}$, $Z\sim\N(0,1)$. Curves are theoretical predictions (linear masked autoencoders), markers are finite-dimensional simulations, and error bars show $\pm1$ SD across seeds. Green curves give the matched \bayesopt{Bayes-optimal} reference; green diamonds show AMP results (\Cref{subsec:BO-AMP,app:experimental-bayes}). \fullrecon{Unmasked reconstruction} is ordinary PCA. The homogeneous panel provides a control; the heterogeneous panels illustrate recovery by masking where the asymptotic PCA alignment vanishes. See \Cref{app:experimental-reporting} for details.}
\label{fig:validation}

\end{figure}

For bounded noise with upper endpoint $\gamma_+$, \ref{as:pca-fail} requires $\beta\,\mathbb E_\gamma[(\gamma_+-\gamma)^{-1}]<1$ and $\mathbb E_\gamma[ \gamma^2(\gamma_+-\gamma)^{-2}]=\infty$. The first condition excludes a separated population covariance outlier; the second ensures a regular upper edge of the noise-only sample-covariance spectrum at every fixed finite $\alpha$. $P_\gamma$ from \Cref{fig:validation}(b) satisfies these conditions.

\Cref{thm:main-separation} therefore shows that the same dataset contains a recoverable feature that its leading principal component misses: changing the prediction objective is enough to obtain positive alignment with $n=O(d)$ samples, without supplying coordinate variances to either learner. The proof first shows that the population masked objective favors an informative direction over all directions with sufficiently small overlap. It then controls empirical fluctuations uniformly over candidate directions, relative to the noise penalty, so that this preference survives at sufficiently large finite $\alpha$. The argument applies to every global minimizer of the objective (\Cref{app:linear-separation,app:dynamic-unbounded}).

\paragraph{Interpreting the separation.} \Cref{fig:validation} (b,c) illustrates how heterogeneity changes the comparison with the homogeneous reference: masking recovers while \fullrecon{PCA} can remain uninformative\footnote{\Cref{app:pca-numerics} shows that the positive cosine similarity of PCA in~\Cref{fig:validation}(b) is due to finite-size effects.}. The bounded example (panel b) shows that this effect does not require arbitrarily large coordinate variances; the signal can remain below the noise spectral edge even when the variance support is compact. With unbounded support (panel c), extreme coordinate variances instead drive the leading sample component away from the shared
feature. The respective PCA arguments are given in \Cref{subsec:sep-compact-pca,subsec:sep-unbounded}.
The masked reward in \eqref{eq:main-masked-reward} involves the same off-diagonal covariance structure used by diagonal-deletion PCA \citep{florescu2016spectral}, but the full masked objective also contains a noise-sensitive, direction-dependent prediction penalty. In our simulations, dynamic masking achieves better alignment than both diagonal-deletion PCA and HeteroPCA (\Cref{app:heteroPCAs}).

The theorem guarantees recovery for all sufficiently large finite $\alpha$ for dynamic masking, but does not locate the sharp transition, for which the replica analysis provides a prediction. Having established why the masked objective can favor useful structure, we next ask how its performance changes when each sample supplies only a finite collection of masks.
\FloatBarrier

\section{Mask diversity and sample complexity}
\label{sec:main-recovery}

We now keep the original dataset fixed and vary the number $K$ of masks per sample. Reusing a fixed collection of masks repeatedly presents the same visible-to-hidden prediction tasks; increasing $K$ adds new tasks without adding independent observations. The question is how this change affects the weak recovery thresholds and downstream performance.

\begin{figure}[!htbp]
\centering

\includegraphics[width=\textwidth]{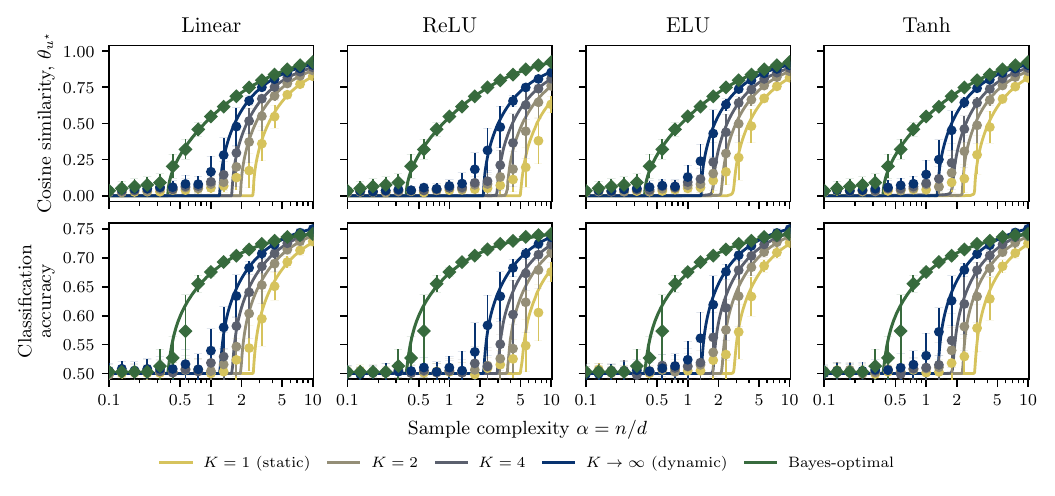}

\caption{\textbf{Replica predictions track learning across mask diversity.} Cosine similarity (top) and downstream classification accuracy (bottom) versus sample complexity $\alpha=n/d$ under bounded heterogeneous noise (variance distribution of \Cref{fig:validation}(b)), for linear, ReLU, ELU, and tanh activations. Colors distinguish $K=$\staticmask{$1$(static)}$,2,4$ and the \dynamicmask{dynamic masking} limit. Solid curves are replica predictions and markers are finite-dimensional simulations. Green curves show the \bayesopt{Bayes-optimal} pretraining reference, evaluated with the same logistic probe for the lower row; green diamonds show AMP results (\Cref{subsec:BO-AMP,app:experimental-bayes}). Error bars show $\pm1$ SD. See \Cref{app:experimental-reporting} for experimental details.}
\label{fig:mask-sample-efficiency}

\end{figure}

\paragraph{Predicting learning across mask diversity.} Our replica analysis yields high-dimensional learning curves for \staticmask{static}, finite-$K$, and \dynamicmask{dynamic masking}, resolving their dependence on sample complexity $\alpha$, masking ratio $\rho$, and activation function (\Cref{sec:replica-derivation}). It predicts alignment with the planted feature and the effective noise along the learned direction; the downstream analysis then translates these quantities into classification accuracy under the linear-probe protocol introduced in \Cref{sec:main-model}.

\Cref{fig:mask-sample-efficiency} compares these predictions with numerical training at finite dimension. Across the tested activations, increasing mask diversity $K$ lowers the weak recovery threshold and improves alignment above the transition in the settings shown. The replica predictions capture not only these threshold shifts, but also the magnitude and $\alpha$-dependence of the cosine similarity over the displayed range. The theory also captures training and test errors (\Cref{fig:app-reconstruction-errors}). Thus the analysis quantifies how mask resampling changes the representation learned from a given number of examples.

\paragraph{From feature recovery to downstream performance.} The lower row of \Cref{fig:mask-sample-efficiency} shows how these representation gains translate into classification accuracy. The linear-probe predictions (\Cref{subsec:downstream-classification}) track the numerical results as sample complexity and mask diversity vary across the four activation functions. Together, the two rows show that the theory captures both latent-feature recovery during pretraining and the downstream usefulness of the frozen representation.

The \bayesopt{Bayes-optimal} pretraining curve provides the reference using the matched priors and the realized coordinate variances (\Cref{app:bayes-optimal}). Increasing mask diversity brings the representations closer to this reference in the displayed comparisons, while a gap can remain even with dynamic masking (\Cref{fig:mask-sample-efficiency}).

\paragraph{How mask diversity lowers sample complexity.} The replica equations describe learning at finite signal strength $\beta$. A weak-signal expansion ($\beta \ll1$) isolates a simple expression for the sample complexity as a function of the number of masks. For activations satisfying $\sig(0)=0$ and $\sig(x)=s_\sig x+O(x^2)$ as $x\to0$, with $0<s_\sig<\infty$, fixed $\rho$, $\laR=0$, we have that the weak recovery threshold follows
\begin{equation}
\alpha_{c,\mathrm{sm}}^{(K)} = \frac{\kappa_2}{\beta^2} \left[1+\frac{1}{K} \left(\frac{1}{2\rho(1-\rho)}-1\right)\right] [1+o_\beta(1)], \qquad \kappa_2=\E_\gamma[\gamma^2].
    \label{eq:mask-diversity-threshold}
\end{equation}

\begin{figure}[!htbp]
\centering
\includegraphics[width=\textwidth]{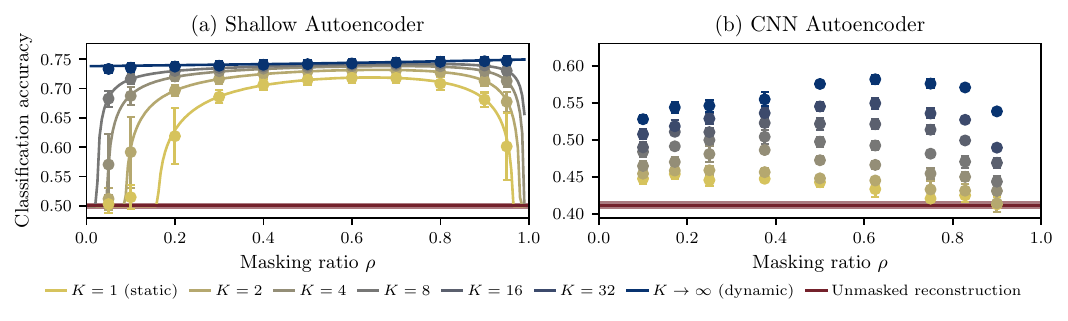}

\caption{\textbf{Finite mask diversity interpolates between static and dynamic training.} Downstream accuracy versus masking ratio for \textbf{(a)} the solvable model, with theoretical curves and simulation markers, and \textbf{(b)} a CNN autoencoder on CIFAR-10. Increasing $K$ improves accuracy in the settings shown, and the curves approach the \dynamicmask{dynamic masking} reference. The number of training samples is held fixed. Bars and shading show $\pm1$ SD across seeds; see \Cref{app:experimental-reporting} for details.}
\label{fig:K-interpolation}

\end{figure}

The derivation and conditions underlying this expansion are given in \Cref{subsec:ridgeless-weak-signal-thresholds,subsec:recovery-threshold-scope}. The prefactor $\kappa_2/\beta^2$ sets the weak-signal scale, including the cost of variance heterogeneity. At leading order, mask diversity lowers the excess threshold by a term proportional to $1/K$: doubling $K$ halves this excess sample requirement, with smaller absolute gains at each doubling.

This leading law is shared by linear, tanh, and ELU activations; ReLU has a different coefficient but retains the same ordering in $K$ (\Cref{subsec:ridgeless-weak-signal-thresholds}). Beyond weak signal, \Cref{fig:app-threshold-ratios} compares linear threshold ratios computed from the full replica equations with the asymptotic prediction. The ratios continue to decrease with $K$ at the finite $\beta$, even when they depart from the weak-signal expression. Finite-dimensional training in \Cref{fig:mask-sample-efficiency} supports this predicted ordering across all four activations.

\paragraph{Mask diversity can change the learning mechanism.} Beyond the weak-signal expansion, replicas reveal a qualitative difference between static and dynamic masking. In the linear case, at the weak recovery transition, static masking does not preferentially use cleaner coordinates to identify the feature: its threshold reflects noise heterogeneity only through the aggregate moment $\kappa_2$. Dynamic masking can instead find and favor cleaner coordinates as recovery begins, implicitly learning which coordinates provide more reliable information. At suitable masking ratios, this adaptation provides an additional benefit not captured by the leading weak-signal law \eqref{eq:mask-diversity-threshold}. \Cref{app:mask-noise-adaptation} derives this distinction and compares noise distributions with identical static but different dynamic thresholds.

\paragraph{Interaction with the masking ratio.} The finite-$K$ penalty in \eqref{eq:mask-diversity-threshold} is smallest at balanced masking and grows towards either endpoint. Near these extremes, each view offers relatively few visible or hidden coordinates, and a small collection of masks provides few distinct prediction tasks. Increasing $K$ reduces this penalty. The shallow-model sweep in \Cref{fig:K-interpolation}(a) illustrates the associated improvement in accuracy and the broader range of useful masking ratios as $K$ increases. At finite signal strength, the best downstream masking ratio also depends on the activation and the learned representation's noise, so it need not minimize the leading threshold alone.

\section{Testing the mask-diversity prediction in vision and language}

\label{sec:main-experiments}

\Cref{fig:K-interpolation} links the theoretical prediction to the CNN experiments: increasing the number \(K\) of fixed masks per sample improves downstream accuracy towards the dynamic-masking reference in both the solvable model (a) and CNN autoencoder (b), with the underlying dataset fixed. We test mask diversity further in ViT and BERT, and examine whether random cropping and flipping narrow the static–dynamic gap in vision.

\begin{figure}[!htbp]
    \centering
\includegraphics[width=0.5\linewidth]{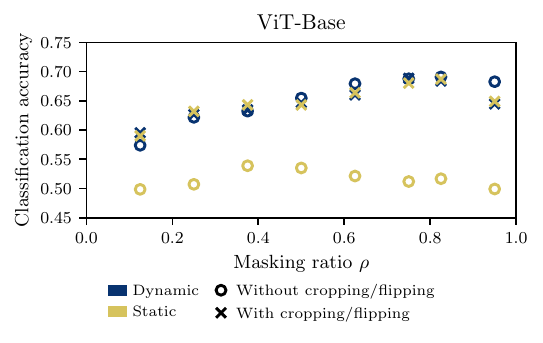}
    \caption{\textbf{Mask renewal and image augmentation.} Linear-probe validation top-5 accuracy for ViT-Base on ImageNet-100. Blue/yellow indicate \dynamicmask{dynamic}/\staticmask{static} masking; circles/crosses denote pretraining without/with stochastic cropping and flipping, respectively. All four conditions use a common seed, data split, initialization, optimizer settings, and linear-probing protocol. Each point represents one run. The blue cross at $\rho=0.75$ is the purple horizontal line in \Cref{fig:static-vs-dynamic}(c). See \Cref{app:experimental-reporting} for training details.}
    \label{fig:masking-augmentation}
\end{figure}

\begin{table}[!htbp]

\centering
\setlength{\tabcolsep}{6pt}
\begin{tabular}{lrrrrr}
\toprule
Masking & MNLI-m & QNLI & QQP & SST-2 & AVG \\
\midrule
$K=1$   & 72.9 & 82.6 & 85.9 & 87.6 & 62.7 \\
$K=10$  & 76.3 & 84.4 & 87.0 & 88.1 &  65.8 \\
Dynamic & 76.8 & 85.5 & 87.2 & 89.4 & 66.4\\
\bottomrule
\end{tabular}
\caption{Development scores ($\times 100$, $\uparrow$) for BERT-Medium under different masking procedures, on a subset of GLUE tasks. AVG is the unweighted mean of the task-specific scores over all eight evaluated tasks (accuracy; Spearman correlation for STS-B; MCC for MRPC, CoLA, and RTE), using MNLI-m for MNLI. Entries are means over three pretraining seeds, each averaged over two fine-tuning runs; standard deviations are in \Cref{tab:bert-glue-means}. See \Cref{app:bert-masking} for details.}
\label{tab:bert-main}
\end{table}

\paragraph{Isolating mask resampling.} We test the qualitative mask-diversity prediction in CNN autoencoders on CIFAR-10 and ViT-Base on ImageNet-100. Without stochastic cropping or flipping during pretraining, \dynamicmask{dynamic masking} improves downstream accuracy over \staticmask{static masking} in both architectures (\Cref{fig:static-vs-dynamic}(b,c)); see \Cref{fig:app-vit-top1} for ViT top-1 accuracy. At suitable masking ratios, resampling alone reaches accuracy comparable to our reference combining \standardaug{dynamic masking with random cropping and flipping} \citep{he2022masked}. With extended pretraining and probing schedules, \dynamicmask{dynamic masking} (without stochastic cropping or flipping) reaches accuracy comparable to the ImageNet-100 MAE baseline reported by \citet{zhang2022mask} (\Cref{tab:vit-training-budget,tab:vit-probe-augmentation}).

\paragraph{Interaction with cropping and flipping.} \Cref{fig:masking-augmentation} compares static and dynamic ViT masking with and without cropping and flipping. Under \staticmask{static masking}, patch positions remain fixed while transformations change the visible content and reconstruction targets. Cropping and flipping substantially improve static masking, narrowing the gap to dynamic masking. \Cref{fig:app-cnn-augmentation} extends the comparison to our CNN experiments with CIFAR-10 leading to the same qualitative conclusions.

\paragraph{Finite mask diversity.} In the CNN sweep of \Cref{fig:K-interpolation}(b), each image is assigned $K$ independently drawn masks that are reused throughout pretraining, whereas \dynamicmask{dynamic masking} draws a fresh mask at every presentation. For validation-selected encoders, mean downstream accuracy generally increases with $K$, approaching the dynamic-masking reference while using the same pretraining images. To assess sensitivity to the training protocol, we repeat the comparison at $\rho=0.625$ with five paired seeds and a fixed learning rate under a schedule defined in optimizer updates. The mask-diversity trend persists both with validation-selected checkpoints and when all encoders are evaluated at the same 125k-update checkpoint (\Cref{app:cnn-training-diagnostics}).

For \fullrecon{unmasked reconstruction} in \Cref{fig:static-vs-dynamic}, we found that lower reconstruction error need not imply better representations for the downstream task: continued pretraining can reduce reconstruction MSE while linear probing degrades (\Cref{fig:app-no-masking-cnn,fig:app-no-masking-vit}).

\paragraph{Mask diversity in language modeling.} A BERT-Medium pilot compares $K=1$, $K=10$, and \dynamicmask{dynamic masking}, using only \texttt{[MASK]} replacements (\Cref{tab:bert-main}). The corpus, architecture, and pretraining budget are shared across masking conditions. The ordering \dynamicmask{dynamic masking} $> K=10 > K=1$ holds in mean development score on seven of the eight GLUE tasks. Additional results and experimental details are provided in \Cref{app:bert-masking}.

\FloatBarrier
\section{Discussion}
\label{sec:main-discussion}

Our analysis separates two benefits of masked learning: masked prediction can favor shared structure over nuisance variance, while mask resampling can reduce the sample complexity required to recover that structure. We establish the first benefit rigorously in the linear model and quantify the second through replica calculations that agree with simulations. The CNN and ViT experiments, together with the BERT pilot, support the relevance of mask diversity across architectures and data modalities.

This perspective connects masked reconstruction to theories of data augmentation. Analyses based on transformation groups explain statistical gains through invariance and averaging \citep{chen2020group,pmlr-v134-mei21a}; in joint-embedding self-supervised learning, multiple views can reduce augmentation variability while retaining the variability associated with the original observations \citep{pmlr-v202-cabannes23a}. Our analysis gives a complementary account for masked reconstruction: changing $K$ preserves both the underlying dataset and the distribution of individual masks, yet changes the predicted recovery threshold. The weak-signal law quantifies the additional sample requirement at finite $K$ and the diminishing gains from increasing mask diversity.

The ViT comparison suggests that geometric augmentation can partly compensate for mask reuse. A predictive theory of augmentation should therefore explain when geometric transformations and mask resampling offer complementary benefits, and when one reduces the benefit of the other. Our replica analysis also identifies a qualitative change in learning: dynamic masking can exploit cleaner coordinates at the recovery transition, while static masking lacks this adaptation. Whether resampling similarly changes which information is learned in larger models remains an open question.

Better representations need not require more data; they can come from a greater diversity of prediction tasks on the same observations.

\section*{Acknowledgments}

We acknowledge funding from the Swiss National Science Foundation grants SNSF SMArtNet (grant number 212049), the Simons Collaboration on the Physics of Learning and Neural Computation via the Simons Foundation grant ($\#$1257413 (LZ)), and the Swiss State Secretariat for Education, Research and Innovation (SERI), which funded the ERC Advanced Grant SPASM-net. 
This work was also supported under project IDs \#43 and SNAI000097 as part of the Swiss AI Initiative, through a grant from the ETH Domain and computational resources provided by the Swiss National Supercomputing Centre (CSCS) under the Alps infrastructure.

We thank Vittorio Erba, Emanuele Troiani, Cédric Koller and Vicente Conde Mendes for preliminary discussions on the architecture and data model. We are also grateful to Alexander Mathis and Amir Zamir for insightful discussions and comments that helped sharpen our perspective on the modeling choices and practical training of masked autoencoders.

\subsection*{AI use statement}
We used generative AI tools in an iterative, author-led research process to assist with literature retrieval, drafting, and polishing. Additional uses included formulating and refining mathematical claims, developing and revising derivations and proof arguments, including input of proof ideas from AI, discussing experimental design and interpretation, implementing and debugging numerical routines, and preparing scientific plots. Throughout this collaboration, the authors actively directed the work, critically examined the suggestions, and checked and revised AI-assisted outputs at every stage, including the mathematical arguments, computational methods, references, and written presentation. All scientific judgments and decisions about the manuscript's content remained with the authors. The authors take full responsibility for the manuscript.

\newpage

\bibliography{references}
\bibliographystyle{unsrtnat}

\clearpage
\appendix

\numberwithin{equation}{section}

% =============================================================================
\startcontents[appendices]
\section*{Appendix Contents}
\printcontents[appendices]{}{1}{\setcounter{tocdepth}{2}}

\newpage

\section{Detailed comparison with the closest prior work}
\label{app:positioning}

The main text outlines the broad lines of related work. Here we expand on the closest references to clarify both their settings and what their analyses establish. We compare data models, prediction objectives, and masking protocols, as well as the observables and guarantees studied. These distinctions locate our contributions: recovery under heterogeneous noise for the exact mask-averaged objective, and replica predictions for mask diversity and downstream performance.

\paragraph{Reconstruction and useful features.} The reconstruction--perception mismatch identified by \citet{balestriero2024learning} motivates our separation: a direction can explain substantial input variance without carrying useful latent information. Their closed-form linear denoising analysis shows how corruption, including masking, can improve this alignment. In our spiked model, coordinate-dependent noise makes the obstruction explicit, and hidden-target prediction admits a recovery guarantee at linear sample complexity: every global minimizer of the exact mask-averaged linear objective weakly recovers at linear sample complexity while ordinary PCA fails (\Cref{thm:main-separation}).

\paragraph{Full reconstruction from masked inputs.} The linear MAE analyzed by \citet{bisulco2025linearity} has untied encoder and decoder matrices, whereas our bottleneck uses tied weights. More crucially, their mask-averaged analytical loss reconstructs all coordinates from masked inputs, including visible coordinates; ours scores only hidden targets. The visible-coordinate rewards retain a contribution from reproducing input variance, changing both the loss landscape and its masking-ratio dependence. Their characterization of critical points and of the spatial correlations selected by patch size and masking ratio is therefore closely relevant, but its spectral solution does not directly solve our objective. Our statistical analysis follows how the structure favored by hidden-target prediction translates into recovery at finite sample size and under finite mask reuse. These are different objectives and constraints, rather than a general performance ordering.

\paragraph{Single-token prediction and factored attention.} \citet{rende2024mapping} connect factored attention to generalized Potts conditionals and use replicas to predict generalization in a Gaussian relaxation. There, predicting one coordinate from all the others probes a full-rank random precision matrix. This differs from our low-rank signal with diagonal heterogeneous noise both in correlation structure and in the number of targets per mask: a positive limiting fraction is hidden in our model. The shared tied bottleneck couples these prediction tasks, so its alignment with the planted feature, and the downstream utility of that alignment, become explicit observables as the number of reused masks varies.

\paragraph{The joint matrix of single-coordinate predictors.} Single-coordinate prediction can itself reveal latent structure: \citet{wortsman2026a} analyze the joint matrix of self-supervised ridge predictors, rigorously characterizing its spectrum, prediction risk, and spiked-model BBP recovery transition. They also establish prediction-risk advantages over PCA in structured covariance settings. Coupling an extensive set of hidden targets through a shared tied bottleneck changes the estimator in our work; mask diversity is an additional statistical control. Our analysis follows its effects on recovery and downstream performance via replicas, alongside a rigorous recovery separation under heteroskedastic noise for the exact mask-averaged objective. Nevertheless, it is important to stress that recovering latent structure through the data correlations exploited by masked prediction is already established in their work.

\paragraph{Masking ratio and a selected prediction target.} \citet{dong2026understanding} start from a loss over all hidden coordinates, but decouple their linear predictors and solve a stylized regression for one selected target, $y=X\boldsymbol\beta+\epsilon$ in their notation. Their isotropic and spiked covariance assumptions describe the remaining covariates $X$; the target coefficient and residual noise are specified separately. Consequently, this target need not be statistically representative of the other coordinates of the augmented sample $(X,y)$. The resulting risk formulas rigorously quantify how masking both corrupts covariates and selects the samples supplying this target. In our common-factor model, every coordinate follows the same generative construction without a designated response, and all hidden targets train the same tied bottleneck; the analysis then follows collective feature recovery. Their fixed-ratio theory motivates randomized masking ratios, with downstream benefits assessed across modalities; our finite-reuse calculation keeps the ratio fixed and links representation alignment to a finite-label probe, with replica rather than rigorous guarantees for the general learning curves.

\paragraph{Contrastive connections and downstream guarantees.} The connection between masked reconstruction and useful representations is also explicit in \citet{zhang2022mask}: reconstruction induces alignment on a mask-based augmentation graph, supporting downstream guarantees and an analysis of masking ratio and dimensional collapse. Our recovery thresholds address a more specific statistical question: how reusing a finite collection of masks changes the learned alignment at a fixed sample-to-dimension ratio. The downstream calculation then tracks the effect of that alignment on a probe trained with finitely many labels.

\paragraph{Identifiability in hierarchical models.} Under the hierarchical generative and invertibility assumptions of \citet{kong2023understanding}, masking ratio and patch size determine which latent variables are shared between visible and hidden regions and can be identified. In our additive-noise model, the relevant observable is partial alignment with a planted direction at finite sample-to-dimension ratio. The recovery transition and its dependence on mask diversity concern statistical estimation under these different data and encoder assumptions, rather than an alternative identifiability statement for their hierarchy.

\paragraph{Nonlinear recovery beyond PCA.} \citet{conde2026solvable} show that nonlinear autoencoders recover higher-order structure invisible to PCA, and that reconstruction loss can misalign with representation quality. Our signal is covariance-visible, but heterogeneous variances obstruct ordinary PCA. Hidden-target prediction changes which covariance structure the objective rewards, yielding a separation that does not require nonlinearity. Complementary high-dimensional work studies how pretraining sample size and representation dimensionality jointly control downstream linear-probe generalization in spiked models \citep{njaradi2026optimal}.

\paragraph{Noise-aware spectral alternatives.} HeteroPCA \citep{zhang2022heteroskedastic} and whitening-based methods \citep{leeb2021optimal} address the same nuisance of heterogeneous noise through spectral correction. \Cref{fig:app-heterogeneous-pca} compares such estimators on the same observations: our separation from ordinary PCA does not imply a separation from all spectral methods. The comparison locates masking among several routes to recovery. It changes which structure the reconstruction objective rewards, without supplying the autoencoder with the coordinate variances or their distribution.

\paragraph{Bayes-optimal inference and the role of prior information.} Our Bayes analysis builds on the low-rank matrix-estimation framework of \citet{lesieur2015mmse,lesieur2017constrained}, which connects matched priors and output channels to approximate message passing and state evolution. We specialize this framework to coordinate-dependent Gaussian noise to obtain the reference learning curves and AMP updates in \Cref{app:bayes-optimal}. For Gaussian latent priors, the continuous recovery threshold is $\alpha_c^{\rm BO}=[\beta^2\E_\gamma(\gamma^{-2})]^{-1}$ when the inverse second moment is finite (\Cref{subsec:BO-weak-threshold}). Its dependence on inverse variances highlights the information available in low-noise coordinates. \citet{ZhangMondelli2024} provide an independent rigorous foundation: their analysis of doubly heteroskedastic matrix denoising yields this weak-estimation threshold under their Gaussian-prior and compact, positive covariance-support assumptions, and a covariance-aware spectral estimator attains it under their additional technical condition. Thus the Bayes reference connects our reconstruction-based learners to an established optimal-inference benchmark, rather than introducing a new information-theoretic limit.

The distinction in available information matters for interpreting the comparison. Our Bayes reference uses the matched latent priors and the realized coordinate variances, whereas the autoencoders receive neither the variances nor their distribution. A remaining gap to Bayes therefore reflects both the reconstruction objective and the additional information used by matched inference. Increasing mask diversity improves how that objective uses the observations; it does not supply the learner with the noise model. The variance-consistency calculation in \Cref{app:bayes-optimal} discusses learning the noise levels from the samples, separately from the known-variance benchmark used here.

\paragraph{Augmentation multiplicity and finite reuse.} Prior work has studied the benefit of drawing multiple augmented views of the same example. \citet{fort2021drawing} show empirically that averaging several freshly sampled augmentations per image can improve supervised generalization, while \citet{pmlr-v202-cabannes23a} analyze the corresponding multi-view setting theoretically and separate variability due to sampling original examples from variability due to sampling their augmentations. These works vary the number of views used for an example while preserving the underlying augmentation distribution. We study a different axis: \emph{finite reuse}. Each training sample is assigned $K$ masks that are held fixed throughout learning, so changing $K$ changes the finite collection of prediction tasks available to the learner rather than the number of views averaged within an update. Replica predictions quantify how this finite mask diversity changes feature recovery and downstream accuracy, including the weak-signal excess sample requirement in \eqref{eq:mask-diversity-threshold}. The masks contain no additional information about the planted feature; we show that their benefit comes from how a richer collection of prediction tasks allows the objective to extract information already present in the observed samples.

\paragraph{Mask resampling and other augmentations.} Dynamic masking in RoBERTa \citep{liu2019roberta} and multiple masked views per update in data2vec~2.0 \citep{baevski2023efficient} already exploit variation in prediction tasks. The interaction with image augmentation matters when interpreting its benefit: even fixed patch positions expose changing content under stochastic cropping and flipping. Our vision comparisons isolate this interaction, while the BERT pilot provides complementary evidence for finite mask diversity. Group-based accounts of augmentation through invariance and averaging \citep{chen2020group,pmlr-v134-mei21a} offer another perspective; here the controlled variable is the reuse of prediction tasks from each original observation. These experiments test the qualitative role of mask diversity beyond the solvable model, without assuming that its quantitative recovery law transfers to these architectures and datasets.

\newpage
\section{Model, notation, and effective output channel}
\label{app:model-details}

\subsection{Notation}
\label{app:notation}
% =============================================================================

\begin{center}
\begin{tabular}{@{}ll@{}}
\toprule
\textbf{Symbol} & \textbf{Meaning}\\
\midrule
\multicolumn{2}{@{}l}{\textbf{\emph{Norms}}}\\
$\|\bm v\|=\|\bm v\|_2$ & Euclidean vector norm, $(\sum_i |v_i|^2)^{1/2}$\\
$\|A\|_{\mathrm{op}}$ & matrix operator norm, $\sup_{\|\bm v\|=1}\|A\bm v\|$\\
$\|A\|_{\mathrm F}$ & Frobenius norm, $(\sum_{i,j}|A_{ij}|^2)^{1/2}$\\
\addlinespace
\multicolumn{2}{@{}l}{\textbf{\emph{Data and model}}}\\
$d,\;n,\;\alpha=n/d$ & dimension, number of samples, sample complexity\\
$\bu^\star$ & planted latent direction\\
$\lambda_\mu$ & latent variable of sample $\mu$; $\E[\lambda^2]=1$\\
$\beta$ & signal strength/signal-to-noise ratio\\
$\gamma_i^\star\sim P_\gamma$ & noise variance of coordinate $i\in[1,d]$; $\E_\gamma[\gamma]=1$\\
$\kappa_2=\E_\gamma[\gamma^2]$ & variance heterogeneity; $\kappa_2=1$ iff homoskedastic\\
$\Gamma^\star$ & $\mathrm{diag}(\gamma_1^\star,\dots,\gamma_d^\star)$\\
$\bw,\;\widehat\bw$ & tied encoder/decoder weights; empirical-risk minimizer\\
$\sig$ & activation function\\
$\laR$ & ridge parameter ($\laR=0$ is ridgeless)\\
\addlinespace
\multicolumn{2}{@{}l}{\textbf{\emph{Masking}}}\\
$\rho$ & masking probability per coordinate\\
$K$ & number of fixed masks per sample (mask diversity)\\
$\chi_{\mu i}^{(r)}$ & mask indicator; $1$ if coordinate $i$ hidden in view $r$\\
$h$ & full preactivation, shared by all $K$ views of a sample\\
$\delta_r$ & mask-induced fluctuation of view $r$\\
\addlinespace
\multicolumn{2}{@{}l}{\emph{\textbf{Order parameters}}}\\
$M=d^{-1}(\bu^\star)^\top\bw$ & teacher--student overlap\\
$q_I=d^{-1}\|\bw\|_2^2$ & Euclidean self-overlap\\
$q_\Gamma=d^{-1}\bw^\top\Gamma^\star\bw$ & noise-weighted self-overlap\\
$\theta_{\bu^\star}=|M|/[\sqrt{q_I}(\|\bu^\star\|/\sqrt d)]$ & cosine similarity with $\bu^\star$\\
$\Psi^{(s)}$ & replicated free entropy with $s$ replicas\\
$\eta_{\rm th}$ & Gibbs inverse temperature; $\eta_{\rm th}\to\infty$ for ERM\\
\addlinespace
\multicolumn{2}{@{}l}{\emph{\textbf{Thresholds and response functions}}}\\
$\alpha_c^{(K)}$ & weak-recovery threshold at mask diversity $K$\\
$\alpha_c^{\rm BO}$ & Bayes-optimal recovery threshold\\
$\widetilde{\bm F}_K$ & whitened output force\\
$\mathcal{X}_K$ & weak-signal limit of the averaged informative response\\
$T_K$ & limiting, $q_\Gamma$-normalized force second moment\\
$\eta_{K,c}^2=\mathcal E_{K,c}/(q_{\Gamma,c}\mathcal{X}_{K,c}^2)\ge1$ & exact output-efficiency factor, \eqref{eq:fixed-beta-eta}\\
$v_{K}(\gamma)\propto(1+\tau_K\gamma)^{-1}$ & input susceptibility of the uninformative saddle\\
$\tau_{K,c}$ & susceptibility shape; sign decides up- or down-weighting\\
$\mathcal R_\gamma(\tau)$ & heteroskedastic reweighting factor; $\mathcal R_\gamma(0)=\kappa_2$\\
$\zeta_c$ & spectral parameter of the full-reconstruction branch\\
$\gamma_+$ & essential upper edge of $P_\gamma$\\
\bottomrule
\end{tabular}
\end{center}
Bold symbols denote vectors, and a star marks ground-truth quantities. Unless conditioning is explicitly stated or specified in the surrounding text, an unsubscripted $\E$ denotes expectation over all model randomness. Subscripts specify the variables being averaged, with other quantities held fixed: in particular, $\E_\gamma$ averages over the noise-variance distribution $P_\gamma$, and $\E_\chi$ averages over masks at a fixed dataset and candidate weight vector. Conditional expectations are written as $\E[\,\cdot\mid\cdot\,]$. Population quantities defined by averaging over fresh samples at fixed quenched variables retain their dependence on the realized teacher $\bu^\star$ and variance matrix $\Gamma^\star$; for example, $\E[\bx\bx^\top\mid\bu^\star,\Gamma^\star] =\Gamma^\star+\beta\bu^\star(\bu^\star)^\top/d$. We also write $[m]=\{1,\ldots,m\}$ for any positive integer $m$.

Samples index rows and coordinates index columns throughout:
\begin{equation}
 X=(x_{\mu i}),\quad Z=(z_{\mu i})\in\R^{n\times d},\qquad X=\sqrt{\beta/d}\,\bm\lambda(\bu^\star)^\top +Z(\Gamma^\star)^{1/2}.
 \label{eq:notation-data-matrices}
\end{equation}
Here $\bm\lambda=(\lambda_1,\ldots,\lambda_n)^\top$.

Here $\mathcal E_{K,c}=\E\|\widetilde{\bm F}_{K,c}\|_2^2$ and $\mathcal X_{K,c}$ are the force second moment and averaged informative response at criticality. Along the regular weak-signal branch at fixed finite $K$, $\eta_{K,c}^2\to T_K/\mathcal X_K^2$ as $\beta\downarrow0$, with the limiting quantities defined in \eqref{eq:weak-signal-chi-definition}--\eqref{eq:weak-signal-T-definition}.

The cosine identity in the table is exact for nonzero teacher and weights. Since $\|\bu^\star\|^2/d\to1$, its high-dimensional expression is $|M|/\sqrt{q_I}$. For a fixed candidate weight vector, independent masks make the fluctuations $\delta_r$ independent across views conditional on the sample; averaging over the shared sample need not preserve independence at finite dimension. Independence in the limiting Gaussian channel is understood under the assumptions in \Cref{sec:finite-K-model}.

\subsection{Detailed model and effective output channel}
\label{sec:finite-K-model}
% =============================================================================

As mentioned in the main document, we study a single hidden-neuron autoencoder with tied weights and $K=O_d(1)$ fixed masks per training sample. The high-dimensional limit is the linear one where
\begin{equation}
    n,d\to\infty, \qquad \frac nd\to\alpha,
\end{equation}
with $K$ fixed. Dynamic masking is obtained afterward through the sequential limit $K\to\infty$ at the level of the effective output channel.

\subsubsection{Data model}
\label{subsec:data-model}

\paragraph{Why allow unequal noise variances?} The distinction between variance and task-relevant information motivates the model. Reconstruction objectives can emphasize components that are weakly useful for perception \citep{balestriero2024learning}. Examples of nuisance variation dominating leading components include illumination in face recognition \citep{belhumeur1997eigenfaces}, natural-image backgrounds superimposed on handwritten digits \citep{abid2018exploring}, and heterogeneous noise in biological data \citep{brennecke2013accounting}. These examples motivate separating useful structure from variance; they do not identify their covariance structure with the coordinate-noise model studied here. Our model isolates this distinction using a single shared feature and independently varying coordinate noise levels.

The training samples $\bx_\mu\in\R^d$ are generated from the heteroskedastic rank-one model
\begin{equation}
    x_{\mu i} = \sqrt{\frac{\beta}{d}}\,\lambda_\mu u_i^\star + \sqrt{\gamma_i^\star}\,z_{\mu i}, \qquad \mu=1,\ldots,n, \quad i=1,\ldots,d.
    \label{eq:data-model}
\end{equation}
The sample latent variables satisfy
\begin{equation}
    \lambda_\mu\iid P_\lambda, \qquad \E[\lambda]=0, \qquad \E[\lambda^2]=1.
    \label{eq:lambda-law}
\end{equation}
The teacher direction is quenched and has coordinates
\begin{equation}
    u_i^\star\iid P_u, \qquad \E[u^\star]=0, \qquad \E[(u^\star)^2]=1.
    \label{eq:teacher-law}
\end{equation}
The noise variables are
\begin{equation}
    z_{\mu i}\iid\N(0,1).
\end{equation}
The coordinate variances are also quenched,
\begin{equation}
    \gamma_i^\star\iid P_\gamma, \qquad \gamma_i^\star>0, \qquad \E_\gamma[\gamma]=1, \qquad \E_\gamma[\gamma^2]<\infty.
    \label{eq:gamma-law}
\end{equation}
It is convenient to collect the coordinate variances into the quenched diagonal matrix
\begin{equation}
    \Gamma^\star := \operatorname{diag}\!\left( \gamma_1^\star,\ldots,\gamma_d^\star \right).
    \label{eq:Gamma-definition}
\end{equation}
Thus $\bu^\star$ and $\bm\gamma^\star=(\gamma_1^\star,\ldots,\gamma_d^\star)$, equivalently $\Gamma^\star$, are coordinate-wise quenched disorder, whereas $\lambda_\mu$ and $z_{\mu i}$ are sample-dependent random variables. Unless stated otherwise, all random families introduced in the data model are mutually independent. In particular,
\begin{equation}
    \bu^\star,\qquad \bm\gamma^\star,\qquad \{\lambda_\mu\}_{\mu=1}^n,\qquad \{z_{\mu i}\}_{\mu,i}
\end{equation}
are mutually independent.

\subsubsection{Finite-$K$ masked autoencoder}
\label{subsec:finite-K-masks}

Fix a masking probability $0<\rho<1$. For each unmasked sample $\bx_\mu$, we draw $K$ independent masks
\begin{equation}
    \chi_{\mu i}^{(r)} \iid \operatorname{Bernoulli}(\rho), \qquad r=1,\ldots,K,
    \label{eq:finite-K-mask-definition}
\end{equation}
with the convention
\begin{equation}
    \chi_{\mu i}^{(r)}=1 \quad\Longleftrightarrow\quad \text{coordinate $i$ is masked in view $r$ of sample $\mu$}.
\end{equation}
The corresponding visible indicator is simply $1-\chi_{\mu i}^{(r)}$; meaning that the input and output of the autoencoder will be complementary.

The collection of masks $\{\chi_{\mu i}^{(r)}\}_{\mu,i,r}$ is independent of all random variables entering the data-generating model. For fixed $\mu$, however, the $K$ masked views share the same underlying sample $\bx_\mu$. The finite-$K$ dataset is therefore
\begin{equation}
    \calD_K = \left\{ \left( \bx_\mu, \bchi_\mu^{(1)},\ldots,\bchi_\mu^{(K)} \right) \right\}_{\mu=1}^n.
    \label{eq:finite-K-dataset}
\end{equation}
The masks are part of the training disorder and remain fixed during optimization.

\paragraph{Architecture and masked reconstruction loss.}

The model has one hidden unit, no bias, a scalar activation function $\sig$, and tied encoder/decoder weights $\bw\in\R^d$. For view $r$ of sample $\mu$, the encoder receives only the visible coordinates. Its preactivation is
\begin{equation}
    h_{\mu r,\vis}(\bw) = \frac1{\sqrt d} \sum_{i=1}^d \bigl(1-\chi_{\mu i}^{(r)}\bigr) w_i x_{\mu i}.
    \label{eq:finite-K-visible-preactivation}
\end{equation}
The tied decoder reconstructs coordinate $i$ as
\begin{equation}
    \widehat x_{\mu i}^{(r)} = \frac{w_i}{\sqrt d}\, \sig\!\left(h_{\mu r,\vis}(\bw)\right).
    \label{eq:finite-K-reconstruction}
\end{equation}
Only the masked coordinates contribute to the reconstruction loss. The loss of view $r$ is therefore
\begin{equation}
    \ell_{d}^{(r)} \left(\bw;\bx_\mu,\bchi_\mu^{(r)}\right) = \sum_{i=1}^d \chi_{\mu i}^{(r)} \left[ x_{\mu i} - \frac{w_i}{\sqrt d} \sig\!\left(h_{\mu r,\vis}(\bw)\right) \right]^2.
    \label{eq:finite-K-view-loss}
\end{equation}
The finite-$K$ loss of the original sample is the average over its $K$ fixed masks,
\begin{equation}
    \ell_{K,d}(\bw;\bx_\mu,\bchi_\mu^{(1:K)}) = \frac1K \sum_{r=1}^K \ell_d^{(r)} \left(\bw;\bx_\mu,\bchi_\mu^{(r)}\right).
    \label{eq:finite-K-sample-loss-finite-d}
\end{equation}
and we study the standard risk minimization of
\begin{equation}
    \mathcal R_{K,d}(\bw) = \frac{\laR}{2}\|\bw\|_2^2 + \sum_{\mu=1}^n \ell_{K,d} \left(\bw;\bx_\mu,\bchi_\mu^{(1:K)}\right).
    \label{eq:finite-K-ERM}
\end{equation}
\paragraph{Effective sample loss.}

For each view, define the contribution of the masked coordinates to the full preactivation,
\begin{equation}
    h_{\mu r,\mask}(\bw) = \frac1{\sqrt d} \sum_{i=1}^d \chi_{\mu i}^{(r)}w_i x_{\mu i}.
    \label{eq:finite-K-masked-preactivation}
\end{equation}
Expanding \eqref{eq:finite-K-view-loss} and dropping the term independent of $\bw$ gives
\begin{equation}
    \ell_{d}^{(r),\mathrm{inf}} = \frac{\sig(h_{\mu r,\vis})^2}{d} \sum_{i=1}^d \chi_{\mu i}^{(r)}w_i^2 - 2\sig(h_{\mu r,\vis})h_{\mu r,\mask}.
    \label{eq:finite-K-view-informative-loss}
\end{equation}
Introduce the Euclidean self-overlap
\begin{equation}
    Q_I(\bw) = \frac1d\sum_{i=1}^d w_i^2 = \frac1d\bw^\top\bw.
    \label{eq:QI-definition}
\end{equation}
\paragraph{Self-averaging of the masked norm.} In the one-sample disorder average, the candidate weight vector $\bw$ is held fixed while the masks are averaged. Hence, we need to assume the following
\[
Q_I(w)=O(1), \qquad \frac{1}{d^2}\sum_{i=1}^d w_i^4 \longrightarrow 0.
\]
to keep the replica derivation consistent. The second condition excludes a macroscopic contribution from any single coordinate to the Euclidean norm. For independent Bernoulli mask entries, with the candidate weights held fixed,
\[
\mathbb{E}_{\chi}\!\left[ \frac{1}{d}\sum_{i=1}^d (\chi_{\mu i}^{(r)}-\rho)w_i^2 \right]=0, \qquad \mathbb{E}_{\chi}\!\left[ \left( \frac{1}{d}\sum_{i=1}^d (\chi_{\mu i}^{(r)}-\rho)w_i^2 \right)^2 \right] = \frac{\rho(1-\rho)}{d^2}\sum_{i=1}^d w_i^4.
\]
Consequently, for each fixed sample and fixed $K$,
\begin{equation}
\max_{1\le r\le K} \left| \frac{1}{d}\sum_{i=1}^d\chi_{\mu i}^{(r)}w_i^2 -\rho Q_I(w) \right| \xrightarrow[d\to\infty]{\mathbb{P}_{\chi}}0.
\label{eq:mask_norm_self_averaging}
\end{equation}
The same conclusion holds jointly for any fixed finite collection of candidate replicas. For fixed $K$, under the high-dimensional self-averaging used throughout the analysis,
\begin{equation}
    \frac1d \sum_{i=1}^d \chi_{\mu i}^{(r)}w_i^2 \dto \rho Q_I(\bw), \qquad r=1,\ldots,K.
    \label{eq:masked-norm-self-average}
\end{equation}
Hence the effective high-dimensional sample loss is
\begin{equation}
    \ell_K \left( \{h_{r,\vis},h_{r,\mask}\}_{r=1}^K;Q_I \right) = \frac1K\sum_{r=1}^K \left[ \rho Q_I\sig(h_{r,\vis})^2 - 2\sig(h_{r,\vis})h_{r,\mask} \right]
    \label{eq:finite-K-effective-loss-raw}
\end{equation}
From this point onward, the sample index $\mu$ is suppressed when discussing the one-sample channel. We also work with the centered, $\bw$-dependent objective obtained after dropping the reconstruction term independent of $\bw$; this leaves the Gibbs measure and the ERM minimizers unchanged.

\subsubsection{Low-dimensional field representation}
\label{subsec:finite-K-decomposition}

Define the ordinary full preactivation
\begin{equation}
    h(\bw) = \frac1{\sqrt d}\sum_{i=1}^d w_i x_i.
    \label{eq:full-preactivation}
\end{equation}
Since every coordinate is either visible or masked in view $r$,
\begin{equation}
    h_{r,\vis}+h_{r,\mask}=h, \qquad r=1,\ldots,K.
    \label{eq:finite-K-sum-identity}
\end{equation}
This identity is exact at finite $d$. In particular, the $2K$ quantities
\begin{equation}
    (h_{1,\vis},\ldots,h_{K,\vis}, h_{1,\mask},\ldots,h_{K,\mask})
\end{equation}
contain only $K+1$ independent directions.

Define the centered visible-field fluctuation of mask $r$ by
\begin{equation}
    \delta_r := h_{r,\vis}-(1-\rho)h.
    \label{eq:finite-K-delta-definition}
\end{equation}
Using \eqref{eq:finite-K-visible-preactivation} and \eqref{eq:full-preactivation},
\begin{equation}
    \delta_r = \frac1{\sqrt d} \sum_{i=1}^d \bigl(\rho-\chi_i^{(r)}\bigr)w_ix_i.
    \label{eq:finite-K-delta-explicit}
\end{equation}
Therefore
\begin{equation}
    h_{r,\vis} =(1-\rho)h+\delta_r, \qquad h_{r,\mask} =\rho h-\delta_r.
    \label{eq:finite-K-decomposition-identities}
\end{equation}
The natural field vector for the finite-$K$ problem is thus
\begin{equation}
    \bfield_K = (h,\delta_1,\ldots,\delta_K)\in\R^{K+1}.
    \label{eq:finite-K-field-vector}
\end{equation}

The centered mask coefficients satisfy
\begin{equation}
    \E[\rho-\chi_i^{(r)}]=0,
\end{equation}
and, by independence of different masks,
\begin{equation}
    \E\!\left[ (\rho-\chi_i^{(r)})(\rho-\chi_i^{(s)}) \right] = \rho(1-\rho)\delta_{rs}.
    \label{eq:centered-mask-covariance}
\end{equation}
This is the basic reason why the mask fluctuations become diagonal in the reduced representation.

Substituting \eqref{eq:finite-K-decomposition-identities} into \eqref{eq:finite-K-effective-loss-raw} gives the finite-$K$ effective loss directly in terms of the nonredundant fields,
\begin{equation}
    \ell_K(h,\delta_1,\ldots,\delta_K;Q_I) = \frac1K\sum_{r=1}^K \left[ \rho Q_I \sig\!\left((1-\rho)h+\delta_r\right)^2 - 2\sig\!\left((1-\rho)h+\delta_r\right) (\rho h-\delta_r) \right].
    \label{eq:finite-K-effective-loss-reduced}
\end{equation}

\paragraph{Finite-$K$ Gaussian channel.}
\label{subsec:finite-K-Gaussian-channel}

We use the replica method \citep{mezard1987spin,mezard2009information,zdeborova2016statistical}, a heuristic tool from statistical physics.

For replicas $\mathsf a,\mathsf b$, define the teacher--student overlap
\begin{equation}
    M^{\mathsf a} := \frac1d\sum_{i=1}^d u_i^\star w_i^{\mathsf a} = \frac1d(\bu^\star)^\top\bw^{\mathsf a}.
    \label{eq:M-overlap}
\end{equation}
More generally, for a deterministic or quenched metric $A\in\R^{d\times d}$, define the metric-dependent replica overlap
\begin{equation}
    Q_A^{\mathsf a\mathsf b} := \frac1d(\bw^{\mathsf a})^\top A\bw^{\mathsf b}.
    \label{eq:QA-overlap-family}
\end{equation}
The two metrics needed in the present heteroskedastic problem are
\begin{equation}
    Q_I^{\mathsf a\mathsf b} = \frac1d\sum_{i=1}^d w_i^{\mathsf a}w_i^{\mathsf b}, \qquad Q_\Gamma^{\mathsf a\mathsf b} = \frac1d\sum_{i=1}^d \gamma_i^\star w_i^{\mathsf a}w_i^{\mathsf b} = \frac1d(\bw^{\mathsf a})^\top\Gamma^\star\bw^{\mathsf b}.
    \label{eq:QI-QGamma-overlaps}
\end{equation}
Thus $Q_I^{\mathsf a\mathsf b}$ is the ordinary Euclidean replica overlap, whereas $Q_\Gamma^{\mathsf a\mathsf b}$ is the noise-weighted overlap controlling the covariance of the projected Gaussian data noise. In the homoskedastic limit $\Gamma^\star=I$, the two coincide: $Q_\Gamma^{\mathsf a\mathsf b}=Q_I^{\mathsf a\mathsf b}$. For a single replica we use the shorthand $Q_I:=Q_I^{\mathsf a\mathsf a}$ and $Q_\Gamma:=Q_\Gamma^{\mathsf a\mathsf a}$.

In the following sample and mask averages, the teacher $\bu^\star$, variance realization $\Gamma^\star$, and candidate replicas are held fixed; their conditioning is suppressed. Conditionally also on the sample latent variable $\lambda$, the full preactivation satisfies
\begin{equation}
    \E[h^{\mathsf a}\mid\lambda] \longrightarrow \sqrt\beta\lambda M^{\mathsf a}, \qquad \Cov(h^{\mathsf a},h^{\mathsf b}\mid\lambda) \longrightarrow Q_\Gamma^{\mathsf a\mathsf b}.
    \label{eq:h-replica-channel}
\end{equation}
For the centered mask fluctuations,
\begin{gather}
    \E[\delta_r^{\mathsf a}\mid\lambda] \longrightarrow 0,
    \label{eq:delta-mean}
    \\
    \Cov(h^{\mathsf a},\delta_r^{\mathsf b}\mid\lambda) \longrightarrow 0,
    \label{eq:h-delta-cov}
    \\
    \Cov(\delta_r^{\mathsf a},\delta_s^{\mathsf b}\mid\lambda) \longrightarrow \rho(1-\rho)\delta_{rs}Q_\Gamma^{\mathsf a\mathsf b}.
    \label{eq:delta-delta-cov}
\end{gather}
The rank-one signal contribution to the centered fluctuations is subleading: because $\rho-\chi_i^{(r)}$ is centered, the corresponding empirical sum is $o_d(1)$ under the same self-averaging assumptions.

Let
\begin{equation}
    \bm M_s=(M^1,\ldots,M^s)^\top, \qquad \mathsf Q_{\Gamma,s}=(Q_\Gamma^{\mathsf a\mathsf b})_{\mathsf a,\mathsf b\le s}, \qquad \bm e_0=(1,0,\ldots,0)^\top\in\R^{K+1}.
\end{equation}
Define
\begin{equation}
    \Omega_K := \operatorname{diag} \left( 1, \rho(1-\rho)\id_K \right).
    \label{eq:finite-K-Omega}
\end{equation}
Then the replicated finite-$K$ field vector has the asymptotic conditional law
\begin{equation}
    (\bfield_K^1,\ldots,\bfield_K^s)\mid\lambda \Longrightarrow \N_{(K+1)s}\!\left( \sqrt\beta\lambda(\bm M_s\otimes\bm e_0), \mathsf Q_{\Gamma,s}\otimes\Omega_K \right).
    \label{eq:finite-K-replicated-Gaussian-law}
\end{equation}
For $0<\rho<1$, $\Omega_K$ is positive definite. This is why the reduced $(K+1)$-dimensional representation is preferable to the raw $2K$-dimensional visible/masked representation.

\subsubsection{Dynamic-mask output channel}
\label{sec:dynamic-limit}
Dynamic masking is represented by the sequential limit
\begin{equation}
    n,d\to\infty \quad\text{with $K$ fixed}, \qquad\text{followed by}\qquad K\to\infty.
    \label{eq:order-of-limits-finite-K-dynamic}
\end{equation}

For one replica, the reduced Gaussian channel of \Cref{subsec:finite-K-Gaussian-channel} can be represented as
\begin{equation}
    \delta_r = \sqrt{\rho(1-\rho)Q_\Gamma}\,\xi_r, \qquad \xi_r\iid\N(0,1), \qquad r=1,\ldots,K,
    \label{eq:delta-single-replica-gaussian}
\end{equation}
independently of the full preactivation $h$. Substituting this representation into \eqref{eq:finite-K-effective-loss-reduced} gives
\begin{equation}
    \begin{aligned}
    \ell_K =\frac1K\sum_{r=1}^K \Big[& \rho Q_I \sig\!\left( (1-\rho)h +\sqrt{\rho(1-\rho)Q_\Gamma}\,\xi_r \right)^2 \\
        &-2 \sig\!\left( (1-\rho)h +\sqrt{\rho(1-\rho)Q_\Gamma}\,\xi_r \right) \left( \rho h -\sqrt{\rho(1-\rho)Q_\Gamma}\,\xi_r \right) \Big].
    \end{aligned}
    \label{eq:finite-K-loss-gaussian-representation}
\end{equation}
By the law of large numbers,
\begin{equation}
    \begin{aligned}
    &\ell_K\xrightarrow[K\to\infty]{} \ell_{\rm dyn}(h;Q_I,Q_\Gamma):=\E_\xi\Bigg[ \rho Q_I\, \sig\!\left((1-\rho)h +\sqrt{\rho(1-\rho)Q_\Gamma}\,\xi\right)^2\\
    &\qquad-2\, \sig\!\left((1-\rho)h +\sqrt{\rho(1-\rho)Q_\Gamma}\,\xi\right) \left(\rho h -\sqrt{\rho(1-\rho)Q_\Gamma}\,\xi\right) \Bigg].
    \end{aligned}
    \label{eq:dynamic-limit-loss}
\end{equation}
Thus dynamic masking does not introduce a new high-dimensional field. Once the mask average is performed, the only explicit sample-dependent field is the ordinary preactivation $h$; the fluctuations generated by the masks are retained inside the averaged loss \eqref{eq:dynamic-limit-loss}.

The reduced $(K+1)$-dimensional representation is sufficient for the replica calculation. For completeness, the covariance in the redundant $2K$ visible/masked representation is recorded here as a consistency check. The singular raw covariance can be recovered from the reduced variables, but it is not needed for the replica calculation. Let
\begin{equation}
    \bm h_{\vis} = (h_{1,\vis},\ldots,h_{K,\vis})^\top, \qquad \bm h_{\mask} = (h_{1,\mask},\ldots,h_{K,\mask})^\top.
\end{equation}
Then
\begin{equation}
    \begin{pmatrix}
        \bm h_{\vis}\\
        \bm h_{\mask}
    \end{pmatrix}
    = \underbrace{
    \begin{pmatrix}
        (1-\rho)\one_K&\id_K\\
        \rho\one_K&-\id_K
    \end{pmatrix}
    }_{L_K}
    \begin{pmatrix}
        h\\
        \bdelta
    \end{pmatrix}.
    \label{eq:raw-linear-map}
\end{equation}
Suppressing the common replica factor $Q_\Gamma^{\mathsf a\mathsf b}$,
\begin{equation}
    \Omega_K^{\rm raw} = L_K
    \begin{pmatrix}
        1&0\\
        0&\rho(1-\rho)\id_K
    \end{pmatrix}
    L_K^\top,
    \label{eq:raw-covariance-factorization}
\end{equation}
which gives
\begin{equation}
    \Omega_K^{\rm raw} =
    \begin{pmatrix}
        (1-\rho)^2\one_K\one_K^\top +\rho(1-\rho)\id_K & \rho(1-\rho)(\one_K\one_K^\top-\id_K) \\
        \rho(1-\rho)(\one_K\one_K^\top-\id_K) & \rho^2\one_K\one_K^\top +\rho(1-\rho)\id_K
    \end{pmatrix}.
    \label{eq:raw-covariance}
\end{equation}
For $K\ge2$ this matrix is singular because
\begin{equation}
    h_{1,\vis}+h_{1,\mask} =\cdots= h_{K,\vis}+h_{K,\mask}=h.
\end{equation}
For $0<\rho<1$,
\begin{equation}
    \rank(\Omega_K^{\rm raw})=K+1, \qquad \dim\ker(\Omega_K^{\rm raw})=K-1.
\end{equation}
The singularity is therefore only a consequence of using $2K$ coordinates to represent a $(K+1)$-dimensional random object.

\newpage
% =============================================================================

% =============================================================================
\section{Replica derivation and output-channel specializations}
\label{app:replica-analysis}

\subsection{Generic replica derivation}
\label{sec:replica-derivation}
% =============================================================================

The computation is performed once for a generic finite-dimensional output channel. The model-specific information enters only through the field dimension $L$, the signal loading $\bkappa$, the covariance matrix $\Omega$, and the output factor $P_{\out}$. Finite-$K$ and dynamic masking are specified after the generic saddle-point calculation.

\subsubsection{High-level description}
\label{subsec:generic-output-channel}
Let
\begin{equation}
    \bfield(\bw) = (f_1(\bw),\ldots,f_L(\bw))\in\R^L, \qquad L=O_d(1),
    \label{eq:generic-field-vector}
\end{equation}
be a finite collection of sample-dependent projected fields of the form
\begin{equation}
    f_\ell(\bw) = \frac1{\sqrt d} \sum_{i=1}^d r_{\ell i}w_i x_i, \qquad \ell=1,\ldots,L,
    \label{eq:generic-field-definition}
\end{equation}
where $r_{\ell i}$ denotes the sample-disorder coefficient associated with component $\ell$. We assume that, for every fixed number of replicas, the empirical mean and covariance coefficients self-average as
\begin{gather}
    \frac1d\sum_{i=1}^d r_{\ell i}u_i^\star w_i^{\mathsf a} \longrightarrow \kappa_\ell M^{\mathsf a},
    \label{eq:generic-field-mean-assumption}
    \\
    \frac1d\sum_{i=1}^d r_{\ell i}r_{k i}\gamma_i^\star w_i^{\mathsf a}w_i^{\mathsf b} \longrightarrow \Omega_{\ell k}Q_\Gamma^{\mathsf a\mathsf b},
    \label{eq:generic-field-cov-assumption}
\end{gather}
for a deterministic loading vector $\bkappa\in\R^L$ and a positive-semidefinite matrix $\Omega\in\R^{L\times L}$.

The generic sample contribution is encoded through
\begin{equation}
    P_{\out} \left( \bfield,Q_I,Q_\Gamma \right),
    \label{eq:generic-output-factor-definition}
\end{equation}
where the diagonal overlaps $Q_I$ and $Q_\Gamma$ are allowed to enter explicitly. We keep the weight factor
\begin{equation}
    P_w(\bw) = \exp\!\left\{ -\eta_{\rm th}\frac{\laR}{2}\|\bw\|_2^2 \right\}.
    \label{eq:generic-weight-factor}
\end{equation}
Here $\eta_{\rm th}>0$ is the Gibbs inverse temperature. The corresponding partition function is
\begin{equation}
    \calZ(\calD) = \int\dd\bw\,P_w(\bw) \prod_{\mu=1}^n P_{\out}\!\left( \bfield_\mu(\bw), Q_I(\bw), Q_\Gamma(\bw) \right),
    \label{eq:generic-partition-function}
\end{equation}
with
\begin{equation}
    Q_I(\bw)=\frac1d\bw^\top\bw =\frac1d\sum_iw_i^2, \qquad Q_\Gamma(\bw)=\frac1d\bw^\top\Gamma^\star\bw =\frac1d\sum_i\gamma_i^\star w_i^2.
\end{equation}

\subsubsection{Replicated moment and overlap parameters}
The free entropy per dimension is
\begin{equation}
    \Phi = \lim_{d\to\infty} \frac1d \E_{\calD}\log\calZ(\calD),
    \label{eq:free-entropy}
\end{equation}
and the replica method formally uses
\begin{equation}
    \E[\log\calZ] = \lim_{s\to0} \frac{\E[\calZ^s]-1}{s}.
    \label{eq:replica-trick}
\end{equation}
For an integer $s\ge1$, introduce replicas $\bw^1,\ldots,\bw^s$ and the same overlap family as above,
\begin{equation}
    M^{\mathsf a} := \frac1d(\bu^\star)^\top\bw^{\mathsf a}, \qquad Q_A^{\mathsf a\mathsf b} := \frac1d(\bw^{\mathsf a})^\top A\bw^{\mathsf b},
    \label{eq:generic-M-QA-overlaps}
\end{equation}
with the two specializations
\begin{equation}
    Q_I^{\mathsf a\mathsf b} = \frac1d(\bw^{\mathsf a})^\top\bw^{\mathsf b}, \qquad Q_\Gamma^{\mathsf a\mathsf b} = \frac1d(\bw^{\mathsf a})^\top\Gamma^\star\bw^{\mathsf b}.
    \label{eq:generic-QI-QGamma-overlaps}
\end{equation}
Then
\begin{equation}
    \calZ^s = \int \prod_{\mathsf a=1}^s \dd\bw^{\mathsf a}\, P_w(\bw^{\mathsf a}) \prod_{\mu=1}^n \prod_{\mathsf a=1}^s P_{\out}\!\left( \bfield_\mu^{\mathsf a}, Q_I^{\mathsf a\mathsf a}, Q_\Gamma^{\mathsf a\mathsf a} \right).
    \label{eq:integer-moment}
\end{equation}
The output factor of replica $\mathsf a$ depends explicitly only on the diagonal quantities $Q_I^{\mathsf a\mathsf a}$ and $Q_\Gamma^{\mathsf a\mathsf a}$. The off-diagonal overlaps $Q_\Gamma^{\mathsf a\mathsf b}$ enter through the joint Gaussian law because all replicas are evaluated on the same training sample and their projected data-noise fields are correlated. The off-diagonal Euclidean overlaps $Q_I^{\mathsf a\mathsf b}$ do not enter the output covariance directly; they become relevant in the weight-side saddle-point constraints. Because $\gamma_i^\star$ is quenched, the primitive overlap variables are resolved by the variance type $\gamma\sim P_\gamma$. For each replica $a$, define the teacher overlap restricted to coordinates with variance $\gamma$
\begin{equation}
    m^a(\gamma) = \frac{1}{d\,P_\gamma(\dd\gamma)} \sum_{i:\gamma_i^\star\simeq\gamma}u_i^\star w_i^a,
    \label{eq:m-gamma}
\end{equation}
and for each pair $a,b$ define the corresponding type-resolved Euclidean overlap
\begin{equation}
    Q^{ab}(\gamma) = \frac{1}{d\,P_\gamma(\dd\gamma)} \sum_{i:\gamma_i^\star\simeq\gamma}w_i^aw_i^b.
    \label{eq:Q-gamma}
\end{equation}
This notation is understood by first discretizing the support of $P_\gamma$ and then taking the continuum limit. The global overlaps are the two metric moments of this type-resolved overlap:
\begin{equation}
    M^a=\E_\gamma[m^a(\gamma)], \qquad Q_I^{ab}=\E_\gamma[Q^{ab}(\gamma)], \qquad Q_\Gamma^{ab}=\E_\gamma[\gamma Q^{ab}(\gamma)].
    \label{eq:global-overlaps}
\end{equation}
Thus $Q_I^{ab}=Q_{A=I}^{ab}$ is the Euclidean overlap, whereas $Q_\Gamma^{ab}=Q_{A=\Gamma^\star}^{ab}$ is the overlap seen by the heteroskedastic Gaussian noise. In the homoskedastic limit $\Gamma^\star=I$, the two matrices coincide.

Here $Q^{ab}(\gamma)$ is a type-resolved overlap density. It is not a third global metric overlap: $Q_I^{ab}$ and $Q_\Gamma^{ab}$ are the two moments of this same object that are required by the Euclidean decoder geometry and the heteroskedastic field covariance, respectively.

Because $u_i^\star$ and $\gamma_i^\star$ are independent and $\E[(u^\star)^2]=1$, the teacher self-overlap at fixed variance type self-averages to
\begin{equation}
    \frac{1}{d\,P_\gamma(\dd\gamma)} \sum_{i:\gamma_i^\star\simeq\gamma}(u_i^\star)^2 \longrightarrow 1.
\end{equation}
Hence the teacher self-overlap is fixed by the prior and does not introduce an additional variational order parameter.

Conditioning on the quenched coordinate variances $\bm\gamma^\star$, independence across the sample index $\mu$ gives
\begin{equation}
    \E_{\calD\mid\bm\gamma^\star}\calZ^s = \E_{\bu^\star} \int \prod_{\mathsf a=1}^s \dd\bw^{\mathsf a}\, P_w(\bw^{\mathsf a}) \times \left[ \E_{\lambda,\bz,\br} \prod_{\mathsf a=1}^s P_{\out}\!\left( \bfield^{\mathsf a}, Q_I^{\mathsf a\mathsf a}, Q_\Gamma^{\mathsf a\mathsf a} \right) \right]^n
    \label{eq:moment-before-field-reduction}
\end{equation}
The expectation in square brackets is a one-sample quantity. Although it originates from a $d$-dimensional sample, the output factor depends on that sample only through the finite set of $Ls$ replicated fields.

\subsubsection{Gaussian reduction and replicated output factor}
Using the data model \eqref{eq:data-model}, the field of replica $\mathsf a$ decomposes as
\begin{equation}
    f_\ell^{\mathsf a} = \frac{\sqrt\beta\lambda}{d} \sum_{i=1}^d r_{\ell i}u_i^\star w_i^{\mathsf a} + \frac1{\sqrt d} \sum_{i=1}^d r_{\ell i}\sqrt{\gamma_i^\star}\, w_i^{\mathsf a}z_i.
    \label{eq:generic-field-spike-noise-decomposition}
\end{equation}
The first term converges, by \eqref{eq:generic-field-mean-assumption}, to
\begin{equation}
    \sqrt\beta\lambda\,\kappa_\ell M^{\mathsf a}.
\end{equation}
The second term is Gaussian conditional on the coefficients $r_{\ell i}$ because the $z_i$ are Gaussian. Its empirical covariance between components $(\ell,\mathsf a)$ and $(k,\mathsf b)$ is
\begin{equation}
    \frac1d \sum_{i=1}^d r_{\ell i}r_{k i}\gamma_i^\star w_i^{\mathsf a}w_i^{\mathsf b},
\end{equation}
which converges by \eqref{eq:generic-field-cov-assumption} to
\begin{equation}
    \Omega_{\ell k}Q_\Gamma^{\mathsf a\mathsf b}.
\end{equation}
Thus no additional central-limit approximation is needed for the Gaussian data-noise term; only its empirical covariance is replaced by its high-dimensional limit.

Let
\begin{equation}
    \bm M_s=(M^1,\ldots,M^s)^\top, \qquad \mathsf Q_{\Gamma,s}=(Q_\Gamma^{\mathsf a\mathsf b})_{\mathsf a,\mathsf b\le s}.
\end{equation}
Then
\begin{equation}
(\bfield^1,\ldots,\bfield^s)\mid\lambda \Longrightarrow \N_{Ls}\!\left( \sqrt\beta\lambda(\bm M_s\otimes\bkappa), \mathsf Q_{\Gamma,s}\otimes\Omega \right).
    \label{eq:generic-replicated-field-law}
\end{equation}
If $\Omega$ is singular, this notation denotes the corresponding Gaussian probability measure on the support of its covariance matrix.

Define
\begin{equation}
    \Psi_{\out}^{(s)} \left( \bm M_s, \mathsf Q_{\Gamma,s}, \{Q_I^{\mathsf a\mathsf a}\}, \{Q_\Gamma^{\mathsf a\mathsf a}\} \right) := \E_\lambda \E_{(\bfield^1,\ldots,\bfield^s)\mid\lambda} \left[ \prod_{\mathsf a=1}^s P_{\out}\!\left( \bfield^{\mathsf a}, Q_I^{\mathsf a\mathsf a}, Q_\Gamma^{\mathsf a\mathsf a} \right) \right].
    \label{eq:generic-replicated-output}
\end{equation}
The inner expectation is with respect to \eqref{eq:generic-replicated-field-law}. Equivalently,
\begin{equation}
    \begin{aligned}
    \Psi_{\out}^{(s)} ={}& \E_\lambda \int \N_{Ls}\!\left( \dd\bfield^1\cdots\dd\bfield^s; \sqrt\beta\lambda(\bm M_s\otimes\bkappa), \mathsf Q_{\Gamma,s}\otimes\Omega \right) \\
    &\hspace{1.6cm}\times \prod_{\mathsf a=1}^s P_{\out}\!\left( \bfield^{\mathsf a}, Q_I^{\mathsf a\mathsf a}, Q_\Gamma^{\mathsf a\mathsf a} \right).
    \end{aligned}
    \label{eq:generic-replicated-output-integral}
\end{equation}
Hence the replicated moment takes the generic form
\begin{equation}
    \E_{\calD\mid\bm\gamma^\star}\calZ^s = \E_{\bu^\star} \int \prod_{\mathsf a=1}^s \dd\bw^{\mathsf a}\, P_w(\bw^{\mathsf a}) \left[ \Psi_{\out}^{(s)} \right]^n.
    \label{eq:moment-after-field-reduction}
\end{equation}
This is the point at which the model-specific sample disorder has been compressed into a finite-dimensional output contribution. The remainder of the replica saddle-point calculation depends on the learning problem only through $\Psi_{\out}^{(s)}$.

\subsubsection{Overlap constraints and replica action}

The low-dimensional output factor depends on the microscopic replicas only through the type-resolved overlaps. We enforce their definitions with conjugate fields, which decouples the coordinate dependence in the usual way. We enforce the type-resolved student overlaps with their conjugate fields
\begin{equation}
    \begin{aligned}
        1 & = \int\mathcal Dm\,\mathcal DQ\,\mathcal D\widehat m\,\mathcal D\widehat Q \exp\Bigg\{ d\E_\gamma\left[ \sum_{a=1}^s\widehat m^a(\gamma)m^a(\gamma) + \sum_{a\le b}\widehat Q^{ab}(\gamma)Q^{ab}(\gamma) \right] \\
         & \qquad - \sum_{i=1}^d \left[ \sum_{a=1}^s\widehat m^a(\gamma_i^\star)u_i^\star w_i^a + \sum_{a\le b}\widehat Q^{ab}(\gamma_i^\star)w_i^aw_i^b \right] \Bigg\}.
    \end{aligned}
    \label{eq:functional-delta}
\end{equation}
Substituting this into Eq.~\eqref{eq:moment-after-field-reduction} and factorizing, one can show that
\begin{equation}
    \E_{\calD\mid\bm\gamma^\star}\calZ^s = \int\mathcal Dm\,\mathcal DQ\,\mathcal D\widehat m\,\mathcal D\widehat Q \exp\left\{d\left[ \Phi_{\rm trace}^{(s)}+ \Phi_{\rm prior}^{(s)}+ \alpha\Phi_{\out}^{(s)} \right]\right\},
    \label{eq:replicated-action-decomp}
\end{equation}
where

\begin{gather}
    \Phi_{\rm trace}^{(s)} = \E_\gamma\left[ \sum_{a=1}^s\widehat m^a(\gamma)m^a(\gamma) + \sum_{a\le b}\widehat Q^{ab}(\gamma)Q^{ab}(\gamma) \right],
    \label{eq:Phi-trace-s}
    \displaybreak[1] \\
    \begin{aligned}
    \Phi_{\rm prior}^{(s)} = \E_\gamma \log \E_{u^\star\sim P_u} \int\prod_{a=1}^s\dd w^a\,P_w(w^a) \exp\left\{ -\sum_{a=1}^s\widehat m^a(\gamma)u^\star w^a -\sum_{a\le b}\widehat Q^{ab}(\gamma)w^aw^b \right\},
    \end{aligned}
    \label{eq:Phi-prior-s}
    \displaybreak[1] \\
    \Phi_{\out}^{(s)} = \log \Psi_{\out}^{(s)} \left( \bm M_s, \mathsf Q_{\Gamma,s}, \{Q_I^{aa}\}, \{Q_\Gamma^{aa}\} \right).
    \label{eq:Phi-output-s}
\end{gather}

\subsubsection{Replica-symmetric ansatz and $s\to0$ reduction}
\label{sec:RS-ansatz}

We now impose the replica-symmetric (RS) ansatz
\begin{gather}
    m^a(\gamma)=m(\gamma), \\
    Q^{ab}(\gamma) = q(\gamma)+V(\gamma)\delta_{ab}, \qquad a,b=1,\ldots,s.
    \label{eq:RS-Q}
\end{gather}
For the conjugated overlaps we use
\begin{gather}
    \widehat m^a(\gamma)=-\widehat m(\gamma), \\
    \widehat Q^{ab}(\gamma) = \left(\frac{\widehat V(\gamma)-\widehat q(\gamma)}{2} +\widehat q(\gamma)\right)\delta_{ab} -\widehat q(\gamma),
    \label{eq:RS-Qhat}
\end{gather}
Under the RS assumption, we have that
\begin{equation}
    \begin{aligned}
    &-\sum_{a=1}^s\widehat m^a(\gamma)u^\star w^a -\sum_{a\le b}\widehat Q^{ab}(\gamma)w^aw^b\\
    &\quad=\widehat m(\gamma)u^\star\sum_{a=1}^sw^a -\frac{\widehat V(\gamma)}2\sum_{a=1}^s(w^a)^2 +\frac{\widehat q(\gamma)}2 \left(\sum_{a=1}^sw^a\right)^2.
    \end{aligned}
    \label{eq:RS-site-exponent-rewrite}
\end{equation}

and a few additional order parameters
\begin{equation}
    \begin{aligned}
        M          &:=\E_\gamma[m(\gamma)],\\
        q_I        &:=\E_\gamma[q(\gamma)], &\qquad q_\Gamma   &:=\E_\gamma[\gamma q(\gamma)],\\
        V_I        &:=\E_\gamma[V(\gamma)], &\qquad V_\Gamma   &:=\E_\gamma[\gamma V(\gamma)].
    \end{aligned}
    \label{eq:macros-finite-temp}
\end{equation}
Consequently, the two global replica-overlap matrices are
\begin{equation}
    Q_I^{ab}=q_I+V_I\delta_{ab}, \qquad Q_\Gamma^{ab}=q_\Gamma+V_\Gamma\delta_{ab}.
    \label{eq:global-RS-matrices}
\end{equation}

\paragraph{Trace term.} For the type-resolved quadratic overlap
\begin{equation}
    \begin{aligned}
        \sum_{a\le b}\widehat Q^{ab}Q^{ab} ={} & s\,\frac{\widehat V-\widehat q}{2}(q+V) -\frac{s(s-1)}2\widehat q\,q \\
        ={} & \frac s2\left[ \widehat Vq+(\widehat V-\widehat q)V \right] -\frac{s^2}{2}\widehat q\,q.
    \end{aligned}
    \label{eq:RS-trace-algebra-pointwise}
\end{equation}
The trace term can be rewritten as
\begin{equation}
    \Phi_{\rm trace}^{(s)} = s\left\{ -\E_\gamma[\widehat m(\gamma)m(\gamma)] +\frac12\E_\gamma[\widehat V(\gamma)q(\gamma)] +\frac12\E_\gamma[(\widehat V(\gamma)-\widehat q(\gamma))V(\gamma)] \right\} +O(s^2).
    \label{eq:trace-RS}
\end{equation}
\paragraph{Prior term.} Substituting the RS ansatz into Eq.~\eqref{eq:Phi-prior-s}, the only term coupling replicas is quadratic in $\sum_a w^a$.  For each fixed $\gamma$,
\begin{equation}
    \exp\left\{\frac{\widehat q(\gamma)}{2}\left(\sum_{a=1}^s w^a\right)^2\right\} = \E_{\xi\sim\N(0,1)} \exp\left\{\sqrt{\widehat q(\gamma)}\xi\sum_{a=1}^s w^a\right\}.
    \label{eq:HS-prior}
\end{equation}
Define the scalar input partition function
\begin{equation}
    \mathcal Z_{\rm in} (\gamma,u^\star,\xi) :=\int_\R\dd w\,P_w(w) \exp\left\{ \sqrt{\widehat q(\gamma)}\,\xi w -\frac{\widehat V(\gamma)}2w^2 +\widehat m(\gamma)u^\star w \right\}.
    \label{eq:RS-scalar-input-partition}
\end{equation}
Using Eq.~\eqref{eq:HS-prior}, the finite-$s$ site partition becomes exactly
\begin{equation}
    \mathcal Z_{\rm in}^{(s)}(\gamma) =\E_{u^\star,\xi} \left[\mathcal Z_{\rm in}(\gamma,u^\star,\xi)\right]^s.
    \label{eq:RS-finite-s-input-factor}
\end{equation}
Therefore
\begin{equation}
    \Phi_{\rm prior}^{(s)} =\E_\gamma \log\E_{u^\star,\xi} \exp\left\{ s\log\mathcal Z_{\rm in}(\gamma,u^\star,\xi) \right\},
    \label{eq:RS-prior-cumulant-form}
\end{equation}
expanding around $s=0$ gives
\begin{equation}
    \log\E e^{sX} =s\E X+O(s^2).
    \label{eq:small-s-cumulant-expansion}
\end{equation}

Applying this with $X=\log\mathcal Z_{\rm in}$ gives
\begin{equation}
    \Phi_{\rm prior}^{(s)} =s\,\E_{\gamma,u^\star,\xi} \log\int\dd w\,P_w(w) \exp\left\{ \sqrt{\widehat q(\gamma)}\xi w -\frac12\widehat V(\gamma)w^2 +\widehat m(\gamma)u^\star w \right\} +O(s^2).
    \label{eq:prior-RS-finite-temp}
\end{equation}

\paragraph{Output term.}

Under Eq.~\eqref{eq:global-RS-matrices}, the replicated field covariance is
\begin{equation}
    (q_\Gamma\one_s\one_s^\top+V_\Gamma\id_s)\otimes\Omega.
\end{equation}
It can be generated by one common Gaussian vector $\bz\sim\N(0,\id_L)$ and independent replica-specific vectors $\bz^a\iid\N(0,\id_L)$:
\begin{equation}
    \bfield^a =\sqrt\beta\lambda\bkappa M +\sqrt{q_\Gamma}\,\Omega^{1/2}\bz +\sqrt{V_\Gamma}\,\Omega^{1/2}\bz^a.
    \label{eq:RS-field-common-private-decomposition}
\end{equation}
Define the common field
\begin{equation}
    \bfield_0 :=\sqrt\beta\lambda\bkappa M +\sqrt{q_\Gamma}\,\Omega^{1/2}\bz, \qquad \bz\sim\N(0,\id_L).
    \label{eq:f0-generic}
\end{equation}
Conditionally on $(\lambda,\bz)$, the replicas are independent and
\begin{equation}
    \bfield^a\mid\lambda,\bz \iid\N(\bfield_0,V_\Gamma\Omega).
\end{equation}
The diagonal norm arguments are
\begin{equation}
    Q_I^{aa}=q_I+V_I, \qquad Q_\Gamma^{aa}=q_\Gamma+V_\Gamma.
    \label{eq:RS-diagonal-norms}
\end{equation}
Define the scalar output partition function
\begin{equation}
    \mathcal Z_{\out}(\bfield_0) :=\int\dd\bfield\, \N(\bfield;\bfield_0,V_\Gamma\Omega) P_{\out}(\bfield,q_I+V_I,q_\Gamma+V_\Gamma).
    \label{eq:RS-scalar-output-partition}
\end{equation}
Then the replicated output factor is exactly
\begin{equation}
    \Psi_{\out}^{(s)} =\E_{\lambda,\bz} \left[\mathcal Z_{\out}(\bfield_0)\right]^s,
    \label{eq:RS-finite-s-output-factor}
\end{equation}
so that the same cumulant expansion gives
\begin{equation}
    \Phi_{\out}^{(s)} =s\,\E_{\lambda,\bz} \log\int\dd\bfield\, \N(\bfield;\bfield_0,V_\Gamma\Omega) P_{\out}(\bfield,q_I+V_I,q_\Gamma+V_\Gamma) +O(s^2).
    \label{eq:output-RS-finite-temp}
\end{equation}

\subsubsection{Replica-symmetric free entropy}
\label{sec:zero-replica}

Using the analytic continuation around $s=0$, the replica trick can equivalently be written as
\begin{equation}
    \Phi = \lim_{s\to0}\frac1s \lim_{d\to\infty}\frac1d \log \E_{\calD}[\calZ^s],
    \label{eq:replica-trick-log-form}
\end{equation}
where the equality is understood through the analytic continuation around $s=0$.

At fixed $s$, the large-$d$ saddle gives
\begin{equation}
    \lim_{d\to\infty}\frac1d \log \E_{\calD}[\calZ^s] = \Extr_{m,q,V,\widehat m,\widehat q,\widehat V} \left\{ \Phi_{\rm trace}^{(s)} +\Phi_{\rm prior}^{(s)} +\alpha\Phi_{\out}^{(s)} \right\}.
    \label{eq:finite-s-RS-saddle}
\end{equation}
Here the extremization is over the $\gamma$-resolved functions $m(\gamma),q(\gamma),V(\gamma)$ and their conjugates $\widehat m(\gamma),\widehat q(\gamma),\widehat V(\gamma)$.

Equations~\eqref{eq:trace-RS}, \eqref{eq:prior-RS-finite-temp}, and \eqref{eq:output-RS-finite-temp} have the common small-$s$ structure
\begin{equation}
    \Phi_X^{(s)} =s\,\Phi_X+O(s^2), \qquad X\in\{\mathrm{trace},\mathrm{prior},\out\}.
    \label{eq:small-s-RS-structure}
\end{equation}
Therefore, dividing Eq.~\eqref{eq:finite-s-RS-saddle} by $s$ and taking $s\to0$ yields a single replica-symmetric saddle for the quenched free entropy,
\begin{equation}
    \Phi = \Extr_{m(\cdot),q(\cdot),V(\cdot), \widehat m(\cdot),\widehat q(\cdot),\widehat V(\cdot)} \left\{ \Phi_{\rm trace} +\Phi_{\rm prior} +\alpha\Phi_{\out} \right\}
    \label{eq:finite-temp-RS-variational}
\end{equation}
The three $s\to0$ contributions are
\begin{equation}
    \Phi_{\rm trace} = -\E_\gamma[\widehat m(\gamma)m(\gamma)] +\frac12\E_\gamma[\widehat V(\gamma)q(\gamma)] \quad +\frac12\E_\gamma[ (\widehat V(\gamma)-\widehat q(\gamma))V(\gamma) ].
    \label{eq:RS-free-entropy-trace}
\end{equation}
For the input contribution,
\begin{equation}
    \Phi_{\rm prior} = \E_{\gamma,u^\star,\xi} \log \mathcal Z_{\rm in}(\gamma,u^\star,\xi),
    \label{eq:RS-free-entropy-prior}
\end{equation}
with $\mathcal Z_{\rm in}$ defined in Eq.~\eqref{eq:RS-scalar-input-partition}.  For the output contribution,
\begin{equation}
    \Phi_{\out} = \E_{\lambda,\bz} \log \mathcal Z_{\out}(\bfield_0),
    \label{eq:RS-free-entropy-output}
\end{equation}
where
\begin{gather}
    \mathcal Z_{\out}(\bfield_0) = \int\dd\bfield\, \N(\bfield;\bfield_0,V_\Gamma\Omega) P_{\out}(\bfield,q_I+V_I,q_\Gamma+V_\Gamma),
    \label{eq:RS-free-entropy-Zout}
    \\
    \bfield_0 = \sqrt\beta\lambda\bkappa M +\sqrt{q_\Gamma}\,\Omega^{1/2}\bz.
    \label{eq:RS-free-entropy-common-field}
\end{gather}
The global metric moments entering the output channel remain
\begin{equation}
    q_I=\E_\gamma[q(\gamma)], \qquad q_\Gamma=\E_\gamma[\gamma q(\gamma)], \qquad V_I=\E_\gamma[V(\gamma)], \qquad V_\Gamma=\E_\gamma[\gamma V(\gamma)].
    \label{eq:RS-free-entropy-metric-moments}
\end{equation}
The replicated problem has therefore reduced to a single RS variational saddle. The zero-temperature ERM limit is taken only after this $s\to0$ reduction.

\subsubsection{Zero-temperature ERM limit}
\label{sec:zero-temp}
We will specialize our equations for ridge regularization:
\begin{equation}
    P_w(w)=\exp\left\{-\eta_{\rm th}\frac{\laR}{2}w^2\right\}, \qquad P_{\out}(\bfield,Q_I,Q_\Gamma)=\exp\{-\eta_{\rm th}\ell_{\out}(\bfield;Q_I,Q_\Gamma)\}.
\end{equation}
We rescale the replica-diagonal excess overlap as
\begin{gather}
    V(\gamma):=\frac{v(\gamma)}{\eta_{\rm th}}, \qquad V_I=\frac{v_I}{\eta_{\rm th}}, \quad v_I:=\E_\gamma[v(\gamma)], \qquad V_\Gamma=\frac{v_\Gamma}{\eta_{\rm th}}, \quad v_\Gamma:=\E_\gamma[\gamma v(\gamma)].
    \label{eq:zero-temp-V}
    \\
    \widehat m(\gamma)\mapsto\eta_{\rm th}\widehat m(\gamma), \qquad \widehat V(\gamma)\mapsto\eta_{\rm th}\widehat v(\gamma), \qquad \widehat q(\gamma)\mapsto\eta_{\rm th}^2\widehat q(\gamma).
    \label{eq:zero-temp-hats}
\end{gather}

\paragraph{Input proximal.} Using Laplace's method in Eq.~\eqref{eq:prior-RS-finite-temp}, the prior term contribution at leading order is
\begin{equation}
    \Phi_{\rm prior}^{0} = -\E_{\gamma,u^\star,\xi} \min_{w\in\R} \left\{ \frac{\laR+\widehat v(\gamma)}{2}w^2 - \left[\widehat m(\gamma)u^\star+\sqrt{\widehat q(\gamma)}\xi\right]w \right\}.
    \label{eq:prior-prox-potential}
\end{equation}
Defining the coordinate proximal
\begin{equation}
    w_\star(\gamma,u^\star,\xi) = \argmin_{w\in\R} \left\{ \frac{\laR+\widehat v(\gamma)}{2}w^2 - \left[\widehat m(\gamma)u^\star+\sqrt{\widehat q(\gamma)}\xi\right]w \right\}.
    \label{eq:w-prox}
\end{equation}
For ridge regularization, this gives
\begin{equation}
    w_\star(\gamma,u^\star,\xi) = \frac{\widehat m(\gamma)u^\star+\sqrt{\widehat q(\gamma)}\xi} {\laR+\widehat v(\gamma)}
    \label{eq:wstar-ridge-style}
\end{equation}
Stationarity with respect to the conjugates gives the first saddle point equations
\begin{gather}
    m(\gamma)=\E_{u^\star,\xi}[u^\star w_\star(\gamma,u^\star,\xi)], \qquad q(\gamma)=\E_{u^\star,\xi}[w_\star(\gamma,u^\star,\xi)^2],
    \label{eq:coordinate-mq-general}
    \\
    v(\gamma) = \E_{u^\star,\xi}\left[\frac{\xi w_\star(\gamma,u^\star,\xi)}{\sqrt{\widehat q(\gamma)}}\right].
    \label{eq:coordinate-v-general}
\end{gather}
For Eq.~\eqref{eq:wstar-ridge-style}, this reduces to
\begin{equation}
    m(\gamma)=\frac{\widehat m(\gamma)}{\laR+\widehat v(\gamma)}, \qquad q(\gamma)=\frac{\widehat m(\gamma)^2+\widehat q(\gamma)}{(\laR+\widehat v(\gamma))^2}, \qquad v(\gamma)=\frac1{\laR+\widehat v(\gamma) }
    \label{eq:coordinate-ridge-equations}
\end{equation}

\paragraph{Output proximal.} For the zero-temperature output analysis we assume that
\begin{equation}
    \Omega\succ0.
\end{equation}
This is satisfied by all output channels considered in the present work. In particular, for finite-$K$ masking,
\begin{equation}
    \Omega_K = \operatorname{diag}\!\left( 1,\rho(1-\rho)\id_K \right) \succ0, \qquad 0<\rho<1,
\end{equation}
while for dynamic masking $\Omega=1$. Hence $\Omega^{-1}$ is well defined in the Moreau-envelope representation below.

The output term Eq.~\eqref{eq:output-RS-finite-temp} becomes, after using $V_\Gamma=v_\Gamma/\eta_{\rm th}$,
\begin{equation}
    \int\dd\bfield\,\mathcal N\left(\bfield;\bfield_0,\frac{v_\Gamma}{\eta_{\rm th}}\Omega\right) \exp\{-\eta_{\rm th}\ell_{\out}(\bfield;q_I,q_\Gamma)\}.
\end{equation}
Laplace's method gives the standard Moreau envelope
\begin{equation}
    \mathcal M_{\out}(\bfield_0;q_I,q_\Gamma,v_\Gamma) = \min_{\bfield\in\R^L} \left\{ \ell_{\out}(\bfield;q_I,q_\Gamma) + \frac1{2v_\Gamma}(\bfield-\bfield_0)^\top\Omega^{-1}(\bfield-\bfield_0) \right\}.
    \label{eq:generic-output-Moreau}
\end{equation}
Let $\bfield_\star$ be any minimizer.  In the $\eta_{\rm th} \rightarrow \infty$ limit we have that
\begin{equation}
    \Phi_{\out}^{0} = -\E_{\lambda,\bz} \left[\mathcal M_{\out}(\bfield_0;q_I,q_\Gamma,v_\Gamma)\right],
    \label{eq:Phi-output-zero-temp}
\end{equation}
with $\bfield_0$ given by Eq.~\eqref{eq:f0-generic}.  Define
\begin{equation}
    \bm F = \frac1{v_\Gamma}\Omega^{-1}(\bfield_\star-\bfield_0).
    \label{eq:generic-force}
\end{equation}
hence
\begin{equation}
    \bm F+\nabla_{\bfield}\ell_{\out}(\bfield_\star;q_I,q_\Gamma)=0.
    \label{eq:generic-output-foc}
\end{equation}
Since $\partial_M\bfield_0=\sqrt\beta\lambda\bkappa$, we have that
\begin{equation}
    A_\star = \bkappa^\top\bm F.
    \label{eq:Astar-generic}
\end{equation}
Moreover,
\begin{equation}
    \frac{\partial\mathcal M_{\out}}{\partial v_\Gamma} = -\frac12\bm F^\top\Omega\bm F.
\end{equation}
For the derivative with respect to the noise-weighted off-diagonal overlap $q_\Gamma$, define
\begin{equation}
    D_\Gamma(\bfield_\star) = \frac{\partial\mathcal M_{\out}}{\partial q_\Gamma} = \partial_{q_\Gamma}\ell_{\out}(\bfield_\star;q_I,q_\Gamma) - \bm F^\top\partial_{q_\Gamma}\bfield_0,
    \label{eq:generic-DGamma}
\end{equation}
where
\begin{equation}
    \partial_{q_\Gamma}\bfield_0 = \frac1{2\sqrt{q_\Gamma}}\Omega^{1/2}\bz.
\end{equation}

\paragraph{Saddle-point equations.} At order $\eta_{\rm th}$, the RS extremization gives
\begin{equation}
    \Phi_0 = \Extr_{m,q,v,\widehat m,\widehat q,\widehat v} \Bigg\{ \Phi_{\rm trace}^0 +\Phi_{\rm prior}^0 +\alpha\Phi_{\out}^0 \Bigg\},
    \label{eq:zero-temp-variational}
\end{equation}
where
\begin{equation}
    \Phi_{\rm trace}^0 = -\E_\gamma[\widehat m(\gamma)m(\gamma)] -\frac12\E_\gamma[\widehat q(\gamma)v(\gamma)] +\frac12\E_\gamma[\widehat v(\gamma)q(\gamma)].
    \label{eq:trace-zero-temp}
\end{equation}
with
\begin{equation}
    \widehat m(\gamma) = \alpha\sqrt\beta\, \E_{\lambda,\bz}[\lambda A_\star]
    \label{eq:hatm-generic-style}
\end{equation}
The conjugate overlap associated with the replica-diagonal excess follows
\begin{equation}
    \widehat q(\gamma) = \alpha\gamma\, \E_{\lambda,\bz}[\bm F^\top\Omega\bm F]
    \label{eq:hatq-generic-style}
\end{equation}
Because $q(\gamma)$ enters the output channel through both $q_I=\E_\gamma[q(\gamma)]$ and $q_\Gamma=\E_\gamma[\gamma q(\gamma)]$, we obtain
\begin{equation}
    \widehat v(\gamma) = 2\alpha\, \E_{\lambda,\bz} \left[ \partial_{q_I}\ell_{\out}(\bfield_\star;q_I,q_\Gamma) + \gamma D_\Gamma(\bfield_\star) \right]
    \label{eq:hatv-generic-style}
\end{equation}
Equations \eqref{eq:coordinate-mq-general}--\eqref{eq:coordinate-v-general} and \eqref{eq:hatm-generic-style}--\eqref{eq:hatv-generic-style}, together with the metric moments, give
\begin{equation}
    \begin{aligned}
        M&=\E_\gamma[m(\gamma)],\\
        q_I&=\E_\gamma[q(\gamma)], &q_\Gamma&=\E_\gamma[\gamma q(\gamma)],\\
        v_I&=\E_\gamma[v(\gamma)], &v_\Gamma&=\E_\gamma[\gamma v(\gamma)].
    \end{aligned}
    \label{eq:final-closures-generic}
\end{equation}

\subsection{Output-channel specializations}
\label{sec:model-specializations}
% =============================================================================

The generic computation of \Cref{sec:replica-derivation} applies to both masking schemes.  The weight-side saddle equations are common to the two models; only the output channel and the corresponding conjugate equations need to be specialized.

% -----------------------------------------------------------------------------
\subsubsection{Finite-$K$ specialization}
\label{subsec:finite-K-specialization}
% -----------------------------------------------------------------------------

For the finite-$K$ masked autoencoder,
\begin{gather}
    L=K+1, \qquad \bfield=\bfield_K=(h,\delta_1,\ldots,\delta_K), \qquad \bkappa=\bm e_0,
    \label{eq:finite-K-generic-specialization-1}
    \\
    \Omega =\Omega_K = \operatorname{diag}\!\left( 1,\rho(1-\rho)\id_K \right).
    \label{eq:finite-K-generic-specialization-2}
\end{gather}
The output factor is
\begin{equation}
    P_{\out,K}(\bfield_K;Q_I) = \exp\!\left\{ -\eta_{\rm th}\ell_K(\bfield_K;Q_I) \right\},
    \label{eq:finite-K-output-factor}
\end{equation}
where $\ell_K$ is given in \eqref{eq:finite-K-effective-loss-reduced}.  It has no additional explicit dependence on $Q_\Gamma$; $Q_\Gamma$ enters through the Gaussian law of the fields.

Substituting
\eqref{eq:finite-K-generic-specialization-1}--%
\eqref{eq:finite-K-output-factor} into the generic replicated output factor gives
\begin{equation}
    \Psi_{K,\out}^{(s)} := \E_\lambda \E_{(\bfield_K^1,\ldots,\bfield_K^s)\mid\lambda} \left[ \prod_{\mathsf a=1}^s \exp\!\left\{ -\eta_{\rm th} \ell_K\!\left( \bfield_K^{\mathsf a}; Q_I^{\mathsf a\mathsf a} \right) \right\} \right],
    \label{eq:finite-K-replicated-output-specialization}
\end{equation}
where
\begin{equation}
    (\bfield_K^1,\ldots,\bfield_K^s)\mid\lambda \Longrightarrow \N_{(K+1)s}\!\left( \sqrt\beta\lambda(\bm M_s\otimes\bm e_0), \mathsf Q_{\Gamma,s}\otimes\Omega_K \right).
    \label{eq:finite-K-replicated-channel-specialization}
\end{equation}

\paragraph{Saddle-point specialization.}

Let $\bfield_\star$ denote the finite-$K$ output proximal associated with $\ell_K$ through the generic Moreau problem \eqref{eq:generic-output-Moreau}, and let $\bm F_K$ denote the corresponding force defined in \eqref{eq:generic-force}.  Since the signal is loaded only along the full-field direction,
\begin{equation}
    A_\star = \bm e_0^\top\bm F_K.
    \label{eq:finite-K-Astar}
\end{equation}
Moreover,
\begin{equation}
    \partial_{q_I}\ell_K(\bfield_\star;q_I) = \frac{\rho}{K} \sum_{r=1}^K \sig\!\left( (1-\rho)h_\star+\delta_{r,\star} \right)^2.
    \label{eq:finite-K-dqI}
\end{equation}
Since $\ell_K$ has no explicit dependence on $q_\Gamma$, \eqref{eq:generic-DGamma} reduces to
\begin{equation}
    D_\Gamma(\bfield_\star) = - \frac{1}{2\sqrt{q_\Gamma}}\, \bm F_K^\top\Omega_K^{1/2}\bz.
    \label{eq:finite-K-DGamma}
\end{equation}

The conjugate saddle equations \eqref{eq:hatm-generic-style}--\eqref{eq:hatv-generic-style} therefore become
\begin{gather}
    \widehat m(\gamma) = \alpha\sqrt\beta\, \E_{\lambda,\bz} \left[ \lambda\, \bm e_0^\top\bm F_K \right],
    \label{eq:finite-K-hatm}
    \\
    \widehat q(\gamma) = \alpha\gamma\, \E_{\lambda,\bz} \left[ \bm F_K^\top \Omega_K \bm F_K \right],
    \label{eq:finite-K-hatq}
    \\
    \widehat v(\gamma) = 2\alpha\, \E_{\lambda,\bz} \Bigg[ \frac{\rho}{K} \sum_{r=1}^K \sig\!\left( (1-\rho)h_\star+\delta_{r,\star} \right)^2 - \frac{\gamma}{2\sqrt{q_\Gamma}}\, \bm F_K^\top \Omega_K^{1/2}\bz \Bigg].
    \label{eq:finite-K-hatv}
\end{gather}
Together with the input-side equations \eqref{eq:coordinate-ridge-equations} and the global closures \eqref{eq:final-closures-generic}, these equations close the finite-$K$ saddle.

Static masking is the special case $K=1$.

% -----------------------------------------------------------------------------
\subsubsection{Dynamic-mask specialization}
\label{subsec:dynamic-specialization}
% -----------------------------------------------------------------------------

After the mask average has been taken, the dynamic-mask loss \eqref{eq:dynamic-limit-loss} depends explicitly only on the full preactivation $h$.  The generic channel therefore reduces to
\begin{equation}
    L=1, \qquad \bfield=h, \qquad \bkappa=1, \qquad \Omega=1.
    \label{eq:dynamic-generic-specialization}
\end{equation}
The corresponding output factor is
\begin{equation}
    P_{{\rm dyn},\out}(h;Q_I,Q_\Gamma) = \exp\!\left\{ -\eta_{\rm th}\ell_{\rm dyn}(h;Q_I,Q_\Gamma) \right\}.
    \label{eq:dynamic-output-factor}
\end{equation}
Thus the generic replicated output factor becomes
\begin{equation}
    \Psi_{{\rm dyn},\out}^{(s)} = \E_\lambda \E_{\bm h_s\mid\lambda} \left[ \prod_{\mathsf a=1}^s \exp\!\left\{ -\eta_{\rm th} \ell_{\rm dyn}\!\left( h^{\mathsf a}; Q_I^{\mathsf a\mathsf a}, Q_\Gamma^{\mathsf a\mathsf a} \right) \right\} \right],
    \label{eq:dynamic-replicated-output-factor}
\end{equation}
where
\begin{equation}
    \bm h_s\mid\lambda \sim \N_s\!\left( \sqrt\beta\lambda\bm M_s, \mathsf Q_{\Gamma,s} \right).
    \label{eq:dynamic-replicated-h-channel}
\end{equation}

\paragraph{Saddle-point specialization.}

Let $h_\star$ denote the scalar output proximal associated with $\ell_{\rm dyn}$ through \eqref{eq:generic-output-Moreau}, and let $F$ be the corresponding scalar force.  Since $\bkappa=1$,
\begin{equation}
    A_\star=F.
    \label{eq:dynamic-Astar}
\end{equation}
The explicit derivative with respect to the Euclidean overlap is
\begin{equation}
    \partial_{q_I} \ell_{\rm dyn}(h;q_I,q_\Gamma) = \rho \E_\xi \left[ \sig\!\left( (1-\rho)h + \sqrt{\rho(1-\rho)q_\Gamma}\,\xi \right)^2 \right].
    \label{eq:dynamic-dqI}
\end{equation}
Unlike the finite-$K$ loss, $\ell_{\rm dyn}$ depends explicitly on $q_\Gamma$.  Hence
\begin{equation}
    D_\Gamma(h_\star) = \partial_{q_\Gamma} \ell_{\rm dyn}(h_\star;q_I,q_\Gamma) - \frac{Fz}{2\sqrt{q_\Gamma}},
    \label{eq:dynamic-DGamma}
\end{equation}
where
\begin{equation}
    \partial_{q_\Gamma} \ell_{\rm dyn}(h;q_I,q_\Gamma) = \sqrt{ \frac{\rho(1-\rho)}{q_\Gamma} }\, \E_\xi \Bigg[ \xi \Bigg( \rho q_I \sig(y)\sig'(y) - t\,\sig'(y) + \sig(y) \Bigg) \Bigg],
    \label{eq:dynamic-dqGamma}
\end{equation}
with, only within this equation,
\begin{equation}
    y = (1-\rho)h + \sqrt{\rho(1-\rho)q_\Gamma}\,\xi, \qquad t = \rho h - \sqrt{\rho(1-\rho)q_\Gamma}\,\xi.
    \label{eq:dynamic-y-t}
\end{equation}

The conjugate saddle equations are therefore
\begin{gather}
    \widehat m(\gamma) = \alpha\sqrt\beta\, \E_{\lambda,z} \left[ \lambda F \right],
    \label{eq:dynamic-hatm}
    \\
    \widehat q(\gamma) = \alpha\gamma\, \E_{\lambda,z} \left[ F^2 \right],
    \label{eq:dynamic-hatq}
    \\
    \widehat v(\gamma) = 2\alpha\, \E_{\lambda,z} \left[ \partial_{q_I} \ell_{\rm dyn}(h_\star;q_I,q_\Gamma) + \gamma \left( \partial_{q_\Gamma} \ell_{\rm dyn}(h_\star;q_I,q_\Gamma) - \frac{Fz}{2\sqrt{q_\Gamma}} \right) \right].
    \label{eq:dynamic-hatv}
\end{gather}
Together with the same input-side equations \eqref{eq:coordinate-ridge-equations} and global closures \eqref{eq:final-closures-generic}, these equations close the dynamic-mask saddle.

The expressions containing $q_\Gamma^{-1/2}$ are understood through their continuous limit when $q_\Gamma\downarrow0$.

% -----------------------------------------------------------------------------
\subsubsection{The limit $K\to\infty$ at the replicated output level}
\label{subsec:replica-dynamic-limit}
% -----------------------------------------------------------------------------

It remains to connect the two specializations.  For fixed positive integer $s$, write
\begin{equation}
    \bm h_s=(h^1,\ldots,h^s)^\top, \qquad \bm\delta_r=(\delta_r^1,\ldots,\delta_r^s)^\top.
\end{equation}
Under the finite-$K$ Gaussian channel,
\begin{equation}
    \bm h_s\mid\lambda \sim \N_s\!\left( \sqrt\beta\lambda\bm M_s, \mathsf Q_{\Gamma,s} \right),
    \label{eq:replicated-h-dynamic-limit}
\end{equation}
while
\begin{equation}
    \bm\delta_r \iid \N_s\!\left( \bm0, \rho(1-\rho)\mathsf Q_{\Gamma,s} \right), \qquad r=1,\ldots,K,
    \label{eq:replicated-delta-iid}
\end{equation}
and the $\bm\delta_r$ are independent of $\bm h_s$.  The components of a fixed $\bm\delta_r$ are generally correlated across replicas, but the fluctuation vectors are i.i.d. across the mask index $r$.

Define
\begin{equation}
    q_{\mathsf a}(h,\delta) := \rho Q_I^{\mathsf a\mathsf a} \sig\!\left((1-\rho)h+\delta\right)^2 - 2 \sig\!\left((1-\rho)h+\delta\right) (\rho h-\delta).
    \label{eq:single-mask-replica-loss}
\end{equation}
Then
\begin{equation}
    \prod_{\mathsf a=1}^s P_{\out,K}\!\left( \bfield_K^{\mathsf a}; Q_I^{\mathsf a\mathsf a} \right) = \exp\!\left\{ -\eta_{\rm th}\frac1K \sum_{r=1}^K \sum_{\mathsf a=1}^s q_{\mathsf a}\!\left( h^{\mathsf a}, \delta_r^{\mathsf a} \right) \right\}.
    \label{eq:replicated-output-mask-average}
\end{equation}
Conditioning on $\bm h_s$ and applying the law of large numbers across the independent mask index gives
\begin{equation}
    \frac1K \sum_{r=1}^K \sum_{\mathsf a=1}^s q_{\mathsf a}\!\left( h^{\mathsf a}, \delta_r^{\mathsf a} \right) \xrightarrow[K\to\infty]{} \sum_{\mathsf a=1}^s \E_{\delta^{\mathsf a}} q_{\mathsf a}\!\left( h^{\mathsf a}, \delta^{\mathsf a} \right).
    \label{eq:replicated-LLN}
\end{equation}
Only the marginal distribution of each replica component appears on the right-hand side.  From \eqref{eq:replicated-delta-iid},
\begin{equation}
    \delta^{\mathsf a} \sim \N\!\left( 0, \rho(1-\rho)Q_\Gamma^{\mathsf a\mathsf a} \right),
\end{equation}
and therefore
\begin{equation}
    \E_{\delta^{\mathsf a}} q_{\mathsf a}\!\left( h^{\mathsf a}, \delta^{\mathsf a} \right) = \ell_{\rm dyn}\!\left( h^{\mathsf a}; Q_I^{\mathsf a\mathsf a}, Q_\Gamma^{\mathsf a\mathsf a} \right).
    \label{eq:replica-dynamic-loss-identification}
\end{equation}
Consequently, under the usual integrability conditions that allow the limit to pass through the remaining Gaussian expectation,
\begin{equation}
    \Psi_{K,\out}^{(s)} \xrightarrow[K\to\infty]{} \Psi_{{\rm dyn},\out}^{(s)}.
    \label{eq:replicated-output-K-to-infinity}
\end{equation}
This establishes the dynamic-mask limit directly inside the same generic replica framework.  The field dimension does not need to be sent formally from $K+1$ to infinity inside the generic derivation.  Rather, the derivation is applied at each fixed $K$, and the resulting finite-$K$ output functional converges to the $L=1$ dynamic output functional.

The sequential interpretation is therefore
\begin{equation}
    \begin{gathered}
    d,n\to\infty \quad\text{at fixed }K, \\
    \text{generic replica reduction with }L=K+1, \\
    K\to\infty: \quad \Psi_{K,\out}^{(s)} \longrightarrow \Psi_{{\rm dyn},\out}^{(s)}, \\
    \text{followed by the usual replica continuation }s\to0.
    \end{gathered}
    \label{eq:replica-order-of-limits}
\end{equation}
Since the weight prior and the definitions of $M$, $Q_I$, and $Q_\Gamma$ are unchanged, the remainder of the saddle-point calculation is common to both finite-$K$ and dynamic masking; only the model-specific output functional changes.

% -----------------------------------------------------------------------------
\subsubsection{Full-reconstruction autoencoder}
\label{subsec:full-reconstruction-specialization}
% -----------------------------------------------------------------------------

We derive the output channel for the full-reconstruction objective \eqref{eq:main-full-objective}, in which both the encoder input and the reconstruction targets include all coordinates. For one sample, the reconstruction loss is
\begin{equation}
    \sum_{i=1}^d \left[ x_i - \frac{w_i}{\sqrt d}\sig(h) \right]^2, \qquad h = \frac1{\sqrt d} \sum_{i=1}^d w_i x_i.
\end{equation}
Dropping the term independent of $\bw$ gives the effective output loss
\begin{equation}
    \ell_{\rm full}(h;Q_I) = Q_I\sig(h)^2 - 2h\sig(h).
    \label{eq:full-reconstruction-loss}
\end{equation}
This is the full-reconstruction counterpart of the masked output loss \eqref{eq:finite-K-effective-loss-raw}: the encoder and target fields are both $h$, and the squared decoder norm is $Q_I$ rather than the masked fraction $\rho Q_I$. In particular, the cross term is $-2h\sig(h)$ instead of $-2h_{r,\mask}\sig(h_{r,\vis})$, making explicit the difference between reconstructing a field from itself and predicting its hidden part from its visible part.

The generic output channel therefore specializes to
\begin{equation}
    L=1, \qquad \bfield=h, \qquad \bkappa=1, \qquad \Omega=1.
    \label{eq:full-reconstruction-channel}
\end{equation}
The corresponding output factor is
\begin{equation}
    P_{{\rm full},\out}(h;Q_I) = \exp\!\left\{ -\eta_{\rm th} \ell_{\rm full}(h;Q_I) \right\}.
    \label{eq:full-reconstruction-output-factor}
\end{equation}
Hence the replicated output contribution is
\begin{equation}
    \Psi_{{\rm full},\out}^{(s)} := \E_\lambda \E_{\bm h_s\mid\lambda} \left[ \prod_{\mathsf a=1}^s \exp\!\left\{ -\eta_{\rm th} \ell_{\rm full}\!\left( h^{\mathsf a}; Q_I^{\mathsf a\mathsf a} \right) \right\} \right],
    \label{eq:full-reconstruction-replicated-output}
\end{equation}
where
\begin{equation}
    \bm h_s\mid\lambda \sim \N_s\!\left( \sqrt\beta\lambda\bm M_s, \mathsf Q_{\Gamma,s} \right).
    \label{eq:full-reconstruction-replicated-channel}
\end{equation}

\paragraph{Saddle-point specialization.}

Let $h_\star$ denote the scalar output proximal obtained from the generic Moreau problem \eqref{eq:generic-output-Moreau}, and let $F$ denote the corresponding scalar force defined in \eqref{eq:generic-force}. Since $\bkappa=1$,
\begin{equation}
    A_\star=F.
    \label{eq:full-reconstruction-Astar}
\end{equation}
The first-order condition of the proximal is
\begin{equation}
    F + 2q_I \sig(h_\star)\sig'(h_\star) - 2\sig(h_\star) - 2h_\star\sig'(h_\star) = 0.
    \label{eq:full-reconstruction-prox-foc}
\end{equation}

The loss depends explicitly on $q_I$ through
\begin{equation}
    \partial_{q_I} \ell_{\rm full}(h_\star;q_I) = \sig(h_\star)^2.
    \label{eq:full-reconstruction-dqI}
\end{equation}
It has no explicit dependence on $q_\Gamma$. Therefore \eqref{eq:generic-DGamma} reduces to
\begin{equation}
    D_\Gamma(h_\star) = - \frac{Fz}{2\sqrt{q_\Gamma}},
    \label{eq:full-reconstruction-DGamma}
\end{equation}
where the common field is
\begin{equation}
    h_0 = \sqrt\beta\,\lambda M + \sqrt{q_\Gamma}\,z, \qquad z\sim\N(0,1).
\end{equation}

Substitution into \eqref{eq:hatm-generic-style}--\eqref{eq:hatv-generic-style} gives
\begin{gather}
    \widehat m(\gamma) = \alpha\sqrt\beta\, \E_{\lambda,z} \left[ \lambda F \right],
    \label{eq:full-reconstruction-hatm}
    \\
    \widehat q(\gamma) = \alpha\gamma\, \E_{\lambda,z} \left[ F^2 \right],
    \label{eq:full-reconstruction-hatq}
    \\
    \widehat v(\gamma) = 2\alpha \E_{\lambda,z} \left[ \sig(h_\star)^2 - \gamma \frac{Fz}{2\sqrt{q_\Gamma}} \right].
    \label{eq:full-reconstruction-hatv}
\end{gather}
Together with the common input-side equations \eqref{eq:coordinate-ridge-equations} and the global closures \eqref{eq:final-closures-generic}, these equations close the full-reconstruction saddle.

\newpage
% =============================================================================

\section{Replica recovery predictions and the role of mask diversity}
\label{app:replica-recovery}

This appendix explains how mask diversity changes the onset of feature recovery and how dynamic masking exploits heterogeneous coordinate noise. We derive the recovery criterion and its factorization in \Cref{app:fixed-beta-comparison}, obtain the weak-signal cost of mask reuse in \Cref{subsec:ridgeless-weak-signal-thresholds}, and examine its finite-signal behavior in \Cref{app:threshold-ratio-assessment}. The linear theory then identifies the noise-adaptation mechanism in \Cref{app:mask-noise-adaptation}. These predictions complement the learning curves and simulation comparisons in \Cref{fig:validation,fig:mask-sample-efficiency,fig:K-interpolation}. We use the replica-symmetric equations of \Cref{app:replica-analysis} at zero ridge; the onset calculation follows the growth of an infinitesimal feature overlap from the uninformative solution.

\subsection{Recovery criterion and fixed-signal factorization}
\label{app:fixed-beta-comparison}
% =============================================================================

This subsection derives the fixed-$\beta$ recovery criterion for the finite-$K$ masked autoencoder. We set $\lambda_R=0$, keep $0<\beta<\infty$ fixed, and work on a regular replica-symmetric uninformative branch with $M=0$, $q_{\Gamma,0}>0$, $\mathcal{X}_{K,0}>0$, and finite output moments. The only small parameter in the instability calculation is the teacher overlap $M$.

\subsubsection{Finite-$K$ masked uninformative saddle}
\label{subsec:fixed-beta-masked-saddle}
Whiten the finite-$K$ field using
\begin{equation}
    \bm g := \Omega_K^{-1/2}\bfield_K = (h,x_1,\ldots,x_K), \qquad \delta_r=\sqrt{\rho(1-\rho)}\,x_r.
    \label{eq:fixed-beta-whitened-field}
\end{equation}
The common field entering the output proximal is
\begin{equation}
    \bm g_0 = \sqrt\beta\,\lambda M\bm e_0 + \sqrt{q_\Gamma}\,\bm z, \qquad \bm z\sim\mathcal N(\bm0,\id_{K+1}).
    \label{eq:fixed-beta-common-field}
\end{equation}
In whitened coordinates the proximal and force are
\begin{gather}
    \bm g_\star = \argmin_{\bm g\in\mathbb R^{K+1}} \left\{ \ell_K(\bm g;q_I) + \frac1{2v_\Gamma}\|\bm g-\bm g_0\|_2^2 \right\},
    \label{eq:fixed-beta-prox}
    \\
    \widetilde{\bm F}_K := \Omega_K^{1/2}\bm F_K = \frac{\bm g_\star-\bm g_0}{v_\Gamma}.
    \label{eq:fixed-beta-force}
\end{gather}
At $M=0$ write
\begin{equation}
    \bm g_0^{(0)} = \sqrt{q_{\Gamma,0}}\,\bm z.
    \label{eq:fixed-beta-uninformative-field}
\end{equation}
All quantities below are evaluated at the corresponding uninformative proximal. Define
\begin{gather}
    \mathcal E_{K,0} := \mathbb E\|\widetilde{\bm F}_{K,0}\|_2^2,
    \label{eq:fixed-beta-E}
    \\
    \mathcal{X}_{K,0} :=\left.\frac{d}{d\varepsilon} \mathbb E_{\bm z}\!\left[ \widetilde F_{K,h}\!\left( \sqrt{q_{\Gamma,0}}\,\bm z+\varepsilon\bm e_0 \right) \right]\right|_{\varepsilon=0},
    \label{eq:fixed-beta-chi}
    \\
    \mathcal A_{K,0} := \mathbb E\!\left[ \partial_{q_I}\ell_K(\bm g_\star;q_{I,0}) \right] = \rho\, \mathbb E\!\left[ \frac1K\sum_{r=1}^K\sigma(y_{r,\star})^2 \right] \ge0,
    \label{eq:fixed-beta-A}
    \\
    \mathcal B_{K,0} := \mathbb E[D_{\Gamma,K}].
    \label{eq:fixed-beta-B-definition}
\end{gather}
In \eqref{eq:fixed-beta-chi}, all other order parameters are held fixed. Differentiating the shifted Gaussian density gives the score expression \eqref{eq:fixed-beta-Stein}; $q_{\Gamma,0}>0$ and the finite force second moment suffice for this step. This averaged response includes changes in the selected proximal minimizer. It agrees with the expectation of the ordinary force derivative when differentiation can be interchanged with Gaussian averaging. At finite $K$, the loss has no explicit dependence on $q_\Gamma$. Using \eqref{eq:finite-K-DGamma} in whitened coordinates gives
\begin{equation}
    \mathcal B_{K,0} = -\frac1{2\sqrt{q_{\Gamma,0}}} \mathbb E\!\left[ \widetilde{\bm F}_{K,0}^{\mathsf T}\bm z \right].
    \label{eq:fixed-beta-B}
\end{equation}
The finite-$K$ conjugate equations \eqref{eq:finite-K-hatq}--\eqref{eq:finite-K-hatv} reduce at $M=0$ to
\begin{equation}
    \widehat q_0(\gamma) = \alpha\gamma\mathcal E_{K,0}, \qquad \widehat v_0(\gamma) = 2\alpha \left( \mathcal A_{K,0} + \gamma\mathcal B_{K,0} \right).
    \label{eq:fixed-beta-hats}
\end{equation}
The ridgeless input equations therefore give
\begin{gather}
    v_0(\gamma) = \frac1{ 2\alpha( \mathcal A_{K,0} + \gamma\mathcal B_{K,0} ) },
    \label{eq:fixed-beta-vgamma}
    \\
    q_0(\gamma) = \alpha\gamma\mathcal E_{K,0} v_0(\gamma)^2.
    \label{eq:fixed-beta-qgamma}
\end{gather}
Consequently,
\begin{gather}
    q_{I,0} = \alpha\mathcal E_{K,0} \mathbb E_\gamma[\gamma v_0(\gamma)^2],
    \label{eq:fixed-beta-qI-closure}
    \\
    q_{\Gamma,0} = \alpha\mathcal E_{K,0} \mathbb E_\gamma[\gamma^2v_0(\gamma)^2],
    \label{eq:fixed-beta-qGamma-closure}
    \\
    v_{\Gamma,0} = \mathbb E_\gamma[\gamma v_0(\gamma)].
    \label{eq:fixed-beta-vGamma-closure}
\end{gather}
A regular ridgeless saddle requires
\begin{equation}
    \mathcal A_{K,0}+\gamma\mathcal B_{K,0}>0 \qquad P_\gamma\text{-a.s.}
    \label{eq:fixed-beta-regularity-AB}
\end{equation}

\subsubsection{Linear instability in the teacher overlap}
\label{subsec:fixed-beta-masked-instability}
Recovery begins when the response to a small feature alignment amplifies that alignment. Perturb the uninformative saddle by an infinitesimal $M$:
\begin{equation}
    \bm g_0 = \bm g_0^{(0)} + \sqrt\beta\,\lambda M\bm e_0.
    \label{eq:fixed-beta-M-perturbation}
\end{equation}
The remaining uninformative order parameters change only at order $M^2$. Expanding the Gaussian-averaged informative force at fixed remaining order parameters gives
\begin{equation}
    \mathbb E_{\bm z}\!\left[ \widetilde F_{K,h}(\bm g_0^{(0)}+\varepsilon\bm e_0) \right] =\mathbb E_{\bm z}\!\left[ \widetilde F_{K,h}(\bm g_0^{(0)}) \right] +\varepsilon\mathcal{X}_{K,0}+o(\varepsilon).
    \label{eq:fixed-beta-force-expansion}
\end{equation}
Substituting $\varepsilon=\sqrt\beta\,\lambda M$ and averaging over $\lambda$, using its independence from the $M=0$ channel, $\mathbb E[\lambda]=0$, and $\mathbb E[\lambda^2]=1$, gives
\begin{equation}
    \mathbb E[\lambda\widetilde F_{K,h}] = \sqrt\beta\,M\mathcal{X}_{K,0} + o(M).
    \label{eq:fixed-beta-force-response}
\end{equation}
Equation~\eqref{eq:finite-K-hatm} then yields the teacher conjugate
\begin{equation}
    \widehat m = \alpha\beta\mathcal{X}_{K,0}M + o(M).
    \label{eq:fixed-beta-hatm-linearized}
\end{equation}
The input relation $m(\gamma)=v_0(\gamma)\widehat m+o(M)$ gives
\begin{equation}
    M_{\rm new} = \alpha\beta\mathcal{X}_{K,0} \mathbb E_\gamma[v_0(\gamma)]M + o(M).
    \label{eq:fixed-beta-M-map}
\end{equation}
Thus the exact continuous-instability condition is
\begin{equation}
    1 = \alpha_c^{{\rm mask},K}\, \beta\,\mathcal{X}_{K,c} \mathbb E_\gamma[v_{K,c}(\gamma)].
    \label{eq:fixed-beta-instability-condition}
\end{equation}
Equivalently, along the regular $M=0$ branch define
\begin{equation}
    \Lambda_K(\alpha;\beta) := \alpha\beta\mathcal{X}_{K,0}(\alpha) \mathbb E_\gamma[v_{K,0}(\gamma;\alpha)].
    \label{eq:fixed-beta-gain}
\end{equation}
The crossing of $\Lambda_K$ through one predicts the continuous onset of feature recovery on this branch.

\subsubsection{Exact fixed-$\beta$ factorization}
\label{subsec:fixed-beta-factorization}

The threshold separates into an output-efficiency factor and a factor describing the response to heterogeneous coordinate noise. The latter is
\begin{equation}
    \mathcal R_\gamma(\tau) := \frac{\displaystyle\E_\gamma\!\left[ \frac{\gamma^2}{(1+\tau\gamma)^2}\right]} {\displaystyle\left(\E_\gamma\!\left[ \frac1{1+\tau\gamma}\right]\right)^2}, \qquad 1+\tau\gamma>0\quad P_\gamma\text{-a.s.}
    \label{eq:main-Rgamma}
\end{equation}
It collects the effect of the variance distribution on the recovery threshold. Define
\begin{equation}
    H_{0,K} := \mathbb E_\gamma[v_{K,c}(\gamma)], \qquad H_{2,K} := \mathbb E_\gamma[\gamma^2v_{K,c}(\gamma)^2].
    \label{eq:fixed-beta-H0-H2}
\end{equation}
At criticality, \eqref{eq:fixed-beta-qGamma-closure} reads
\begin{equation}
    q_{\Gamma,c} = \alpha_c^{{\rm mask},K}\mathcal E_{K,c}H_{2,K}.
    \label{eq:fixed-beta-qGamma-critical}
\end{equation}
Introduce the output-efficiency factor
\begin{equation}
    \eta_{K,c}^2 := \frac{\mathcal E_{K,c}} {q_{\Gamma,c}\mathcal{X}_{K,c}^2}.
    \label{eq:fixed-beta-eta}
\end{equation}
Substituting $\mathcal E_{K,c}=\eta_{K,c}^2q_{\Gamma,c}\mathcal{X}_{K,c}^2$ into \eqref{eq:fixed-beta-qGamma-critical} and cancelling $q_{\Gamma,c}>0$ gives
\begin{equation}
    1 = \alpha_c^{{\rm mask},K} \eta_{K,c}^2\mathcal{X}_{K,c}^2H_{2,K}.
    \label{eq:fixed-beta-norm-condition}
\end{equation}
On the other hand, \eqref{eq:fixed-beta-instability-condition} gives
\begin{equation}
    \mathcal{X}_{K,c} = \frac1{ \alpha_c^{{\rm mask},K}\beta H_{0,K} }.
    \label{eq:fixed-beta-chi-elimination}
\end{equation}
Eliminating $\mathcal{X}_{K,c}$ yields
\begin{equation}
    \alpha_c^{{\rm mask},K} = \frac{\eta_{K,c}^2}{\beta^2} \frac{H_{2,K}}{H_{0,K}^2}.
    \label{eq:fixed-beta-factorization-v}
\end{equation}
Assume $\mathcal A_{K,c}>0$ and define
\begin{equation}
    \tau_{K,c} := \frac{\mathcal B_{K,c}}{\mathcal A_{K,c}}.
    \label{eq:fixed-beta-tau}
\end{equation}
Equation~\eqref{eq:fixed-beta-vgamma} becomes
\begin{equation}
    v_{K,c}(\gamma) = \frac1{2\alpha_c^{{\rm mask},K}\mathcal A_{K,c}} \frac1{1+\tau_{K,c}\gamma}.
    \label{eq:fixed-beta-v-tau}
\end{equation}
The overall scale cancels from \eqref{eq:fixed-beta-factorization-v}, giving
\begin{equation}
    \alpha_c^{{\rm mask},K} = \frac{\eta_{K,c}^2}{\beta^2} \mathcal R_\gamma(\tau_{K,c}),
    \label{eq:fixed-beta-factorization-R}
\end{equation}
with $\mathcal R_\gamma$ defined in \eqref{eq:main-Rgamma}. No $\beta\downarrow0$ limit has been used.

\paragraph{Dynamic masking.} The same derivation applies after the mask average. One keeps $\mathcal E_{\infty,0}=\mathbb E[F_{{\rm dyn},0}^2]$ and $\mathcal{X}_{\infty,0}=q_{\Gamma,0}^{-1/2}\mathbb E[zF_{{\rm dyn},0}]$ with the same averaged-response convention, while $\mathcal B_{\infty,0}$ includes the explicit derivative of the dynamic loss,
\begin{equation}
    \mathcal B_{\infty,0} = \mathbb E\!\left[ \partial_{q_\Gamma}\ell_{\rm dyn}(h_\star;q_{I,0},q_{\Gamma,0}) - \frac{F_{\rm dyn}z}{2\sqrt{q_{\Gamma,0}}} \right].
    \label{eq:fixed-beta-dynamic-B}
\end{equation}
Since $\widehat v_0(\gamma)=2\alpha(\mathcal A_{\infty,0}+ \gamma\mathcal B_{\infty,0})$, the factorization \eqref{eq:fixed-beta-factorization-R} carries over with $K=\infty$ whenever the regular dynamic saddle exists.

\subsubsection{Output-efficiency bound}
\label{subsec:fixed-beta-output-bound}
Differentiating the Gaussian density in the averaged response \eqref{eq:fixed-beta-chi}, evaluated at criticality, gives
\begin{equation}
    \mathcal{X}_{K,c} = \frac1{\sqrt{q_{\Gamma,c}}} \mathbb E[z_h\widetilde F_{K,h}].
    \label{eq:fixed-beta-Stein}
\end{equation}
This expression also evaluates the response for a discontinuous proximal force, without differentiating within a selected minimizer branch. Therefore
\begin{equation}
    q_{\Gamma,c}\mathcal{X}_{K,c}^2 = \left(\mathbb E[z_h\widetilde F_{K,h}]\right)^2 \le \mathbb E[\widetilde F_{K,h}^2] \le \mathbb E\|\widetilde{\bm F}_{K,c}\|_2^2 = \mathcal E_{K,c}.
    \label{eq:fixed-beta-eta-bound-proof}
\end{equation}
Hence
\begin{equation}
    \eta_{K,c}^2\ge1.
    \label{eq:fixed-beta-eta-bound}
\end{equation}
Equality requires that the output force have no component in the mask-specific directions and that its informative component be proportional to $z_h$ in $L^2$.

\subsubsection{Monotonicity and bounds for the heteroskedastic factor}
\label{subsec:fixed-beta-R-bounds}
Let
\begin{equation}
    a_\tau(\gamma) := \frac1{1+\tau\gamma}, \qquad X_\tau(\gamma) := \frac{\gamma}{1+\tau\gamma}.
    \label{eq:fixed-beta-a-X}
\end{equation}
Then $a_\tau=1-\tau X_\tau$ and
\begin{equation}
    \mathcal R_\gamma(\tau) = \frac{\mathbb E_\gamma[X_\tau^2]} {\mathbb E_\gamma[a_\tau]^2}.
    \label{eq:fixed-beta-R-X}
\end{equation}
The derivatives are
\begin{equation}
    \frac{\mathrm d X_\tau}{\mathrm d\tau}=-X_\tau^2, \qquad \frac{\mathrm d a_\tau}{\mathrm d\tau}=-X_\tau a_\tau,
    \label{eq:fixed-beta-a-X-derivatives}
\end{equation}
so
\begin{equation}
    \frac12\frac{\mathrm d}{\mathrm d\tau} \log\mathcal R_\gamma(\tau) = -\frac{\mathbb E_\gamma[X_\tau^3]}{\mathbb E_\gamma[X_\tau^2]} + \frac{\mathbb E_\gamma[X_\tau a_\tau]}{\mathbb E_\gamma[a_\tau]}.
    \label{eq:fixed-beta-R-derivative-intermediate}
\end{equation}
Let $\gamma'\sim P_\gamma$ be independent of $\gamma$, and set $X_\tau'=X_\tau(\gamma')$. Symmetrizing the last expression gives
\begin{equation}
    \frac12\frac{\mathrm d}{\mathrm d\tau} \log\mathcal R_\gamma(\tau) = - \frac{ \mathbb E_{\gamma,\gamma'}\!\left[ (X_\tau-X_\tau')^2 (X_\tau+X_\tau'-\tau X_\tau X_\tau') \right] }{ 2\,\mathbb E_\gamma[X_\tau^2]\,\mathbb E_\gamma[a_\tau] } \le0.
    \label{eq:fixed-beta-R-derivative}
\end{equation}
The sign follows because
\begin{equation}
    X_\tau+X_\tau'-\tau X_\tau X_\tau' = \frac{ \gamma+\gamma'+\tau\gamma\gamma' }{ (1+\tau\gamma)(1+\tau\gamma') } >0
    \label{eq:fixed-beta-positive-factor}
\end{equation}
throughout the regularity domain. The inequality is strict for a nondegenerate variance law. At $\tau=0$,
\begin{equation}
    \mathcal R_\gamma(0) = \mathbb E_\gamma[\gamma^2] = \kappa_2.
    \label{eq:fixed-beta-R-zero}
\end{equation}
If $\mathbb E_\gamma[\gamma^{-1}]<\infty$, dominated convergence gives
\begin{equation}
    \lim_{\tau\to+\infty}\mathcal R_\gamma(\tau) = \frac1{\mathbb E_\gamma[\gamma^{-1}]^2}.
    \label{eq:fixed-beta-R-infinity}
\end{equation}
Thus every regular critical saddle obeys
\begin{equation}
    \mathcal R_\gamma(\tau_{K,c}) \ge \frac1{\mathbb E_\gamma[\gamma^{-1}]^2}.
    \label{eq:fixed-beta-R-lower-bound}
\end{equation}
Combining this with \eqref{eq:fixed-beta-eta-bound} gives
\begin{equation}
    \alpha_c^{{\rm mask},K} \ge \frac1{\beta^2\mathbb E_\gamma[\gamma^{-1}]^2}.
    \label{eq:fixed-beta-mask-lower-bound}
\end{equation}
Finally,
\begin{equation}
    \mathbb E_\gamma[\gamma^{-2}] \ge \mathbb E_\gamma[\gamma^{-1}]^2
    \label{eq:fixed-beta-negative-moment-CS}
\end{equation}
by Cauchy--Schwarz, and hence
\begin{equation}
    \alpha_c^{{\rm mask},K} \ge \frac1{\beta^2\mathbb E_\gamma[\gamma^{-1}]^2} \ge \alpha_c^{\rm BO}.
    \label{eq:fixed-beta-mask-BO-hierarchy}
\end{equation}

For unbounded variance support, the positivity condition in \eqref{eq:main-Rgamma} requires $\tau_{K,c}\geq0$. Hence the factorization gives
\begin{equation}
    \alpha_c^{{\rm mask},K} \leq\frac{\eta_{K,c}^2\kappa_2}{\beta^2},
    \label{eq:fixed-beta-unbounded-upper-bound}
\end{equation}
whenever the critical solution has finite output efficiency and $\kappa_2<\infty$. The linear dynamic case is developed in \Cref{app:dynamic-linear-bound}.

\subsection{Mask diversity in the weak-signal regime}
\label{subsec:ridgeless-weak-signal-thresholds}
% =============================================================================
\subsubsection{Recovery predictions and weak-signal regime}
\label{subsec:recovery-threshold-scope}
The weak-signal limit makes the sample cost of mask reuse explicit. Starting from the high-dimensional equations, we take
\begin{equation}
    \laR=0, \qquad \beta\downarrow0, \qquad \alpha=O(\beta^{-2}), \qquad K=O(1),
    \label{eq:weak-signal-scaling-alpha}
\end{equation}
with $\alpha\beta^2=O(1)$ and $K$ held fixed while the weak-signal limit is taken. The dynamic-mask coefficient is obtained afterward through the sequential limit $K\to\infty$. We denote
\begin{equation}
    \kappa_2:=\E_\gamma[\gamma^2].
    \label{eq:weak-signal-kappa2}
\end{equation}
The analysis concerns the regular continuous branch on which $q_I,q_\Gamma,v_\Gamma\to0$ as $\beta\downarrow0$.

% -----------------------------------------------------------------------------
% -----------------------------------------------------------------------------

\paragraph{Weak-signal recovery predictions.} Fix $0<\rho<1$ and finite $K$.

For the smooth class, assume $\sig\in C^1(\R)$, with both $\sig$ and $\sig'$ globally Lipschitz, and
\begin{equation}
    \sig(0)=0, \qquad \sig(x)=s_\sig x+O(x^2), \qquad 0<s_\sig<\infty,
    \label{eq:weak-signal-smooth-activation-assumption}
\end{equation}
the continuous weak-recovery threshold is
\begin{equation}
    \alpha_{{\rm c},{\rm sm}}^{(K)} = \frac{\kappa_2}{\beta^2} \left[ 1+ \frac{ 1-2\rho(1-\rho) }{ 2K\rho(1-\rho) } \right] \left[1+o_\beta(1)\right].
    \label{eq:weak-signal-smooth-threshold-K}
\end{equation}
These sufficient regularity conditions justify the Gaussian moment and proximal expansions below; they include linear, tanh, and the standard ELU $\sig(x)=x$ for $x\geq0$, $\sig(x)=e^x-1$ for $x<0$. The second derivative need only exist almost everywhere and be bounded, so the jump in the ELU second derivative at zero is allowed. Activations satisfying these conditions belong to the same weak-signal universality class. Activations with $\sig(0)\ne0$, vanishing linear response, or discontinuities require a different expansion.

For ReLU,
\begin{equation}
    \sig(x)=x_+:=\max\{x,0\},
    \label{eq:weak-signal-relu-definition}
\end{equation}
define the off-diagonal force covariance
\begin{equation}
    \mathcal D_{\rm R}(\rho) = \rho^2(1-\rho)^2 \left[ 1+\frac2\pi\arcsin(1-\rho) \right] + \frac{\rho(1-\rho)}{2\pi} \sqrt{\frac{\rho}{2-\rho}}.
    \label{eq:weak-signal-relu-D-rho}
\end{equation}
The ReLU threshold is
\begin{equation}
    \alpha_{{\rm c},{\rm R}}^{(K)} = \frac{\kappa_2}{\beta^2} \Bigg[ \frac{ \mathcal D_{\rm R}(\rho) }{ \rho^2(1-\rho)^2 } + \frac1K \left( \frac1{\rho(1-\rho)} - \frac{ \mathcal D_{\rm R}(\rho) }{ \rho^2(1-\rho)^2 } \right) \Bigg] \left[1+o_\beta(1)\right].
    \label{eq:weak-signal-relu-threshold-K}
\end{equation}

For both activation classes, the leading weak-signal threshold is strictly decreasing with the number of independent masks in the fixed collection. The following ordering refers to these leading coefficients, with the dynamic coefficient obtained in the sequential limit specified above:
\begin{equation}
    \alpha_{\rm c}^{(1)} > \alpha_{\rm c}^{(2)} > \cdots > \alpha_{\rm c}^{(\infty)}.
    \label{eq:weak-signal-main-monotonicity}
\end{equation}
Since $K=1$ corresponds to static masking and $K\to\infty$ to dynamic masking, the leading dynamic coefficient is strictly smaller than the static one. The fixed-$K$ expansion does not establish uniform finite-signal monotonicity for $K$ increasing jointly with $\beta^{-1}$.

At half masking, $\rho=1/2$, the leading coefficients multiplying $\kappa_2/\beta^2$ are
\begin{equation}
\begin{array}{c|ccc}
 &K=1&K=2&K\to\infty\\ \hline
\text{Smooth}&2&3/2&1\\
\text{ReLU}&4&\dfrac83+\dfrac1{\pi\sqrt3} &\dfrac43+\dfrac2{\pi\sqrt3}
\end{array}
\label{eq:weak-signal-half-coefficients}
\end{equation}
ReLU has a distinct coefficient because its two one-sided slopes differ.

\paragraph{Excess over the dynamic weak-signal reference.} For the centered smooth class, let $\alpha_c^{(\infty)}=\kappa_2\beta^{-2}[1+o_\beta(1)]$ denote the sequential dynamic reference. The finite-$K$ expansion can equivalently be written as
\begin{equation}
    \alpha_c^{(K)}-\alpha_c^{(\infty)} = \frac{\alpha_c^{(1)}-\alpha_c^{(\infty)}}{K} + o_\beta(\beta^{-2}).
    \label{eq:main-one-over-K}
\end{equation}
Thus each additional independent mask reduces the excess sample requirement, with diminishing returns. The full fixed-signal criterion of \Cref{app:fixed-beta-comparison} is evaluated numerically in \Cref{app:threshold-ratio-assessment}.

\subsubsection{Common weak-signal reduction}
\paragraph{Finite-$K$ output channel.} Use the whitened field, common Gaussian center, proximal, and force from \eqref{eq:fixed-beta-whitened-field}--\eqref{eq:fixed-beta-uninformative-field}. At $M=0$, $\bm g_0\sim\N(\bm0,q_\Gamma\id_{K+1})$. The visible and masked fields of view $r$ are therefore
\begin{equation}
    y_r := h_{r,\vis} = (1-\rho)h + \sqrt{\rho(1-\rho)}\,x_r, \qquad t_r := h_{r,\mask} = \rho h - \sqrt{\rho(1-\rho)}\,x_r.
    \label{eq:weak-signal-y-t-fields}
\end{equation}
The finite-$K$ output loss becomes
\begin{equation}
    \ell_K(\bm g;q_I) = \frac1K\sum_{r=1}^K \left[ \rho q_I\sig(y_r)^2 - 2\sig(y_r)t_r \right].
    \label{eq:weak-signal-loss-whitened}
\end{equation}

At differentiable proximal minimizers, $\widetilde{\bm F}_K=-\nabla_{\bm g}\ell_K(\bm g_\star;q_I)$; at ReLU kinks the displacement in \eqref{eq:fixed-beta-force} defines the force. Since the signal loads on the common component,
\begin{equation}
    A_\star=\widetilde F_{K,h},\qquad \bm F_K^\top\Omega_K\bm F_K=\|\widetilde{\bm F}_K\|_2^2.
    \label{eq:weak-signal-force-identities}
\end{equation}
The different views share the field $h$; their force covariances must be retained when computing the effect of $K$.

For the two activation classes considered below, take the weak-signal limit of the averaged response \eqref{eq:fixed-beta-chi}, using its Gaussian-score form \eqref{eq:fixed-beta-Stein}, and define
\begin{gather}
    \mathcal{X}_K := \lim_{\substack{ q_I,q_\Gamma,v_\Gamma\to0\\
        M=0 }} \frac1{\sqrt{q_\Gamma}} \E_{\bm z}\!\left[ z_h\widetilde F_{K,h}(\sqrt{q_\Gamma}\,\bm z) \right],
    \label{eq:weak-signal-chi-definition}
    \\
    T_K := \lim_{\substack{ q_I,q_\Gamma,v_\Gamma\to0\\
        M=0 }} \frac1{q_\Gamma} \E \left[ \|\widetilde{\bm F}_K\|_2^2 \right].
    \label{eq:weak-signal-T-definition}
\end{gather}
We also require
\begin{equation}
    U_K := \lim_{q_\Gamma\downarrow0} \frac1{q_\Gamma} \E \left[ \left. \partial_{q_I} \ell_K(\bm g_0;q_I) \right|_{q_I=0,M=0} \right].
    \label{eq:weak-signal-U-definition}
\end{equation}
The coefficient $U_K$ controls the norm of the uninformative solution but will cancel from the recovery threshold.

% -----------------------------------------------------------------------------
\paragraph{Generic reduction of the weak-recovery condition.}

Taking the weak-signal limit of the response in \eqref{eq:fixed-beta-force-expansion} gives, at leading order,
\begin{equation}
    \widehat m=\alpha\beta\mathcal{X}_K M+o(M).
    \label{eq:weak-signal-hatm-linearized}
\end{equation}
At $M=0$,
\begin{equation}
    \widehat q(\gamma) = \alpha\gamma q_\Gamma T_K + o(\alpha q_\Gamma).
    \label{eq:weak-signal-hatq-linearized}
\end{equation}

It remains to expand $\widehat v(\gamma)$. At fixed finite $K$, the loss \eqref{eq:weak-signal-loss-whitened} has no explicit dependence on $q_\Gamma$. Thus, at $M=0$,
\begin{equation}
    D_\Gamma = - \widetilde{\bm F}_K^\top \partial_{q_\Gamma}\bm g_0 = - \frac1{2q_\Gamma} \widetilde{\bm F}_K^\top\bm g_0.
    \label{eq:weak-signal-DGamma-whitened}
\end{equation}

\paragraph{Euler identity and the nonlinear remainder.} In this paragraph write $q=q_\Gamma$, $Q=q_I$, and $v=v_\Gamma$. All limits have $q,Q,v\downarrow0$ at fixed $K$ and $\rho$; constants may depend on $K$, $\rho$, and the activation, but not on $q,Q,v$. For the smooth class the loss decomposes as
\begin{equation}
 \ell_K(\bm g;Q)=\ell_K^{(2)}(\bm g)+QG_K^{(2)}(\bm g) +\mathcal E_K(\bm g;Q),\qquad |\mathcal E_K(\bm g;Q)|\leq C(1+Q)\|\bm g\|^3,
 \label{eq:weak-signal-leading-loss-decomposition}
\end{equation}
where $\ell_K^{(2)}=-2s_\sig K^{-1}\sum_j y_jt_j$ and $G_K^{(2)}=\rho s_\sig^2K^{-1}\sum_j y_j^2$ are homogeneous of degree two. At the uninformative Gaussian field, $y_j$ and $t_j$ are independent and centered for each view. Thus
\begin{equation}
 \E[\ell_K^{(2)}(\bm g_0)]=0,\qquad \E[G_K^{(2)}(\bm g_0)]=U_Kq,\qquad \E[\ell_K(\bm g_0;Q)]=QU_Kq+O(Qq^{3/2}).
 \label{eq:weak-signal-leading-loss-means}
\end{equation}
The last estimate uses the exact cancellation $\E[\sig(y_j)t_j]=0$, rather than just the pointwise cubic bound.

Define the Euler defect of the full loss by
\begin{equation}
 \mathcal R_K(\bm g;Q) :=\bm g^\top\nabla\ell_K(\bm g;Q)-2\ell_K(\bm g;Q).
 \label{eq:weak-signal-Euler-defect}
\end{equation}
Proximal stationarity gives the exact identity
\begin{equation}
 \widetilde{\bm F}_K^\top\bm g_0 =-2\ell_K(\bm g_\star;Q)-v\|\widetilde{\bm F}_K\|^2 -\mathcal R_K(\bm g_\star;Q).
 \label{eq:weak-signal-Euler-force-identity}
\end{equation}
The defect vanishes for the linear loss. For a general activation in the smooth class, put $d_\sig(y)=y\sig'(y)-\sig(y)$. Direct differentiation gives
\begin{equation}
 \mathcal R_K(\bm g;Q) =\frac2K\sum_{j=1}^K [\rho Q\sig(y_j)-t_j]d_\sig(y_j).
 \label{eq:weak-signal-Euler-defect-explicit}
\end{equation}
The assumptions on $\sig$ imply $|\sig(y)|\leq C|y|$, $|d_\sig(y)|\leq Cy^2$, and $|d_\sig'(y)|\leq C|y|$ almost everywhere. Independence within each pair $(y_j,t_j)$ consequently yields
\begin{equation}
 \E[\mathcal R_K(\bm g_0;Q)]=O(Qq^{3/2}),\qquad \|\nabla\mathcal R_K(\bm g;Q)\| \leq C(1+Q)\|\bm g\|^2\quad\text{a.e.}
 \label{eq:weak-signal-Euler-defect-bounds}
\end{equation}

To transfer this estimate to the proximal, note that $\ell_K(\bm g;Q)\geq-C\|\bm g\|^2$ and $\|\nabla\ell_K(\bm g;Q)\|\leq C(1+Q)\|\bm g\|$. For sufficiently small $v$ the proximal objective is coercive, and stationarity of any minimizer gives, for small $Q,v$,
\begin{equation}
 \|\bm g_\star\|\leq C\|\bm g_0\|,\qquad \|\bm g_\star-\bm g_0\|\leq Cv\|\bm g_0\|.
 \label{eq:weak-signal-prox-displacement}
\end{equation}
Integrating the defect gradient along this displacement and using $\E\|\bm g_0\|^3=O(q^{3/2})$ proves
\begin{equation}
 |\E[\mathcal R_K(\bm g_\star;Q)]| \leq C(Q+v)q^{3/2}=o\bigl(q(Q+v)\bigr).
 \label{eq:weak-signal-Euler-defect-negligible}
\end{equation}
This is an expectation bound: cancellation of the hidden-field term is essential, and a pointwise cubic remainder alone would not suffice.

The same regularity assumptions give $\|\nabla^2\ell_K(\bm g;Q)\|\leq C(1+Q)(1+\|\bm g\|)$ almost everywhere. Taylor expansion between $\bm g_\star$ and $\bm g_0$, together with \eqref{eq:weak-signal-prox-displacement}, therefore gives
\begin{equation}
 \E[\ell_K(\bm g_\star;Q)] =\E[\ell_K(\bm g_0;Q)]-v\E\|\widetilde{\bm F}_K\|^2 +o(vq).
 \label{eq:weak-signal-loss-at-prox-expansion}
\end{equation}
Indeed the Taylor remainder is $O(v^2(q+q^{3/2}))=o(vq)$. Moreover, $\nabla\ell_K(\bm g;Q)=\nabla\ell_K^{(2)}(\bm g) +O(Q\|\bm g\|+\|\bm g\|^2)$ as $Q\downarrow0$. The displacement estimate and Gaussian moments then show that $q^{-1}\E\|\widetilde{\bm F}_K\|^2\to T_K$, as evaluated below.

\paragraph{Positive homogeneity of ReLU.} For ReLU, \eqref{eq:weak-signal-leading-loss-decomposition} holds with $\mathcal E_K=0$, using the ReLU homogeneous terms. The full loss is positively homogeneous of degree two, so the Euler defect is zero. At a kink the force is defined by the proximal displacement; varying $\bm g_\star$ along its radial direction gives \eqref{eq:weak-signal-Euler-force-identity} with zero defect there as well. The bounds \eqref{eq:weak-signal-prox-displacement} still hold, since the loss is locally Lipschitz and every generalized gradient has norm at most $C(1+Q)\|\bm g\|$.

For completeness, rescale $\bm g_0=\sqrt q\,\bm z$ and $\bm g_\star=\sqrt q\,\bm z_\star$. At almost every Gaussian $\bm z$, all visible fields are nonzero. For small enough $v$ the proximal remains in the same activation region, where the loss is quadratic, and the first-order energy change divided by $vq$ tends to $-\|\nabla\ell_K^{(2)}(\bm z)\|^2$ as $Q,v\downarrow0$. The displacement and gradient bounds dominate this quotient, and the squared force divided by $q$, by $C\|\bm z\|^2$. Dominated convergence proves \eqref{eq:weak-signal-loss-at-prox-expansion} and, more precisely,
\[
    q^{-1/2}\widetilde{\bm F}_K(\sqrt q\,\bm z) \longrightarrow-\nabla\ell_K^{(2)}(\bm z) \qquad\text{in }L^2(\bm z).
\]
This convergence yields both $q^{-1}\E\|\widetilde{\bm F}_K\|^2\to T_K$ and the Gaussian-score response in \eqref{eq:weak-signal-chi-definition}, including the effect of activation-boundary crossings. Also $\E[\ell_K(\bm g_0;Q)]=QU_Kq$ exactly.

\paragraph{Susceptibility expansion.} For either activation class, substitute \eqref{eq:weak-signal-leading-loss-means}, \eqref{eq:weak-signal-Euler-force-identity}, and \eqref{eq:weak-signal-loss-at-prox-expansion} into \eqref{eq:weak-signal-DGamma-whitened}. The smooth Euler defect contributes $o(Q+v)$ after division by $2q$; it is identically zero in the ReLU case. Restoring the original notation yields
\begin{equation}
    \E[D_\Gamma] = U_Kq_I - \frac12v_\Gamma T_K + o(q_I+v_\Gamma).
    \label{eq:weak-signal-DGamma-expansion}
\end{equation}
Likewise, the displacement bound and the Gaussian moment estimate give $\E[\partial_{q_I}\ell_K(\bm g_\star;q_I)]=U_Kq_\Gamma+o(q_\Gamma)$ for both activation classes. The conjugate susceptibility therefore becomes
\begin{equation}
    \widehat v(\gamma) = 2\alpha \left[ U_Kq_\Gamma + \gamma \left( U_Kq_I - \frac12v_\Gamma T_K \right) \right] + o\!\left( \alpha(q_I+q_\Gamma+v_\Gamma) \right).
    \label{eq:weak-signal-hatv-expansion}
\end{equation}

The regular uninformative branch has the scaling
\begin{equation}
    q_I = \beta\bar q_I+o(\beta), \qquad q_\Gamma = \beta\bar q_\Gamma+o(\beta), \qquad v(\gamma) = \beta\bar v(\gamma)+o(\beta), \qquad v_\Gamma = \beta\bar v_\Gamma+o(\beta),
    \label{eq:weak-signal-order-parameter-scaling}
\end{equation}
with $\bar q_\Gamma>0$.

At $M=0$, the ridgeless input equations together with \eqref{eq:weak-signal-hatq-linearized} give
\begin{equation}
    q(\gamma) = \alpha\beta^3 \gamma\bar q_\Gamma T_K \bar v(\gamma)^2 + o(\beta).
    \label{eq:weak-signal-qgamma-leading}
\end{equation}
Consequently,
\begin{gather}
    1 = \alpha\beta^2T_K \E_\gamma \left[ \gamma^2\bar v(\gamma)^2 \right],
    \label{eq:weak-signal-uninformative-norm-condition-general}
    \\
    \bar q_I = \alpha\beta^2 \bar q_\Gamma T_K \E_\gamma \left[ \gamma\bar v(\gamma)^2 \right].
    \label{eq:weak-signal-qI-general}
\end{gather}

On the other hand, $v(\gamma)=1/\widehat v(\gamma)$ in the ridgeless problem. Using \eqref{eq:weak-signal-hatv-expansion},
\begin{equation}
    \frac1{\bar v(\gamma)} = 2\alpha\beta^2 \left[ U_K\bar q_\Gamma + \gamma \left( U_K\bar q_I - \frac12\bar v_\Gamma T_K \right) \right].
    \label{eq:weak-signal-vgamma-general}
\end{equation}
Multiplying by $\gamma\bar v(\gamma)^2$, averaging over $\gamma$, and using \eqref{eq:weak-signal-uninformative-norm-condition-general} and \eqref{eq:weak-signal-qI-general} gives
\begin{equation}
    \bar v_\Gamma = \frac{4U_K}{T_K}\bar q_I - \bar v_\Gamma.
\end{equation}
Hence
\begin{equation}
    2U_K\bar q_I = T_K\bar v_\Gamma.
    \label{eq:weak-signal-gamma-cancellation}
\end{equation}
The coefficient multiplying $\gamma$ in \eqref{eq:weak-signal-vgamma-general} therefore vanishes. Thus the leading input susceptibility is forced to be independent of the variance type,
\begin{equation}
    \bar v(\gamma)=\bar v.
    \label{eq:weak-signal-constant-v}
\end{equation}

Using $\E_\gamma[\gamma]=1$ and $\kappa_2=\E_\gamma[\gamma^2]$, the uninformative saddle then satisfies
\begin{equation}
    1 = \alpha\beta^2 \kappa_2T_K\bar v^2.
    \label{eq:weak-signal-uninformative-norm-condition}
\end{equation}
Moreover,
\begin{equation}
    \frac{\bar q_I}{\bar q_\Gamma} = \alpha\beta^2T_K\bar v^2 = \frac1{\kappa_2}, \qquad \bar v_\Gamma=\bar v.
    \label{eq:weak-signal-qI-qGamma-vGamma-relations}
\end{equation}
Thus the full variance distribution enters the leading weak-signal threshold only through its second moment $\kappa_2$.

Finally, the magnetization equation gives
\begin{equation}
    M = \alpha\beta^2 \mathcal{X}_K\bar v\,M + o(M).
    \label{eq:weak-signal-M-closure}
\end{equation}
A continuous instability occurs when
\begin{equation}
    1 = \alpha\beta^2\mathcal{X}_K\bar v.
    \label{eq:weak-signal-magnetization-condition}
\end{equation}
Eliminating $\bar v$ between \eqref{eq:weak-signal-uninformative-norm-condition} and \eqref{eq:weak-signal-magnetization-condition} gives
\begin{equation}
    \alpha_{\rm c}^{(K)} = \frac{\kappa_2}{\beta^2} \frac{T_K}{\mathcal{X}_K^2} \left[1+o_\beta(1)\right].
    \label{eq:weak-signal-generic-threshold}
\end{equation}
The activation and the number of masks therefore enter only through the output ratio $T_K/\mathcal{X}_K^2$.

% -----------------------------------------------------------------------------
\subsubsection{Centered smooth activations with nonzero local slope}

Under \eqref{eq:weak-signal-smooth-activation-assumption}, the leading output loss is
\begin{equation}
    \ell_{K,{\rm sm}}^{(2)} = -\frac{2s_\sig}{K} \sum_{r=1}^K y_rt_r.
    \label{eq:weak-signal-smooth-leading-loss}
\end{equation}
The positive reconstruction term does not contribute to the leading force, but its $q_I$ derivative is required in \eqref{eq:weak-signal-hatv-expansion}. Since
\begin{equation}
    \Var(y_r) = (1-\rho)q_\Gamma,
\end{equation}
we obtain
\begin{equation}
    U_{K,{\rm sm}} = \rho(1-\rho)s_\sig^2.
    \label{eq:weak-signal-smooth-U}
\end{equation}

Expanding the product of the visible and masked fields,
\begin{equation}
    y_rt_r = \rho(1-\rho)h^2 + (2\rho-1) \sqrt{\rho(1-\rho)}\,hx_r - \rho(1-\rho)x_r^2.
    \label{eq:weak-signal-smooth-yrtr}
\end{equation}
The Hessian of the leading loss in the coordinates $(h,x_1,\ldots,x_K)$ is therefore
\begin{equation}
    \mathsf H_{K,{\rm sm}} =
    \begin{pmatrix}
        -4s_\sig\rho(1-\rho) & \dfrac{ 2s_\sig(1-2\rho) \sqrt{\rho(1-\rho)} }{K} \one_K^\top \\[3mm]
        \dfrac{ 2s_\sig(1-2\rho) \sqrt{\rho(1-\rho)} }{K} \one_K & \dfrac{ 4s_\sig\rho(1-\rho) }{K}\id_K
    \end{pmatrix}.
    \label{eq:weak-signal-smooth-Hessian}
\end{equation}
The preceding displacement and moment bounds give, in $L^2$,
\begin{equation}
    \widetilde{\bm F}_{K,{\rm sm}} = -\mathsf H_{K,{\rm sm}}\bm g_0 + o(\sqrt{q_\Gamma}).
    \label{eq:weak-signal-smooth-force}
\end{equation}
The informative susceptibility is therefore
\begin{equation}
    \mathcal{X}_{K,{\rm sm}} = 4s_\sig\rho(1-\rho).
    \label{eq:weak-signal-smooth-chi}
\end{equation}
Moreover,
\begin{equation}
    \begin{aligned}
    T_{K,{\rm sm}} &= \operatorname{Tr} \left( \mathsf H_{K,{\rm sm}}^2 \right) \\
    &= 16s_\sig^2\rho^2(1-\rho)^2 + \frac{ 8s_\sig^2\rho(1-\rho) }{K} \left[ (1-2\rho)^2 + 2\rho(1-\rho) \right] \\
    &= 16s_\sig^2\rho^2(1-\rho)^2 + \frac{ 8s_\sig^2\rho(1-\rho) [1-2\rho(1-\rho)] }{K}.
    \end{aligned}
    \label{eq:weak-signal-smooth-T}
\end{equation}
Hence
\begin{equation}
    \frac{ T_{K,{\rm sm}} }{ \mathcal{X}_{K,{\rm sm}}^2 } = 1+ \frac{ 1-2\rho(1-\rho) }{ 2K\rho(1-\rho) }.
    \label{eq:weak-signal-smooth-output-ratio}
\end{equation}
The slope $s_\sig$ cancels, and substitution into \eqref{eq:weak-signal-generic-threshold} gives \eqref{eq:weak-signal-smooth-threshold-K}.

For consecutive values of $K$,
\begin{equation}
    \alpha_{{\rm c},{\rm sm}}^{(K)} - \alpha_{{\rm c},{\rm sm}}^{(K+1)} = \frac{\kappa_2}{\beta^2} \frac{ 1-2\rho(1-\rho) }{ 2\rho(1-\rho)K(K+1) } \left[1+o_\beta(1)\right] >0.
    \label{eq:weak-signal-smooth-monotonicity}
\end{equation}
Indeed, $\rho(1-\rho)\le1/4$, so $1-2\rho(1-\rho)\ge1/2$.

\subsubsection{ReLU activation}

For ReLU, positive homogeneity gives the leading loss
\begin{equation}
    \ell_{K,{\rm R}}^{(2)} = -\frac2K \sum_{r=1}^K (y_r)_+t_r.
    \label{eq:weak-signal-relu-leading-loss}
\end{equation}
At $M=0$, $y_r\sim\N(0,(1-\rho)q_\Gamma)$, and hence
\begin{equation}
    U_{K,{\rm R}} = \rho \frac{ \E[(y_r)_+^2] }{ q_\Gamma } = \frac{\rho(1-\rho)}2.
    \label{eq:weak-signal-relu-U}
\end{equation}

Let $\bm e_0$ denote the $h$ direction and $\bm e_r$ the $x_r$ direction. The force contribution of view $r$ can be written as
\begin{equation}
    \bm\psi_r := \one_{\{y_r>0\}} \left[ \big( (1-\rho)t_r+\rho y_r \big)\bm e_0 + \sqrt{\rho(1-\rho)} (t_r-y_r)\bm e_r \right].
    \label{eq:weak-signal-relu-psi}
\end{equation}
The leading force is then
\begin{equation}
    \widetilde{\bm F}_{K,{\rm R}} = \frac2K \sum_{r=1}^K \bm\psi_r.
    \label{eq:weak-signal-relu-force}
\end{equation}

\paragraph{Informative susceptibility.}

The common component of $\bm\psi_r$ is
\begin{equation}
    (\bm\psi_r)_h = \one_{\{y_r>0\}} \left[ (1-\rho)t_r+\rho y_r \right].
    \label{eq:weak-signal-relu-psi-h}
\end{equation}
Evaluate \eqref{eq:weak-signal-chi-definition} with the leading force. At $M=0$, $h=y_r+t_r$, while $y_r$ and $t_r$ are independent centered Gaussians with variances $(1-\rho)q_\Gamma$ and $\rho q_\Gamma$. Consequently,
\[
    \E[h(\bm\psi_r)_h] =(1-\rho)\E[t_r^2]\mathbb P(y_r>0) +\rho\E[y_r^2\one_{\{y_r>0\}}] =\rho(1-\rho)q_\Gamma.
\]
Since $h=\sqrt{q_\Gamma}z_h$, summing the $K$ view contributions in \eqref{eq:weak-signal-relu-force} and dividing by $q_\Gamma$ gives
\begin{equation}
    \mathcal{X}_{K,{\rm R}} = 2\rho(1-\rho).
    \label{eq:weak-signal-relu-chi}
\end{equation}

\paragraph{Force variance: diagonal term.}

A direct expansion gives
\begin{equation}
    \|\bm\psi_r\|_2^2 = \one_{\{y_r>0\}} \left[ (1-\rho)t_r^2 + \rho y_r^2 \right].
    \label{eq:weak-signal-relu-psi-norm}
\end{equation}
At $q_\Gamma=1$,
\begin{equation}
    y_r\sim\N(0,1-\rho), \qquad t_r\sim\N(0,\rho),
\end{equation}
and the two variables are independent. Therefore
\begin{equation}
    \E\|\bm\psi_r\|_2^2 = \rho(1-\rho).
    \label{eq:weak-signal-relu-diagonal-force-moment}
\end{equation}

\paragraph{Force variance: off-diagonal term.}

For $r\ne s$, the mask-specific directions $\bm e_r$ and $\bm e_s$ are orthogonal, so only the common components contribute to $\bm\psi_r^\top\bm\psi_s$. We have
\begin{equation}
    (1-\rho)t_r+\rho y_r = 2\rho(1-\rho)h + (2\rho-1) \sqrt{\rho(1-\rho)}\,x_r.
    \label{eq:weak-signal-relu-common-force-component}
\end{equation}
Conditioning on $h=H$, with $H\sim\N(0,1)$, the variables $x_r$ and $x_s$ are independent.

Let
\begin{equation}
    \varphi(u) := \frac{e^{-u^2/2}}{\sqrt{2\pi}}, \qquad \Phi_{\rm G}(u) := \int_{-\infty}^u \varphi(v)\,\dd v.
    \label{eq:weak-signal-relu-normal-functions}
\end{equation}
The conditional expectation of the common force component is
\begin{equation}
    \begin{aligned}
    &\E_x \Bigg[ \one_{ \{ (1-\rho)H + \sqrt{\rho(1-\rho)}\,x >0 \} } \times \left( 2\rho(1-\rho)H + (2\rho-1) \sqrt{\rho(1-\rho)}\,x \right) \Bigg] \\
    &= 2\rho(1-\rho)H \Phi_{\rm G} \left( \sqrt{\frac{1-\rho}{\rho}}\,H \right) + (2\rho-1) \sqrt{\rho(1-\rho)} \, \varphi \left( \sqrt{\frac{1-\rho}{\rho}}\,H \right).
    \end{aligned}
    \label{eq:weak-signal-relu-conditional-common-force}
\end{equation}
Consequently,
\begin{equation}
    \mathcal D_{\rm R}(\rho) = \E[ \bm\psi_r^\top\bm\psi_s ], \qquad r\ne s,
    \label{eq:weak-signal-relu-D-definition}
\end{equation}
is the expectation over $H$ of the square of \eqref{eq:weak-signal-relu-conditional-common-force}.

The required Gaussian identities are
\begin{gather}
    \E \left[ \varphi\!\left( \sqrt{\frac{1-\rho}{\rho}}H \right)^2 \right] = \frac1{2\pi} \sqrt{\frac{\rho}{2-\rho}},
    \label{eq:weak-signal-relu-gaussian-1}
    \\
    \E \left[ H \Phi_{\rm G}\!\left( \sqrt{\frac{1-\rho}{\rho}}H \right) \varphi\!\left( \sqrt{\frac{1-\rho}{\rho}}H \right) \right] = \frac{ \sqrt{\rho(1-\rho)} }{ 2\pi } \sqrt{\frac{\rho}{2-\rho}},
    \label{eq:weak-signal-relu-gaussian-2}
    \\
    \E \left[ H^2 \Phi_{\rm G}\!\left( \sqrt{\frac{1-\rho}{\rho}}H \right)^2 \right] = \frac14 + \frac{\arcsin(1-\rho)}{2\pi} + \frac{1-\rho}{\pi} \sqrt{\frac{\rho}{2-\rho}}.
    \label{eq:weak-signal-relu-gaussian-3}
\end{gather}
Substituting these identities into \eqref{eq:weak-signal-relu-D-definition} gives the closed expression \eqref{eq:weak-signal-relu-D-rho}.

Combining the diagonal and off-diagonal contributions,
\begin{equation}
    T_{K,{\rm R}} = 4 \left[ \mathcal D_{\rm R}(\rho) + \frac{ \rho(1-\rho) - \mathcal D_{\rm R}(\rho) }{K} \right].
    \label{eq:weak-signal-relu-T}
\end{equation}
Together with \eqref{eq:weak-signal-relu-chi},
\begin{equation}
    \frac{ T_{K,{\rm R}} }{ \mathcal{X}_{K,{\rm R}}^2 } = \frac{ \mathcal D_{\rm R}(\rho) }{ \rho^2(1-\rho)^2 } + \frac1K \left[ \frac1{\rho(1-\rho)} - \frac{ \mathcal D_{\rm R}(\rho) }{ \rho^2(1-\rho)^2 } \right].
    \label{eq:weak-signal-relu-output-ratio}
\end{equation}
Substitution into \eqref{eq:weak-signal-generic-threshold} gives \eqref{eq:weak-signal-relu-threshold-K}.

To establish strict monotonicity, Cauchy--Schwarz gives
\begin{equation}
    \mathcal D_{\rm R}(\rho) = \E[ \bm\psi_r^\top\bm\psi_s ] \le \sqrt{ \E\|\bm\psi_r\|_2^2 \E\|\bm\psi_s\|_2^2 } = \rho(1-\rho).
\end{equation}
The inequality is strict for $0<\rho<1$, because independently masked views do not generate identical force vectors almost surely. Hence
\begin{equation}
    \alpha_{{\rm c},{\rm R}}^{(K)} - \alpha_{{\rm c},{\rm R}}^{(K+1)} = \frac{\kappa_2}{\beta^2} \frac{ \rho(1-\rho)-\mathcal D_{\rm R}(\rho) }{ \rho^2(1-\rho)^2K(K+1) } \left[1+o_\beta(1)\right] >0.
    \label{eq:weak-signal-relu-monotonicity}
\end{equation}

\input{appendices/appendix_dynamic_linear_bound.tex}
\input{appendices/appendix_mask_adaptation.tex}
\FloatBarrier

\section{Proofs of the linear reconstruction separation}
\label{app:rigorous-separation}

\begin{assumption}
\label{as:main-separation}
We collect the precise hypotheses for \Cref{thm:main-separation}. The activation is linear, $\sig(h)=h$, and $\laR=0$. The noise law is fixed as the dimension grows; the independence assumptions are those of \eqref{eq:main-data-model}.
\begin{enumerate}[label=\textup{(A\arabic*)},leftmargin=*,itemsep=1pt]
  \item\label{as:model} \textup{(Model and scaling.)} Data are generated by \eqref{eq:main-data-model} with $P_u=P_\lambda=\N(0,1)$; the signal strength $\beta>0$, the masking ratio $\rho\in(0,1)$ and the sample complexity $\alpha=n/d$ are held fixed as $n,d\to\infty$.
  \item\label{as:noise} \textup{(Coordinate noise.)} $P_\gamma$ is a mean-one distribution supported in $[a,\infty)$ for some $a>0$, with $\kappa_2=\E_\gamma[\gamma^2]<\infty$. We write $\gamma_+=\operatorname{ess\,sup}\gamma$, allowing $\gamma_+=\infty$.
  \item\label{as:pca-fail} \textup{(Bounded failure conditions.)} In addition to \ref{as:model} and \ref{as:noise}, $\gamma_+<\infty$ and
        \begin{equation}
            \beta\,\E_\gamma\frac1{\gamma_+-\gamma}<1, \qquad \E_\gamma\frac{\gamma^2}{(\gamma_+-\gamma)^2}=\infty.
            \label{eq:main-bounded-separation-conditions}
        \end{equation}
\end{enumerate}
\end{assumption}

We establish the equivalence of full linear reconstruction and PCA, prove the PCA obstructions, and give finite sufficient sample bounds for dynamic masking. The bounded-noise proof is followed by its extension to variance laws without an upper support bound.

\input{appendices/appendix_linear_separation.tex}
\input{appendices/appendix_dynamic_unbounded.tex}

\section{Bayes-optimal pretraining and AMP}
\label{app:bayes-optimal}

The benchmark in this appendix conditions on the realized variance labels and uses priors matched to the data-generating model. We derive its scalar state evolution and specialize the bipartite AMP of \citet{lesieur2017constrained} in \Cref{subsec:BO-AMP}. The calculation of variance consistency below is distinct from a proof that an estimated-variance algorithm attains the same benchmark.

\subsection{Mapping to bipartite low-rank estimation}
We map the heteroskedastic spiked model to the bipartite low-rank output-channel formulation of \citep{lesieur2015mmse}:
\begin{equation}
    x_{\mu i} = \sqrt{\frac{\beta}{d}}\,\lambda_\mu u_i^\star + \sqrt{\gamma_i^\star}\,z_{\mu i}, \qquad \mu=1,\ldots,n, \quad i=1,\ldots,d,
    \label{eq:BO-original-model}
\end{equation}
with
\begin{equation}
    \frac nd\longrightarrow\alpha, \qquad \lambda_\mu\stackrel{\mathrm{iid}}{\sim}P_\lambda, \qquad u_i^\star\stackrel{\mathrm{iid}}{\sim}P_u, \qquad \gamma_i^\star\stackrel{\mathrm{iid}}{\sim}P_\gamma, \qquad z_{\mu i}\stackrel{\mathrm{iid}}{\sim}\mathcal N(0,1).
\end{equation}
We assume throughout that
\begin{equation}
    \mathbb E[u^\star]=\mathbb E[\lambda]=0, \qquad \mathbb E[(u^\star)^2]=\mathbb E[\lambda^2]=1,
\end{equation}
and that the moments needed below are finite.  In particular, for the linearized weak-recovery calculation we assume
\begin{equation}
    \mathbb E_\gamma[\gamma^{-2}]<\infty.
    \label{eq:BO-negative-moment-assumption}
\end{equation}

We retain $n,d$ for the sample and coordinate counts, respectively, and the sample-first data matrix $X=(x_{\mu i})$. Relative to the cited bipartite formulation, its coordinate and sample counts correspond to $d$ and $n$, and its observation matrix is $X^\top$. We identify
\begin{equation}
    P_U=P_u, \qquad P_V=P_\lambda, \qquad v_\mu^\star=\lambda_\mu.
\end{equation}
The rank-one variable entering the output channel is
\begin{equation}
    w_{\mu i}^\star := \frac{u_i^\star v_\mu^\star}{\sqrt d}.
\end{equation}
The SNR parameter $\beta$ is absorbed into the output channel, which, conditional on the variance label $\gamma$, is
\begin{equation}
    P_{\mathrm{out},\gamma}(x\mid w) = \frac{1}{\sqrt{2\pi\gamma}} \exp\left\{ -\frac{(x-\sqrt\beta\,w)^2}{2\gamma} \right\}.
    \label{eq:BO-output-channel}
\end{equation}
Thus
\begin{equation}
    x_{\mu i} = \sqrt\beta\,w_{\mu i}^\star + \sqrt{\gamma_i^\star}\,z_{\mu i},
\end{equation}
which is exactly Eq.~\eqref{eq:BO-original-model} in the sample-first convention.

\subsection{Asymptotic identifiability of the coordinate variances}
\label{subsec:BO-gamma-identifiability}

The bipartite formulas of \citep{lesieur2017constrained} are written for a known output channel. The state-evolution reference conditions on the realized coordinate variances. In the numerical AMP implementation, we instead estimate these variances from the unlabeled observations.

For a fixed coordinate \(i\), define the empirical second moment
\begin{equation}
    \widehat\gamma_i := \frac1n\sum_{\mu=1}^n x_{\mu i}^2 .
    \label{eq:BO-empirical-gamma}
\end{equation}
Using the data model \eqref{eq:BO-original-model},
\begin{equation}
    \widehat\gamma_i = \gamma_i^\star \frac1n\sum_{\mu=1}^n z_{\mu i}^2 + \frac{\beta (u_i^\star)^2}{d} \frac1n\sum_{\mu=1}^n\lambda_\mu^2 + 2\sqrt{\frac{\beta\gamma_i^\star}{d}}\, u_i^\star \frac1n\sum_{\mu=1}^n \lambda_\mu z_{\mu i}.
    \label{eq:BO-gamma-decomposition-cross}
\end{equation}

We now estimate the three terms separately.  Since \(z_{\mu i}\stackrel{\mathrm{iid}}{\sim}\mathcal N(0,1)\),
\begin{equation}
    \frac1n\sum_{\mu=1}^n z_{\mu i}^2 = 1+O_{\mathbb P}(n^{-1/2}) = 1+O_{\mathbb P}(d^{-1/2}),
    \label{eq:BO-noise-second-moment}
\end{equation}
where the last equality uses \(n/d\to\alpha\in(0,\infty)\) in the high-dimensional limit. Therefore the first term is
\begin{equation}
    \gamma_i^\star \frac1n\sum_{\mu=1}^n z_{\mu i}^2 = \gamma_i^\star+O_{\mathbb P}(d^{-1/2}).
    \label{eq:BO-leading-variance-term}
\end{equation}
Because \(\mathbb E[\lambda^2]=1\), we have that
\begin{equation}
    \frac1n\sum_{\mu=1}^n\lambda_\mu^2 = 1+o(1).
\end{equation}
For each fixed coordinate $i$, $u_i^\star=O_{\mathbb P}(1)$, and hence the second term is subleading
\begin{equation}
    \frac{\beta (u_i^\star)^2}{d} \frac1n\sum_{\mu=1}^n\lambda_\mu^2 = O_{\mathbb P}(d^{-1}).
    \label{eq:BO-signal-bias-order}
\end{equation}
Finally, the variables \(\lambda_\mu z_{\mu i}\) are centered and have unit variance, so
\begin{equation}
    \frac1n\sum_{\mu=1}^n \lambda_\mu z_{\mu i} = O_{\mathbb P}(n^{-1/2}) = O_{\mathbb P}(d^{-1/2}).
\end{equation}
Multiplying by the prefactor \(d^{-1/2}\) shows that the cross term is
\begin{equation}
    2\sqrt{\frac{\beta\gamma_i^\star}{d}}\, u_i^\star \frac1n\sum_{\mu=1}^n \lambda_\mu z_{\mu i} = O_{\mathbb P}(d^{-1}).
    \label{eq:BO-cross-term-order}
\end{equation}

Combining Eqs.~\eqref{eq:BO-leading-variance-term}, \eqref{eq:BO-signal-bias-order}, and \eqref{eq:BO-cross-term-order}, we obtain
\begin{equation}
        \widehat\gamma_i = \gamma_i^\star+O_{\mathbb P}(d^{-1/2}).
    \label{eq:BO-gamma-consistency}
\end{equation}
In particular,
\begin{equation}
    \widehat\gamma_i \xrightarrow{\mathbb P} \gamma_i^\star .
    \label{eq:BO-gamma-convergence}
\end{equation}

Thus each fixed coordinate variance is consistently estimable. In the Gaussian-prior specialization with $\gamma_i^\star\geq a>0$, one can also obtain uniform relative consistency. Conditionally on the teacher and variances, $\widehat\gamma_i/[\gamma_i^\star+\beta(u_i^\star)^2/d]$ has law $\chi_n^2/n$. A union bound and $\max_i(u_i^\star)^2=O_{\mathbb P}(\log d)$ give
\begin{equation}
 \max_i\left|\frac{\widehat\gamma_i}{\gamma_i^\star}-1\right| =O_{\mathbb P}\left(\sqrt{\frac{\log d}{n}} +\frac{\log d}{d}\right) \longrightarrow0.
 \label{eq:BO-gamma-uniform-relative}
\end{equation}
Neither pointwise nor uniform parameter consistency alone establishes that substituting all $\widehat\gamma_i$ into the posterior or AMP preserves its limiting overlap and threshold. That substitution also introduces dependence through the data used to estimate the variances. We therefore state the following equations for the known-variance benchmark; equivalence of its plug-in version requires an additional stability argument.

\subsection{Bayes-optimal output channel and Fisher information}
Conditioned on $\gamma_i^\star$, define the output channel
\begin{equation}
    g_i(x,w) := \log P_{\mathrm{out},\gamma_i^\star}(x\mid w).
\end{equation}
The posterior can be written in the normalized form used in the bipartite low-rank results:
\begin{align}
    P(U,V\mid X,\bm\gamma^\star) \propto{} & \prod_{i=1}^d P_U(\mathrm du_i) \prod_{\mu=1}^n P_V(\mathrm dv_\mu) \nonumber          \\
              & \times \exp\left\{ \sum_{i=1}^d\sum_{\mu=1}^n \left[ g_i\left(x_{\mu i},\frac{u_iv_\mu}{\sqrt d}\right) -g_i(x_{\mu i},0) \right] \right\}.
    \label{eq:BO-posterior-lesieur-form}
\end{align}

Following the Fisher-score notation
\begin{gather}
    S_i(x) := \left. \frac{\partial g_i(x,w)}{\partial w} \right|_{w=0} \\
    R_i(x) := \left. \left[ \left(\frac{\partial g_i(x,w)}{\partial w}\right)^2 + \frac{\partial^2g_i(x,w)}{\partial w^2} \right] \right|_{w=0}.
\end{gather}
In particular, for the Gaussian channel \eqref{eq:BO-output-channel},
\begin{equation}
    S_i(x) = \frac{\sqrt\beta}{\gamma_i^\star}x,
    \label{eq:BO-score-Gaussian}
\end{equation}
whereas
\begin{equation}
    R_i(x) = \beta\left( \frac{x^2}{(\gamma_i^\star)^2} - \frac1{\gamma_i^\star} \right).
    \label{eq:BO-R-Gaussian}
\end{equation}
Under the null output law $x\sim\mathcal N(0,\gamma_i^\star)$,
\begin{equation}
    \mathbb E[S_i(x)]=0, \qquad \mathbb E[R_i(x)]=0.
\end{equation}
The local Fisher-information parameter is defined by
\begin{equation}
    \frac1{\Delta_i} := \mathbb E_{x\sim P_{\mathrm{out},\gamma_i^\star}(\cdot\mid0)} \left[S_i(x)^2\right].
    \label{eq:BO-local-Delta-definition}
\end{equation}
and using Eq.~\eqref{eq:BO-score-Gaussian},
\begin{equation}
    \frac1{\Delta_i} = \frac{\beta}{\gamma_i^\star}, \qquad \Delta_i = \Delta(\gamma_i^\star) = \frac{\gamma_i^\star}{\beta}.
    \label{eq:BO-local-Delta-Gaussian}
\end{equation}
Thus heteroskedasticity enters the Bayes-optimal low-rank equations only in the site-dependent Fisher information
\begin{equation}
    \Delta(\gamma)^{-1}=\beta\gamma^{-1}.
\end{equation}

\subsection{Denoisers}

For a prior $P_U$, define the scalar partition function
\begin{equation}
    Z_U(A,B) := \int P_U(\mathrm du) \exp\left\{ Bu-\frac A2u^2 \right\},
    \label{eq:BO-ZU}
\end{equation}
and the posterior-mean denoiser
\begin{equation}
    f_{\mathrm{in}}^U(A,B) := \frac{\partial}{\partial B}\log Z_U(A,B) = \frac{ \int P_U(\mathrm du)\,u \exp\{Bu-Au^2/2\} }{ Z_U(A,B) }.
    \label{eq:BO-fin-U}
\end{equation}
Define $Z_V$ and $f_{\mathrm{in}}^V$ analogously with $P_V=P_\lambda$.

For a scalar Bayes channel of $A$ introduce the overlap response functions
\begin{align}
    \mathcal F_U(A) :={} & \mathbb E_{u^\star,W} \left[ u^\star f_{\mathrm{in}}^U \left(A,Au^\star+\sqrt A\,W\right) \right],
    \label{eq:BO-FU-definition}
    \\
    \mathcal F_V(A) :={} & \mathbb E_{v^\star,W} \left[ v^\star f_{\mathrm{in}}^V \left(A,Av^\star+\sqrt A\,W\right) \right],
    \label{eq:BO-FV-definition}
\end{align}
where
\begin{equation}
    u^\star\sim P_U, \qquad v^\star\sim P_V, \qquad W\sim\mathcal N(0,1),
\end{equation}
independently.  By the Nishimori identity,
\begin{align}
    \mathcal F_U(A) & = \mathbb E_{u^\star,W} \left[ f_{\mathrm{in}}^U \left(A,Au^\star+\sqrt A\,W\right)^2 \right], \\
    \mathcal F_V(A) & = \mathbb E_{v^\star,W} \left[ f_{\mathrm{in}}^V \left(A,Av^\star+\sqrt A\,W\right)^2 \right].
\end{align}

\input{appendices/appendix_bayes_amp.tex}

\subsection{State evolution}
\label{subsec:BO-SE-lesieur}

Let
\begin{equation}
    M_v^t := \frac1n\sum_{\mu=1}^n \widehat v_\mu^t v_\mu^\star
\end{equation}
denote the overlap on the latent variables side. On the coordinate side (spike), the effective channel depends on the quenched variances. For a site with label $\gamma$, the bipartite state-evolution field has to be rescaled with
\begin{equation}
    A_U^t(\gamma) = \frac{\alpha M_v^t}{\Delta(\gamma)}.
    \label{eq:BO-AU-general}
\end{equation}
Therefore the conditional coordinate overlap is
\begin{equation}
    M_u^t(\gamma) = \mathcal F_U\left( \frac{\alpha M_v^t}{\Delta(\gamma)} \right).
    \label{eq:BO-Mu-gamma-general}
\end{equation}

and asymptotically
\begin{equation}
    \overline M_u^t := \lim_{d\to\infty} \frac1d\sum_{i=1}^d \frac{M_{u,i}^t}{\Delta_i} = \mathbb E_\gamma \left[ \frac{M_u^t(\gamma)}{\Delta(\gamma)} \right].
    \label{eq:BO-Mubar-general}
\end{equation}
The $V$-side update is consequently
\begin{equation}
    M_v^{t+1} = \mathcal F_V\left(\overline M_u^t\right).
    \label{eq:BO-Mv-update-general}
\end{equation}

For the heteroskedastic Gaussian channel $\Delta(\gamma)=\gamma/\beta$, these equations become
\begin{gather}
    M_u^t(\gamma) = \mathcal F_U\left( \frac{\alpha\beta}{\gamma}M_v^t \right),
    \label{eq:BO-Mu-gamma-Gaussian}
    \\
    \overline M_u^t = \beta\, \mathbb E_\gamma \left[ \frac{M_u^t(\gamma)}{\gamma} \right],
    \label{eq:BO-Mubar-Gaussian}
    \\
    M_v^{t+1} = \mathcal F_V\left(\overline M_u^t\right).
    \label{eq:BO-Mv-update-Gaussian}
\end{gather}

\subsection{Fixed-point equations in the present notation}

Returning to
\begin{equation}
    v_\mu^\star=\lambda_\mu, \qquad M_v=q_\lambda, \qquad M_u(\gamma)=q_u(\gamma),
\end{equation}
the coordinate-side scalar precision is
\begin{equation}
    \widehat q_u(\gamma) = \frac{\alpha\beta}{\gamma}q_\lambda.
    \label{eq:BO-qhat-u-original}
\end{equation}
The sample-amplitude scalar precision is the Fisher-weighted coordinate overlap
\begin{equation}
    \widehat q_\lambda = \overline M_u = \beta\mathbb E_\gamma \left[ \frac{q_u(\gamma)}{\gamma} \right].
    \label{eq:BO-qhat-lambda-original}
\end{equation}
Thus the fixed-point equations are
\begin{gather}
    q_u(\gamma) = \mathbb E_{u^\star,W} \left[ u^\star f_{\mathrm{in}}^u\left( \widehat q_u(\gamma), \widehat q_u(\gamma)u^\star + \sqrt{\widehat q_u(\gamma)}\,W \right) \right],
    \label{eq:BO-qu-original}
    \\
    q_\lambda = \mathbb E_{\lambda,W} \left[ \lambda f_{\mathrm{in}}^\lambda\left( \widehat q_\lambda, \widehat q_\lambda\lambda + \sqrt{\widehat q_\lambda}\,W \right) \right],
    \label{eq:BO-qlambda-original}
\end{gather}
with conjugate relations
\begin{equation}
        \widehat q_u(\gamma) = \frac{\alpha\beta}{\gamma}q_\lambda, \qquad \widehat q_\lambda = \beta\mathbb E_\gamma \left[ \frac{q_u(\gamma)}{\gamma} \right].
    \label{eq:BO-conjugates-original}
\end{equation}
By Nishimori,
\begin{align}
    q_u(\gamma) & = \mathbb E_{u^\star,W} \left[ f_{\mathrm{in}}^u\left( \widehat q_u(\gamma), \widehat q_u(\gamma)u^\star + \sqrt{\widehat q_u(\gamma)}W \right)^2 \right], \\
    q_\lambda & = \mathbb E_{\lambda,W} \left[ f_{\mathrm{in}}^\lambda\left( \widehat q_\lambda, \widehat q_\lambda\lambda + \sqrt{\widehat q_\lambda}W \right)^2 \right].
\end{align}

\subsection{Continuous weak-recovery threshold}
\label{subsec:BO-weak-threshold}

The uninformative fixed point is
\begin{equation}
    q_\lambda=0, \qquad q_u(\gamma)=0.
\end{equation}
On the matched Nishimori trajectory, and under the scalar-response regularity used in the linearization, at small order
\begin{equation}
    \mathcal F_U(A)=A+o(A), \qquad \mathcal F_V(A)=A+o(A).
    \label{eq:BO-small-A-expansion}
\end{equation}
Linearizing Eq.~\eqref{eq:BO-Mu-gamma-Gaussian} around $M_v=0$ gives
\begin{equation}
    M_u^t(\gamma) = \frac{\alpha\beta}{\gamma}M_v^t + o(M_v^t).
    \label{eq:BO-linear-Mu}
\end{equation}
Consequently,
\begin{equation}
    \overline M_u^t = \beta\mathbb E_\gamma \left[ \frac{M_u^t(\gamma)}{\gamma} \right] = \alpha\beta^2 \mathbb E_\gamma[\gamma^{-2}] M_v^t + o(M_v^t).
    \label{eq:BO-linear-Mubar}
\end{equation}
The $V$-side update then gives
\begin{equation}
    M_v^{t+1} = \alpha\beta^2 \mathbb E_\gamma[\gamma^{-2}] M_v^t + o(M_v^t).
    \label{eq:BO-linear-map}
\end{equation}
For Gaussian priors, dominated convergence justifies the integrated linear coefficient with only \eqref{eq:BO-negative-moment-assumption}. A quadratic remainder generally requires an inverse third moment; the linear threshold does not require that stronger expansion. Therefore the uninformative fixed point is locally stable when
\begin{equation}
    \alpha\beta^2 \mathbb E_\gamma[\gamma^{-2}]<1
\end{equation}
and loses stability when
\begin{equation}
    \alpha\beta^2 \mathbb E_\gamma[\gamma^{-2}]>1.
\end{equation}
The Bayes-optimal continuous weak-recovery threshold is consequently
\begin{equation}
    \alpha_{\mathrm c}^{\mathrm{BO}} = \frac{1}{ \beta^2\mathbb E_\gamma[\gamma^{-2}] }.
    \label{eq:BO-final-threshold}
\end{equation}
This is the local continuous instability of the uninformative Bayes fixed point; for generic priors it does not by itself exclude an informative branch appearing discontinuously at a smaller sample complexity. For the homogeneous Gaussian noise-specialization $P_\gamma=\delta_1$, this reduces to
\begin{equation}
    \alpha_{\mathrm c}^{\mathrm{BO}} = \frac1{\beta^2},
\end{equation}

\subsection{Gaussian-prior specialization}
\label{subsec:BO-Gaussian-priors}

If
\begin{equation}
    P_u=P_\lambda=\mathcal N(0,1),
\end{equation}
then
\begin{equation}
    Z(A,B) = \frac1{\sqrt{1+A}} \exp\left\{ \frac{B^2}{2(1+A)} \right\},
\end{equation}
so
\begin{gather}
    f_{\mathrm{in}}(A,B) = \frac{B}{1+A} \\
    \mathcal F(A) = \frac{A}{1+A}.
\end{gather}
The coordinate overlap is therefore
\begin{equation}
    q_u(\gamma) = \frac{ \alpha\beta q_\lambda }{ \gamma+\alpha\beta q_\lambda }.
    \label{eq:BO-Gaussian-prior-qu}
\end{equation}
The sample-side precision is
\begin{gather}
    \widehat q_\lambda = \beta\mathbb E_\gamma \left[ \frac{q_u(\gamma)}{\gamma} \right] = \alpha\beta^2q_\lambda \mathbb E_\gamma \left[ \frac1{ \gamma(\gamma+\alpha\beta q_\lambda) } \right],
    \label{eq:BO-Gaussian-prior-qhatlambda}
    \\
    q_\lambda = \frac{\widehat q_\lambda}{1+\widehat q_\lambda}.
    \label{eq:BO-Gaussian-prior-qlambda}
\end{gather}
Linearizing Eqs.~\eqref{eq:BO-Gaussian-prior-qu}--\eqref{eq:BO-Gaussian-prior-qlambda} again yields Eq.~\eqref{eq:BO-final-threshold}.

\paragraph{Complete overlap curve.} For $q_\lambda>0$, the preceding equations are equivalent to the single scalar equation
\begin{equation}
 \frac1{1-q_\lambda} =\alpha\beta^2\E_\gamma \frac1{\gamma(\gamma+\alpha\beta q_\lambda)}, \qquad 0<q_\lambda<1.
 \label{eq:BO-complete-scalar-root}
\end{equation}
The left side strictly increases from one to infinity; the right side strictly decreases from $\alpha\beta^2\E_\gamma\gamma^{-2}$. Hence a positive fixed point exists if and only if $\alpha\beta^2\E_\gamma\gamma^{-2}>1$, and it is unique within these Gaussian-prior equations. Below and at this threshold only the uninformative fixed point remains. This excludes a competing positive fixed point within this scalar reduction, not an algorithmic failure caused by its initialization or finite-dimensional dynamics.

The Bayes-optimal cosine similarity plotted in the figures is $\sqrt{\mathcal O_{\rm BO}}$, where $\mathcal O_{\rm BO}:=\mathbb E_\gamma[q_u(\gamma)]$. For homogeneous noise, elimination of $q_\lambda$ recovers $\mathcal O_{\rm BO}=(\alpha\beta^2-1)/(\alpha\beta^2+\beta)$ above threshold, in agreement with the PCA benchmark.

For Gaussian priors and known covariance with compact positive variance support, the weak-estimation threshold also has an independent rigorous justification through Proposition~4.1 and Corollary~4.3 of \citet{ZhangMondelli2024}, after transposing the model and identifying their Section~4 SNR with $\beta$ and aspect ratio with $\alpha$. That result concerns sign-invariant estimation and has covariance-support assumptions; it does not by itself justify plug-in AMP or the full unwhitened cosine curve for arbitrary unbounded laws.

% =============================================================================
\newpage

\section{Additional observables and finite-shot downstream analysis}
\label{app:additional-observables}
% =============================================================================
\subsection{Training and test errors}
\label{subsec:observables}
Besides the cosine similarity defined in the main text, we can also track the training and test error of the pretraining for the autoencoder. Using the objective $\mathcal R_{K,d}$ in \eqref{eq:main-finite-K-objective}, we define the centered training reconstruction error per sample by removing the ridge penalty and subtracting the risk assuming $\widehat \bw=\mathbf{0}$,
\begin{equation}
    \Delta\mathcal L_{{\rm train},d} :=\frac1n\left[ \mathcal R_{K,d}(\widehat\bw)-\mathcal R_{K,d}(\mathbf 0) -\frac{\laR}{2}\|\widehat\bw\|_2^2 \right].
    \label{eq:train-error}
\end{equation}
For dynamic masking, the view average is replaced by an expectation over fresh masks on the training samples. The centered population test error uses an independent sample and an independent fresh mask:
\begin{equation}
    \Delta\mathcal L_{{\rm test},d} :=\E_{\bx_{\rm new},\bchi_{\rm new}} \left[ \left\|\bchi_{\rm new}\odot \left[\bx_{\rm new}-\frac{\widehat\bw}{\sqrt d} \sig(h_{\vis}^{\rm new})\right]\right\|_2^2 -\|\bchi_{\rm new}\odot\bx_{\rm new}\|_2^2 \right].
    \label{eq:test-error}
\end{equation}
Both quantities exclude the ridge penalty. Lower values are better. In the simulations, the test expectation is estimated using independent test samples and masks.

At the replica saddle, these observables are
\begin{equation}
    \Delta\mathcal L_{\rm train} =\E_{\lambda,\bz} [\ell_{\out}(\bfield_\star;q_I,q_\Gamma)],
    \label{eq:train-error-RS}
\end{equation}
where $\bfield_\star$ is the training-sample proximal field, and
\begin{equation}
    \Delta\mathcal L_{\rm test} =\E_{\lambda,z}[\ell_{\rm dyn}(h_0;q_I,q_\Gamma)], \qquad h_0=\sqrt{\beta}\lambda M+\sqrt{q_\Gamma}\,z, \quad z\sim\N(0,1).
    \label{eq:test-error-RS}
\end{equation}

\paragraph{Why center the reconstruction error?} The population test loss for the null estimator is
\begin{equation}
    \mathcal L_{0,{\rm test},d} =\rho\sum_{i=1}^d\gamma_i +\frac{\rho\beta}{d}\|\bu^\star\|_2^2 \sim\rho d,
\end{equation}
Under the order-one scaling of the replica solution, the reconstruction-dependent terms are only $O(1)$, so this $O(d)$ baseline dominates the raw test loss. Subtracting it reduces to:
\begin{equation}
    \Delta\mathcal L_{{\rm test},d} =\E\left[\sum_i\chi_i^{\rm new} \left(\widehat x_i^2-2x_i^{\rm new}\widehat x_i\right)\right], \qquad \widehat\bx=\frac{\widehat\bw}{\sqrt d}\sig(h_{\vis}^{\rm new}).
\end{equation}
The centered error therefore isolates the $O(1)$ contribution that changes with the learned representation. The same cancellation applies to the training error, using its own samples and masks.

\subsection{Downstream task}
\label{subsec:downstream-classification}
% =============================================================================

After self-supervised pretraining, we freeze the learned weight vector $\widehat{\bw}$ and evaluate it on an independent labeled dataset
\begin{equation}
    \mathcal D_{\rm ds} = \left\{ (\bx_\mu^{\rm ds},y_\mu) \right\}_{\mu=1}^{n_{\rm ds}}, \qquad y_\mu=\operatorname{sign}(\lambda_\mu), \qquad n_{\rm ds}=O(1).
\end{equation}
We assume $P_\lambda(\{0\})=0$, so labels lie in $\{-1,+1\}$. No masking is applied at downstream time.

For a fresh downstream sample, the frozen preactivation satisfies
\begin{equation}
    H := \frac1{\sqrt d}\widehat{\bw}^{\top}\bx^{\rm ds} \dto \sqrt{\beta}\,M\lambda + \sqrt{q_\Gamma}\,Z, \qquad Z\sim\N(0,1).
    \label{eq:downstream-scalar-channel}
\end{equation}
We use this preactivation directly as the downstream representation. The downstream channel therefore depends only on the pretraining overlaps $M$ and $q_\Gamma$: $M$ controls the informative component and $q_\Gamma$ the projected noise variance.

We fit a one-dimensional logistic probe,
\begin{equation}
    \widehat a = \argmin_{a\in\R} \sum_{\mu=1}^{n_{\rm ds}} \log\!\left( 1+\exp[-y_\mu a H_\mu] \right).
    \label{eq:downstream-logistic-probe}
\end{equation}
Its orientation is determined by
\begin{equation}
    \operatorname{sign}(\widehat a) = \operatorname{sign} \left( \sum_{\mu=1}^{n_{\rm ds}}y_\mu H_\mu \right).
    \label{eq:downstream-probe-orientation}
\end{equation}
If the unregularized probe is separable, this orientation is understood along a minimizing sequence. Adding a positive quadratic penalty makes the minimizer finite and preserves the same orientation rule.

Assume $M\ne0$ in the following orientation formulas. The population-optimal orientation of the frozen feature is
\begin{equation}
    s_\star = \operatorname{sign} \left( \E_{\lambda,Z} \left[ \operatorname{sign}(\lambda) \left( \sqrt{\beta}\,M\lambda + \sqrt{q_\Gamma}\,Z \right) \right] \right) = \operatorname{sign}(M),
    \label{eq:downstream-population-orientation}
\end{equation}
where we used
\begin{equation}
    \E_{\lambda,Z} \left[ \operatorname{sign}(\lambda)H \right] = \sqrt{\beta}\,M\,\E_\lambda[|\lambda|].
\end{equation}
The corresponding population classification error is
\begin{equation}
    \epsilon_{\rm ds}^{\infty} = \mathbb P \left[ s_\star\, \operatorname{sign}(\lambda)\, \left( \sqrt{\beta}\,M\lambda + \sqrt{q_\Gamma}\,Z \right) <0 \right].
    \label{eq:downstream-population-error}
\end{equation}

For $n_{\rm ds}=O(1)$ labeled examples, the probability that the fitted probe chooses the correct population orientation is
\begin{equation}
    p_{n_{\rm ds}} = \mathbb P \left[ s_\star \sum_{\mu=1}^{n_{\rm ds}} \operatorname{sign}(\lambda_\mu) \left( \sqrt{\beta}\,M\lambda_\mu + \sqrt{q_\Gamma}\,Z_\mu \right) >0 \right],
    \label{eq:downstream-orientation-probability}
\end{equation}
where $(\lambda_\mu,Z_\mu)$ are i.i.d. copies of $(\lambda,Z)$.

The expected few-shot classification error is therefore
\begin{equation}
    \epsilon_{\rm ds}^{(n_{\rm ds})} = p_{n_{\rm ds}}\epsilon_{\rm ds}^{\infty} + (1-p_{n_{\rm ds}}) (1-\epsilon_{\rm ds}^{\infty}).
    \label{eq:downstream-finite-error}
\end{equation}

At $M=0$, both orientations give chance-level performance, and we set $\epsilon_{\rm ds}^{\infty}=\epsilon_{\rm ds}^{(n_{\rm ds})}=1/2$. Predictions with zero score are resolved by a fair random choice, including when the frozen feature vanishes identically.

These downstream formulas are conditional on the limiting overlaps $(M,q_\Gamma)$. For the autoencoders, the replica saddle supplies their predicted values; the subsequent calculation evaluates the specified logistic probe.

\paragraph{Gaussian amplitudes.} For $P_\lambda=\N(0,1)$, write $A=\beta M^2/q_\Gamma$ and $k=n_{\rm ds}\geq1$. The population error is
\begin{equation}
    \epsilon_{\rm ds}^{\infty}(A) =\frac1\pi\arccos\sqrt{\frac{A}{1+A}}.
    \label{eq:downstream-Gaussian-population-error}
\end{equation}
Conditioning the orientation statistic in \eqref{eq:downstream-orientation-probability} on the amplitudes gives
\begin{equation}
    p_k(A) =\E_{\lambda_1,\ldots,\lambda_k} \Phi\!\left(\sqrt{\frac Ak}\sum_{\mu=1}^k|\lambda_\mu|\right), \qquad \lambda_\mu\stackrel{\rm iid}{\sim}\N(0,1),
    \label{eq:downstream-Gaussian-orientation}
\end{equation}
where $\Phi$ is the standard normal cumulative distribution function. Together with \eqref{eq:downstream-finite-error}, these expressions give the finite-label accuracy curve. Set $A=0$ when the feature vanishes; the error is then $1/2$.

\paragraph{Interpretation} The downstream scalar channel,
\begin{equation}
    H = \sqrt{\beta}\,M\lambda + \sqrt{q_\Gamma}\,Z,
\end{equation}
is controlled by the downstream signal-to-noise ratio
\begin{equation}
    \frac{\beta M^2}{q_\Gamma} = \frac{ \beta\,\theta_{\bu^\star}^2 }{ q_\Gamma/q_I },
    \label{eq:main-downstream-snr}
    \qquad \theta_{\bu^\star} = \frac{|M|}{\sqrt{q_I}}.
\end{equation}
The ratio
\begin{equation}
    \frac{q_\Gamma}{q_I} = \frac{ \sum_i \gamma_i^\star \widehat w_i^2 }{ \sum_i \widehat w_i^2 }
\end{equation}
can therefore be interpreted as the effective noise level seen by the learned representation. Whereas $\beta$ specifies the signal-to-noise ratio in the data model, \eqref{eq:main-downstream-snr} describes the signal-to-noise ratio after projection onto the learned direction. Hence downstream performance depends on two distinct properties of the pretrained representation: its alignment with the planted feature, measured by $\theta_{\bu^\star}$, and its effective exposure to heteroskedastic noise, measured by $q_\Gamma/q_I$. Changes in the number of masks $K$ or in the activation function can improve downstream performance through either mechanism.

\paragraph{AMP pretraining reference with the same probe.} At the matched state-evolution solution of \Cref{app:bayes-optimal}, the unweighted AMP feature has $|M_{\rm AMP}|=\E_\gamma q_u(\gamma)$ and $q_{\Gamma,{\rm AMP}}=\E_\gamma[\gamma q_u(\gamma)]$ by the Nishimori identities. Its downstream SNR is therefore
\begin{equation}
    A_{\rm AMP} =\frac{\beta\bigl(\E_\gamma q_u(\gamma)\bigr)^2} {\E_\gamma[\gamma q_u(\gamma)]}.
    \label{eq:downstream-AMP-snr}
\end{equation}
For Gaussian priors, solve \eqref{eq:BO-complete-scalar-root}, obtain $q_u(\gamma)$ from \eqref{eq:BO-Gaussian-prior-qu}, and substitute $A_{\rm AMP}$ into \eqref{eq:downstream-Gaussian-population-error}, \eqref{eq:downstream-Gaussian-orientation}, and \eqref{eq:downstream-finite-error}. At the uninformative fixed point set $A_{\rm AMP}=0$. This is the green downstream reference in \Cref{fig:mask-sample-efficiency}: Bayes-optimal pretraining followed by the same finite-label logistic probe as the autoencoders. Its identification with AMP performance requires the corresponding state-evolution solution to be attained. Algorithmically, freeze the AMP estimate as $\widehat{\bw}=\widehat{\bu}$, form the unweighted scores \eqref{eq:downstream-scalar-channel}, and fit \eqref{eq:downstream-logistic-probe}.

\newpage

% =============================================================================
\section{Experimental settings and reporting}
\label{app:experimental-reporting}

We report the settings needed to interpret the comparisons below. Source code and per-run configurations are provided in the accompanying code and data supplement. In all figures, error bars and shaded bands indicate $\pm1$ sample standard deviation (SD) across seeds around the empirical mean for each condition. Single-run results are shown without error bars.

\subsection{Masking and augmentation protocols}
\label{app:experimental-protocols}

\paragraph{Mask diversity and training budget.} For finite $K$, each example has $K$ independently drawn masks, fixed throughout pretraining. Static masking is $K=1$; dynamic masking draws fresh masks on repeated presentations. The number of underlying examples is unchanged. In the shallow experiments, each update averages over the fixed mask collection, or uses one fresh mask for dynamic training (\Cref{app:experimental-metadata}). The CNN training schedule is specified below.

\paragraph{CNN data and evaluation.} \Cref{fig:static-vs-dynamic}(b) and \Cref{fig:K-interpolation}(b) use CIFAR-10 \citep{krizhevsky2009learning} with five independent encoder-initialization and data-split seeds. Each seed uses a balanced $40{,}000$-image training subset for both pretraining and labeled frozen probing. Masked pretraining and linear probing use $10{,}000$ held-out images for validation. The separate official $10{,}000$-image test set supplies final accuracy. One probe is fitted per encoder; points report mean test accuracy and bars show $\pm1$ SD across the five seeds. The masked comparisons share the split and initialization within each seed.

\paragraph{CNN architecture and preprocessing.} All CNN conditions use the same $854{,}563$-parameter autoencoder. Images are scaled to $[0,1]$ and standardized using channel statistics from the pretraining subset. The encoder receives the three masked image channels and a binary mask channel, with hidden pixels set to zero after standardization. Its four convolutions have output widths $(32,64,128,128)$, kernel sizes $(3,4,4,4)$, strides $(1,2,2,2)$, padding one, and ReLU activations. The decoder uses three stride-two, kernel-four transposed convolutions with widths $(128,64,32)$, padding one and ReLU, followed by a kernel-three, padding-one convolution to RGB. There are no skip connections or normalization layers. Global spatial averaging of the encoder's $128\times4\times4$ output gives the frozen representation. Images contain $64$ nonoverlapping $4\times4$ patches; each mask hides $\operatorname{round}(64\rho)$ patches chosen uniformly without replacement.

\paragraph{CNN optimization.} For the main CNN sweeps in \Cref{fig:static-vs-dynamic,fig:K-interpolation,fig:app-cnn-augmentation,fig:app-cnn-training-original}, we use AdamW \citep{loshchilov2019decoupled} with batch size $512$, learning rate $10^{-3}$, weight decay $10^{-4}$, and gradient norm clipping at $5$. A validation-based plateau scheduler halves the learning rate, with patience five epochs and a minimum of $10^{-5}$.

For finite $K$, each epoch traverses the fixed collection of $40{,}000K$ image--mask pairs. Dynamic masking draws a fresh mask at every presentation throughout training. We group $32$ presentations per image into one epoch for validation and scheduling, corresponding to $2{,}500$ optimizer updates. Under dynamic masking, each image therefore receives $32E$ independent mask draws over $E$ epochs.

Validation uses one fixed mask per held-out image after each epoch. We retain the checkpoint with the lowest validation reconstruction MSE and stop after $15$ epochs without improvement, with a minimum of $20$ and a ceiling of $460$ epochs. Under these parameters, our setting is always early-stopped and not undertrained.

\paragraph{CNN linear probes.} Probes receive unmasked, unaugmented images and use the frozen $128$-dimensional representation, standardized coordinatewise using the probe training set. A linear ten-class head is fitted with cross-entropy, AdamW, batch size $512$, constant learning rate $10^{-2}$, and zero weight decay. Classifier initialization and minibatch-shuffling seeds are fixed across encoders. The head with lowest validation cross-entropy is retained; training stops after $25$ epochs without improvement, with a minimum of $20$ and a ceiling of $200$ epochs. All reported masked-model probes satisfy this stopping criterion. The unmasked checkpoint sweep uses the same probe settings with a $300$-epoch ceiling (\Cref{app:unmasked-training}).

\paragraph{ViT architecture and data.} Our implementation adapts the U-MAE code of \citet{zhang2022mask}, which builds on MAE \citep{he2022masked}.\footnote{U-MAE implementation: \url{https://github.com/zhangq327/U-MAE}.} We use ViT-Base/16 on ImageNet-100, with $126{,}689$ training and $5{,}000$ validation images across $100$ classes. Images have resolution $224\times224$, yielding $196$ patches. The encoder has $12$ blocks, width $768$ and $12$ attention heads; the decoder has $8$ blocks, width $512$ and $16$ heads, with MLP expansion ratio $4$ throughout. Pretraining minimizes MSE on masked patches, using targets standardized within each patch. U-MAE's uniformity penalty is disabled in all reported ViT runs; AdamW weight decay is retained. The inherited online classifier receives detached features and does not train the encoder. Reported accuracies use a separately fitted frozen probe.

\paragraph{ViT preprocessing and masking.} Without stochastic cropping or flipping, images are resized bicubically to shorter side $256$ and center-cropped to $224\times224$. The augmented condition instead uses random resized crops with area fraction in $[0.2,1]$ and aspect ratio in $[3/4,4/3]$, followed by horizontal flipping with probability $1/2$. Both conditions use ImageNet channel normalization. A random patch permutation retains $\lfloor196(1-\rho)\rfloor$ visible patches. Static masking fixes this permutation per training image; dynamic masking redraws it at every presentation. With cropping and flipping, the static patch positions refer to the transformed image, whose content can change across epochs.

\paragraph{ViT pretraining.} All masked runs in \Cref{fig:static-vs-dynamic}(c) and \Cref{fig:masking-augmentation} use $100$ epochs of AdamW with momentum parameters $(0.9,0.95)$, weight decay $0.05$, peak learning rate $7.5\times10^{-5}$, ten warmup epochs and cosine decay to zero. The effective batch size is $128$: batches of $64$ with two accumulation steps without cropping/flipping, and batches of $128$ without accumulation in the augmented condition. Biases and normalization parameters are excluded from weight decay. We probe the final pretraining checkpoint.

The comparison uses one common seed, with identical training and validation splits, initial autoencoder weights, optimizer settings, and pretraining budget across conditions. A separate frozen probe is fitted to each encoder using the same probe initialization and training and evaluation protocol. In contrast, the unaugmented curves in \Cref{fig:static-vs-dynamic}(c) summarize three encoder initializations (seeds $0,1,2$), with means and $\pm1$ SD.

\paragraph{ViT linear probes.} The frozen encoder processes all patches. Its $768$-dimensional class-token representation feeds a non-affine batch-normalization layer and a linear $100$-class head. The head is trained with cross-entropy and LARS \citep{you2017large}, batch size $256$, momentum $0.9$, zero weight decay, peak learning rate $0.1$, ten warmup epochs and cosine decay to zero over $50$ epochs. Probe training uses random resized crops with area fraction in $[0.08,1]$, aspect ratio in $[3/4,4/3]$, and horizontal flipping with probability $1/2$ for every pretraining condition. Validation uses the deterministic resize and center crop above. Thus the no-cropping/no-flipping condition refers to pretraining. Probe seeds are fixed at zero, and we report validation top-5 accuracy at the final probe epoch. Unmasked checkpoint selection is described in \Cref{app:unmasked-training}.

\Cref{tab:vit-probe-augmentation} ablates the use of stochastic augmentation during linear probing. We find that including or omitting stochastic augmentation yields quantitatively similar downstream performance.

\paragraph{Longer ViT training.} The $100$-epoch protocol above is used to compare masking conditions under a common training budget. To assess the attainable absolute performance with the same architecture, we also adapt the longer training schedules of U-MAE and MAE \citep{zhang2022mask,he2022masked}, retaining dynamic masking at $\rho=0.75$, no stochastic cropping or flipping during pretraining, and no uniformity penalty. \Cref{tab:vit-training-budget} reports these checks, all using initialization zero and ImageNet-100 pretraining only. The $200$-epoch schedule uses effective batch $1024$ ($128\times8$ accumulation), $40$ warmup epochs and peak learning rate $6\times10^{-4}$. The longer MAE schedule uses effective batch $4096$ ($64\times64$ accumulation), $40$ warmup epochs and peak learning rate $2.4\times10^{-3}$; it was stopped early, and the evaluated checkpoint is after $641$ epochs of the planned $1600$-epoch cosine schedule. Other pretraining settings remain as above.

Each encoder is evaluated with the existing probe and with a $90$-epoch probe using effective batch $16{,}384$ ($512\times32$ accumulation), peak learning rate $6.4$ and ten warmup epochs, following the MAE linear-probing recipe. Both probes use cropping and flipping. The larger batch is implemented by gradient accumulation; batch normalization uses the physical batch, and incomplete accumulation windows at the end of an epoch do not produce an optimizer update. With the $200$-epoch encoder, the longer probe reaches $61.50\%$ top-1 and $86.04\%$ top-5 accuracy without pretraining augmentation. For context, \citet{zhang2022mask} report $61.2\%$ and $67.5\%$ top-1 for MAE and regularized U-MAE, respectively, on ImageNet-100 with ViT-Base. These published results provide a reference under different training conditions. The additional runs show that the shorter protocol used for our masking comparisons does not exhaust the performance of the architecture; they vary the optimization schedule as well as batch size. See \cite{baevski2023efficient} for experiments showing the effect of the batch size on downstream performance.

\begin{table}[!htbp]
\centering
\small
\setlength{\tabcolsep}{5pt}
\begin{tabular}{lrrrrrr}
\toprule
\multicolumn{3}{c}{Pretraining} & \multicolumn{2}{c}{Linear probe} &
\multicolumn{2}{c}{Accuracy (\%, $\uparrow$)} \\
\cmidrule(lr){1-3}\cmidrule(lr){4-5}\cmidrule(lr){6-7}
Schedule & Epochs & Batch & Epochs & Batch & Top-1 & Top-5 \\
\midrule
Main comparison & 100 & 128 & 50 & 256 & 39.96 & 68.78 \\
U-MAE & 200 & 1024 & 50 & 256 & 47.56 & 76.14 \\
U-MAE & 200 & 1024 & 90 & 16384 & 61.50 & 86.04 \\
MAE & 641 & 4096 & 50 & 256 & 51.38 & 79.60 \\
MAE & 641 & 4096 & 90 & 16384 & 59.12 & 84.76 \\
\bottomrule
\end{tabular}
\caption{\textbf{Higher accuracy with longer ViT training schedules.} Frozen-probe validation accuracy for dynamic masking at $\rho=0.75$, without cropping or flipping during pretraining. Batch sizes are effective batches. All rows use the same initialization; repeated pretraining rows probe the same encoder checkpoint. Schedule names indicate adaptations of the published optimization recipes, with uniformity regularization disabled. Results use one run per configuration and the final probe epoch.}
\label{tab:vit-training-budget}
\end{table}

\paragraph{Removing augmentation from the ViT probe.} We also repeat the $50$-epoch probes of the static and dynamic $100$-epoch encoders at $\rho=0.75$, using initialization zero and deterministic resize/center crop during probe training. All other probe settings and the pretrained checkpoints are held fixed. \Cref{tab:vit-probe-augmentation} shows that removing cropping and flipping raises top-5 accuracy by $2.28$ and $2.08$ percentage points for dynamic and static masking, respectively. Their gap remains nearly unchanged ($17.58$ versus $17.78$ points), indicating that probe augmentation does not account for the dynamic-masking advantage in this check.

\begin{table}[!htbp]
\centering
\small
\begin{tabular}{lrrrr}
\toprule
& \multicolumn{2}{c}{Probe with crop/flip} &
\multicolumn{2}{c}{Probe without crop/flip} \\
\cmidrule(lr){2-3}\cmidrule(lr){4-5}
Pretraining masks & Top-1 & Top-5 & Top-1 & Top-5 \\
\midrule
Static & 25.34 & 51.20 & 26.22 & 53.28 \\
Dynamic & 39.96 & 68.78 & 41.98 & 71.06 \\
\bottomrule
\end{tabular}
\caption{\textbf{The masking advantage persists without probe augmentation.} ViT-Base/ImageNet-100 validation accuracy (\%, $\uparrow$), using the same pretrained encoders and probe hyperparameters in each pair. Both encoders were pretrained without cropping or flipping at $\rho=0.75$. Results use one initialization and one probe seed.}
\label{tab:vit-probe-augmentation}
\end{table}

\subsection{Finite-size effects of PCA}
\label{app:pca-numerics}

\paragraph{Finite-dimensional cosine similarity.} \Cref{fig:app-pca-scaling} uses the same bounded-noise setting as \Cref{fig:validation}(b), with $\gamma=\frac14+\frac94 B$, $B\sim\operatorname{Beta}(1,2)$, and $\beta=1$, while varying the dimension $d$ and sample complexity $\alpha$. It examines the residual finite-dimensional PCA alignment visible in that main-text comparison. In \Cref{fig:app-pca-scaling}, the observed cosine decreases as the dimension increases, including at the largest values of $\alpha$ shown, consistent with the high-dimensional scaling $A_\alpha d^{-1/2}$ and supporting its interpretation as a finite-size effect. We use uncentered, unstandardized PCA at $\beta=1$, dimensions $d=100\times2^j$ for $j=0,\ldots,8$, ten logarithmically spaced values of $\alpha$ from $1$ to $50$, and $30$ seeds. The plotted observable is the mean absolute cosine, with $\pm1$ SD.

\begin{figure}[!htbp]
\centering
\includegraphics[width=\textwidth]{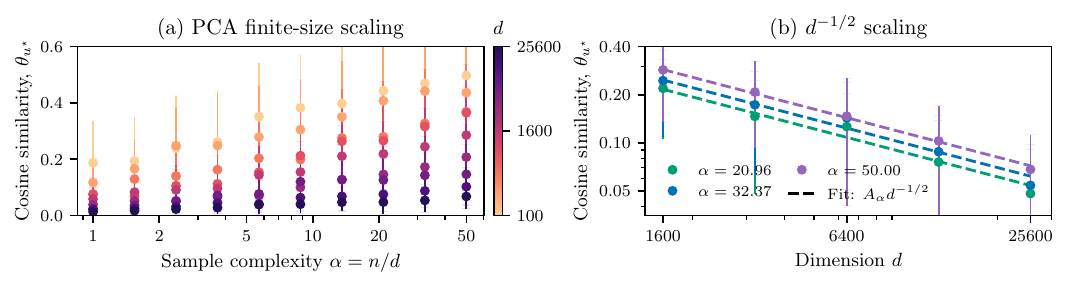}
\caption{\textbf{Finite-dimensional PCA cosine similarity decreases with dimension.} Same bounded-noise setting as \Cref{fig:validation}(b): $\gamma=\frac14+\frac94 B$, $B\sim\operatorname{Beta}(1,2)$, and $\beta=1$. \textbf{(a)} Cosine similarity versus sample complexity, with color indicating dimension size $d$. \textbf{(b)} Dimension dependence at fixed $\alpha$. Dashed lines fit $A_\alpha d^{-1/2}$ for $d\geq1600$, with exponent fixed. Bars show $\pm1$ SD over $30$ seeds.}
\label{fig:app-pca-scaling}
\end{figure}

\subsection{Top-1 accuracy for ViT}
\label{app:vit-top1}

\Cref{fig:app-vit-top1} reports the top-1 counterparts of the ViT comparisons in \Cref{fig:static-vs-dynamic}(c) and \Cref{fig:masking-augmentation}, using the same pretrained encoders and frozen linear probes. The qualitative conclusions also hold for this metric: dynamic masking improves accuracy without cropping or flipping during pretraining, while these augmentations substantially improve static masking and narrow the gap to dynamic masking.

\Cref{tab:vit-training-budget} puts these absolute accuracies in context. At $\rho=0.75$, dynamic masking reaches $39.96\%$ top-1 under the main comparison protocol, increasing to $61.50\%$ and $59.12\%$ with the longer U-MAE and MAE schedules and $90$-epoch probes, respectively. All three settings use no cropping or flipping during pretraining. Thus, the shorter protocol used to compare masking conditions does not exhaust the architecture's performance; these additional runs vary both the training schedule and batch size.

\Cref{tab:vit-probe-augmentation} separately checks standard augmentation during linear probing. Removing probe cropping and flipping raises top-1 accuracy from $39.96\%$ to $41.98\%$ for dynamic masking and from $25.34\%$ to $26.22\%$ for static masking. The dynamic--static gap therefore persists ($14.62$ versus $15.76$ percentage points), supporting the same conclusion with either probe protocol. Both tables report single-seed checks, distinct from the three-seed averages in \Cref{fig:app-vit-top1}(a).

\begin{figure}[!htbp]
\centering
\includegraphics[width=\textwidth]{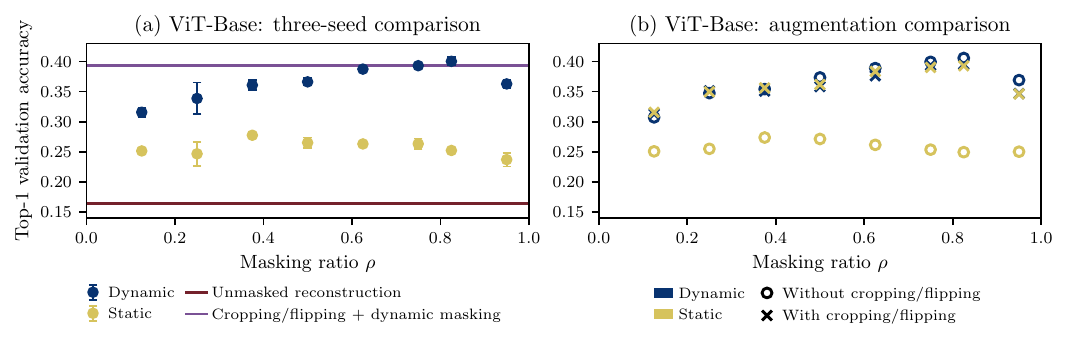}
\caption{\textbf{The ViT masking comparison also holds for top-1 accuracy.} Frozen linear-probe validation accuracy on ImageNet-100. \textbf{(a)} Static and dynamic masking without cropping or flipping: mean $\pm1$ SD over three seeds. Each horizontal reference uses one seed. \textbf{(b)} Static and dynamic masking with and without cropping and flipping, using the common seed and protocol of \Cref{fig:masking-augmentation}; no error bars are shown. Training and evaluation settings are given in \Cref{app:experimental-protocols}.}
\label{fig:app-vit-top1}
\end{figure}

\subsection{Noise-aware spectral methods}
\label{app:heteroPCAs}

\Cref{fig:app-heterogeneous-pca} adds diagonal-deletion PCA, HeteroPCA, and estimated-whitening PCA to the bounded and unbounded settings of \Cref{fig:validation}(b,c), with $\beta=1$ and $\rho=0.75$ for dynamic masking. Spectral estimators and autoencoders use $d=2000$ and $20$ seeds; AMP uses $d=4000$ and ten seeds (\Cref{app:experimental-bayes}). Bars show $\pm1$ SD. The added spectral methods use the same observations and seeds as ordinary PCA and receive neither the spike nor the true coordinate variances. Covariances are centered in the unbounded setting and uncentered in the bounded setting.

\paragraph{Spectral estimators.} Write $S=\widehat\Sigma$ the empirical covariance matrix, $\operatorname{diag}(S)$ for its diagonal matrix, and $\operatorname{off}(S)=S-\operatorname{diag}(S)$. In this section we will work with the following additional spectral estimators:

\emph{Diagonal-deletion PCA} removes the empirical diagonal before spectral extraction, using $\operatorname{off}(S)$. This baseline is studied alongside HeteroPCA by \citet{zhang2022heteroskedastic} but was originally introduced by \citet{florescu2016spectral}. We extract its largest-magnitude eigenpair, corresponding to a leading singular direction. This removes the direct diagonal noise contribution, but also discards the signal diagonal.

\begin{figure}[!htbp]
\centering
\includegraphics[width=\textwidth]{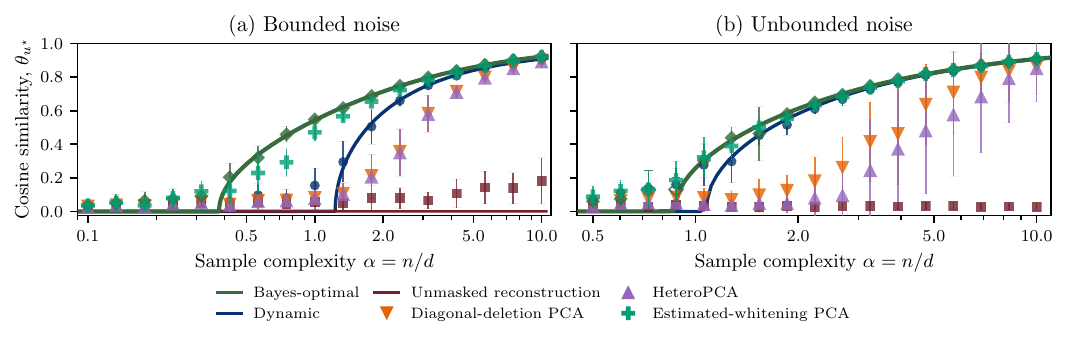}
\caption{\textbf{Noise-aware spectral baselines for heteroskedastic noise.} Cosine similarity versus sample complexity under \textbf{(a)} bounded and \textbf{(b)} unbounded heterogeneous noise. The comparison includes dynamic masking, unmasked reconstruction (ordinary PCA), diagonal-deletion PCA \cite{florescu2016spectral}, HeteroPCA \citep{zhang2022heteroskedastic}, estimated-whitening PCA \citep{leeb2021optimal}, and the \bayesopt{Bayes-optimal} baseline. Lines show theoretical references; markers show numerical estimates. Green curves give the Bayes-optimal state-evolution reference; green diamonds show AMP results (\Cref{subsec:BO-AMP,app:experimental-bayes}). Bars show $\pm1$ SD; hollow AMP markers identify groups containing iteration-limited runs.}
\label{fig:app-heterogeneous-pca}
\end{figure}

\emph{HeteroPCA} iteratively imputes the diagonal from a rank-one approximation while retaining the observed off-diagonal entries \citep[Algorithm~1]{zhang2022heteroskedastic}. We initialize with $\operatorname{off}(S)$, retain the signed largest-magnitude eigenpair, and stop when the relative diagonal change is at most $10^{-8}$. Unresolved runs are continued up to $10^5$ iterations; all $680$ fits in the displayed comparison meet this criterion.

\emph{Estimated-whitening PCA} follows the noise-whitening approach of \citet{leeb2021optimal}, who also discuss estimated noise covariances in Section~4.3. We estimate $\widehat\Gamma_{ii}=\max(S_{ii},10^{-12})$ and apply PCA to $S_{\rm W}=\widehat\Gamma^{-1/2}S\widehat\Gamma^{-1/2}$. If $\bv_{\rm W}$ is its leading eigenvector, alignment is evaluated using $\widehat\Gamma^{1/2}\bv_{\rm W}$ in the original coordinates. This baseline uses whitening and unwhitening without additional spectral shrinkage.

Dynamic masking achieves higher mean alignment than diagonal-deletion PCA and HeteroPCA across both displayed sweeps, without explicit noise estimation or whitening. Neither the coordinate variances nor their distribution are supplied to the autoencoder: adaptation to heterogeneity emerges through the masked prediction objective (\Cref{app:mask-noise-adaptation}). This provides a complementary perspective to covariance-aware inference. For example, \citet{ZhangMondelli2024} construct spectral estimators using known noise covariances and signal strength, with optimality guarantees under their assumptions. Estimated-whitening PCA generally achieves higher alignment in our experiments and lies closer to the Bayes reference, using variances estimated from the same observations. The comparison therefore highlights the ability of masked prediction to learn informative structure without a separate procedure for estimating and correcting the noise.

\paragraph{The whitening-PCA threshold.} The whitening threshold follows from the spiked-model analysis of \citet{leeb2021optimal}. In our rank-one model with Gaussian latent priors, whitening by $(\Gamma^\star)^{-1/2}$ produces isotropic noise and a population signal eigenvalue $\ell^{\rm W}=\beta\|(\Gamma^\star)^{-1/2}\bu^\star\|^2/d \to\beta\E_\gamma[\gamma^{-1}]$. Their recovery condition is $\ell^{\rm W}>\alpha^{-1/2}$, with the same onset after unwhitening, giving $\alpha_c^{\rm W}=[\beta^2\E_\gamma[\gamma^{-1}]^2]^{-1}$.

\citet{leeb2021optimal} also justify replacing the noise covariance by an estimate consistent in operator norm. For diagonal noise and a delocalized signal direction, they obtain such an estimate from the sample coordinate variances, as in the baseline above. Our bounded variance law and normalized Gaussian spike satisfy these conditions, so the same threshold applies to estimated whitening in that setting. For the shifted-lognormal law, $\alpha_c^{\rm W}$ is the ideal-whitening reference; their bounded-spectrum assumptions do not directly cover the estimated-whitening algorithm, whose performance is shown empirically.

\paragraph{Relation to the masking and Bayes thresholds.} Our replica analysis connects this established whitening threshold to dynamic masking. Combining the Bayes expression \eqref{eq:BO-final-threshold} with the lower bound in \eqref{eq:fixed-beta-mask-lower-bound}, specialized to dynamic linear masking in \eqref{eq:dyn-bound-sandwich}, gives

\begin{equation}
    \alpha_c^{\rm BO} =\frac{1}{\beta^2\E_\gamma[\gamma^{-2}]} \leq \alpha_c^{\rm W} =\frac{1}{\beta^2\E_\gamma[\gamma^{-1}]^2} \leq \alpha_{c,\mathrm{RS}}^{\rm dyn}.
    \label{eq:experimental-whitening-threshold-order}
\end{equation}

The inverse moments are finite for both displayed laws. The first inequality is strict for a nonconstant variance law by Cauchy--Schwarz. It orders the recovery onsets; the relative alignments above threshold are assessed by the simulations. Thus masking can outperform diagonal correction while remaining less effective than whitening at the onset of recovery.

\subsection{Dynamic masking and nonlinear autoencoders}
\label{app:dynamic-full-comparison}

\begin{figure}[!htbp]
\centering
\includegraphics[width=\textwidth]{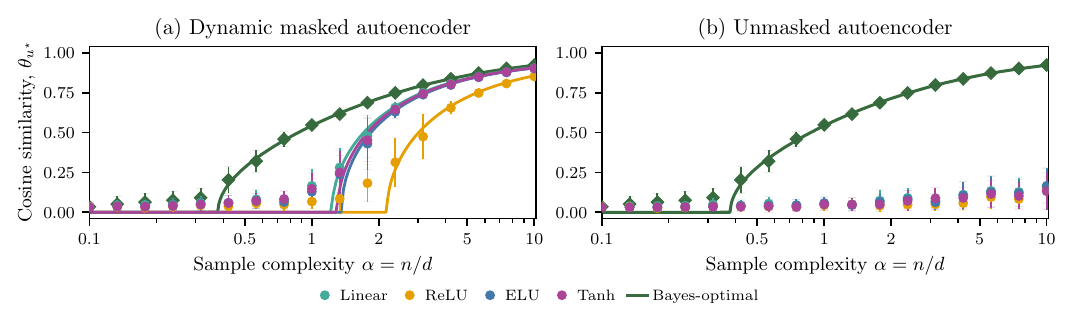}
\caption{\textbf{Masking improves recovery across activation functions.} Cosine similarity for \textbf{(a)} dynamic masking and \textbf{(b)} no masking in the same setting as \Cref{fig:mask-sample-efficiency}. Colored curves are replica predictions, colored markers in both panels are ERM simulations at finite $d=2000$, averaged across $30$ seeds $\pm 1$ SD. In both panels, the green curve shows the Bayes-optimal prediction and green diamonds show AMP results (\Cref{subsec:BO-AMP,app:experimental-bayes}).}
\label{fig:app-dynamic-full}
\end{figure}

\Cref{fig:app-dynamic-full} compares dynamic masking with the no-masking reconstruction, \Cref{eq:main-full-objective}, for linear, ReLU, ELU, and tanh activations. We use the same setting as \Cref{fig:mask-sample-efficiency}.

Across the four activations, dynamic masking develops substantially higher alignment, whereas full reconstruction retains low mean alignment throughout the displayed range. Thus, changing the activation of the unmasked autoencoder does not help to recover the spike. On the other hand, for dynamic masking, linear, ELU, and tanh dynamic curves are close, while ReLU requires larger sample complexity to attain comparable alignment.

\subsection{Reconstruction errors across mask diversity}
\label{app:reconstruction-errors}

\Cref{fig:app-reconstruction-errors} evaluates the centered training and test errors defined in \Cref{subsec:observables}, using linear activation, $\alpha=8$, $\beta=1$, $\laR=10^{-4}$, and the bounded variance law of \Cref{fig:validation}(b). Again, $d=2000$, $30$ seeds. Each loss subtracts the zero-decoder loss on the same samples and masks, excludes the ridge penalty, and is reported per sample without division by $\rho$. Test errors use independent samples and fresh masks; dynamic training errors average fresh masks on the training samples.

\begin{figure}[!htbp]
\centering
\includegraphics[width=\textwidth]{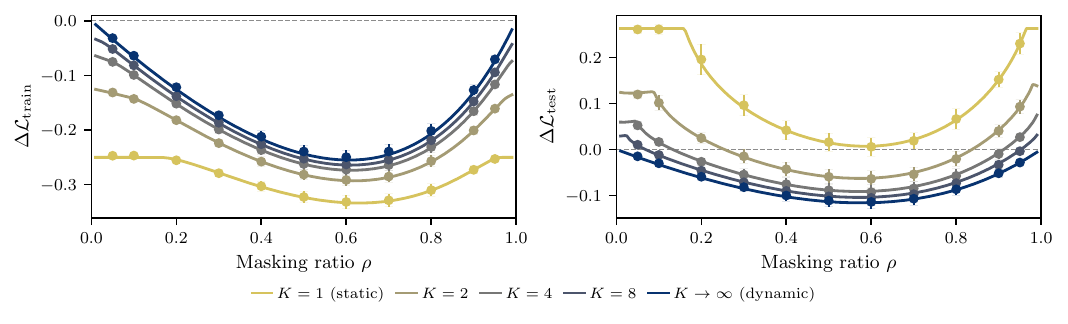}
\caption{\textbf{Mask resampling worsens training error while improving test loss.} Centered reconstruction errors versus masking ratio at $\alpha=8$ under bounded heterogeneous noise. Left: training error. Right: independent test error. Curves are the replica predictions in \eqref{eq:train-error-RS} and \eqref{eq:test-error-RS}; markers and bars show mean $\pm1$ SD over $30$ simulations. Lower values are better; zero denotes the zero-decoder baseline. Colors distinguish the fixed mask collections and dynamic masking.}
\label{fig:app-reconstruction-errors}
\end{figure}

\subsection{Number of downstream labels}
\label{app:downstream-label-count}

\Cref{fig:app-downstream-label-count} varies the number of independently labeled examples used to fit the frozen scalar logistic probe. We reuse the linear dynamic-masking replica solution from \Cref{fig:mask-sample-efficiency}, with bounded noise, $\beta=1$, $\rho=0.75$, and $\laR=10^{-4}$, and evaluate \eqref{eq:downstream-finite-error} for $n_{\rm ds}\in[1,128]$. The pretraining overlaps remain fixed for each $\alpha$: additional labels only change the probability that the probe selects the correct orientation of the representation. We evaluate \eqref{eq:downstream-Gaussian-orientation} with $10^5$ Monte Carlo amplitude samples, except at $n_{\rm ds}=1$, where $p_1(A)=\frac12+\pi^{-1}\arctan\sqrt A$ is exact.

Accuracy improves with the label count and approaches the population-orientation limit rapidly away from the recovery transition. Increasing $n_{\rm ds}$ therefore improves the finite-label readout without changing the pretraining recovery threshold.

\begin{figure}[!htbp]
\centering
\includegraphics[width=0.5\textwidth]{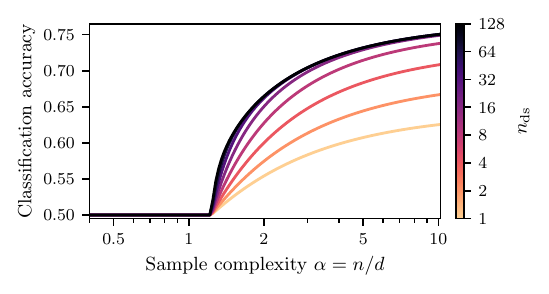}
\caption{\textbf{Additional labels resolve the orientation of a frozen feature.} Predicted downstream classification accuracy for linear dynamic masking versus sample complexity. The colorbar gives $n_{\rm ds}$ on a logarithmic scale. All curves use the same pretraining solution at each $\alpha$ and differ only in the size of the independent labeled dataset used by the scalar logistic probe.}
\label{fig:app-downstream-label-count}
\end{figure}
\FloatBarrier

\subsection{Early-stopping without masking}
\label{app:unmasked-training}

\Cref{fig:app-no-masking-cnn,fig:app-no-masking-vit} evaluate frozen linear probes at intermediate unmasked-pretraining checkpoints for the architectures and datasets of \Cref{fig:static-vs-dynamic}. In both architectures, training and validation reconstruction MSE continue to improve after downstream accuracy has peaked. The deterioration in the representation is therefore not accompanied by a rise in held-out reconstruction error.

\begin{figure}[ht]
\centering
\includegraphics[width=\textwidth]{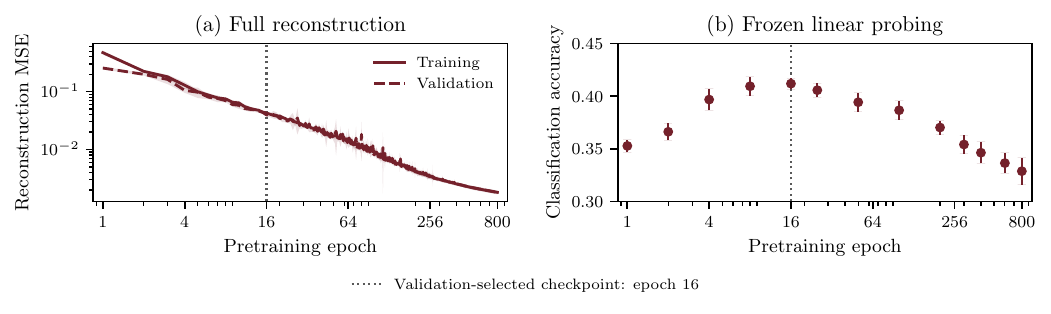}
\caption{\textbf{Better reconstruction need not improve CNN representations.} \textbf{(a)} Training and validation reconstruction MSE during unmasked CNN pretraining. \textbf{(b)} Linear-probe test accuracy at different pretraining epochs. Shading and bars show $\pm1$ SD over five seeds. The dotted line marks epoch $16$, selected by mean probe validation accuracy for \Cref{fig:static-vs-dynamic}.}
\label{fig:app-no-masking-cnn}
\end{figure}

\begin{figure}[ht]
\centering
\includegraphics[width=\textwidth]{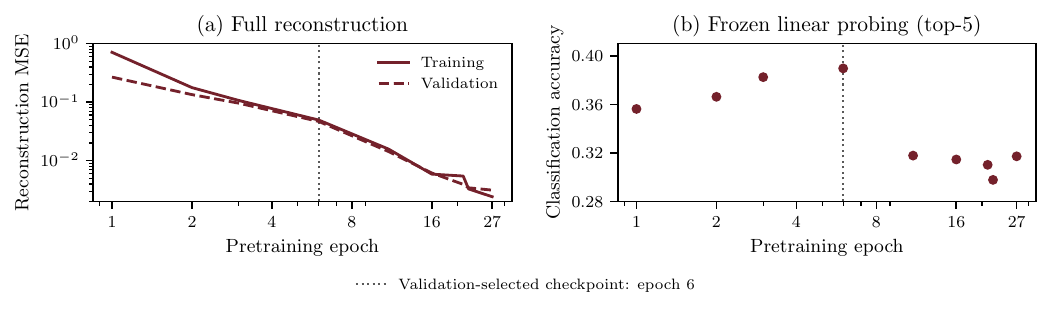}
\caption{\textbf{Reconstruction and downstream performance also diverge for ViT.} \textbf{(a)} Training and validation reconstruction MSE during unmasked ViT pretraining. \textbf{(b)} Linear-probe top-5 accuracy on the validation set at different pretraining epochs. The dotted line marks six completed pretraining epochs, selected by probe validation accuracy for \Cref{fig:static-vs-dynamic}. Results use one pretraining seed.}
\label{fig:app-no-masking-vit}
\end{figure}

For CNN, epoch $16$ maximizes mean probe validation accuracy across seeds and supplies the unmasked baseline in \Cref{fig:static-vs-dynamic,fig:K-interpolation}; the plotted curve reports test accuracy. Reconstruction losses are measured in the standardized image coordinates.

For ViT, unmasked pretraining uses the architecture, deterministic preprocessing and optimizer settings of \Cref{app:experimental-protocols}, retaining all patches and evaluating reconstruction MSE over all patches. Targets remain standardized within each patch, and the learning-rate schedule retains its $100$-epoch horizon. The final probe's validation top-5 accuracy selects six pretraining epochs ($38.96\%$). \Cref{fig:app-no-masking-vit}(a) shows normalized-patch reconstruction MSE, excluding the detached online classifier's loss; validation MSE is evaluated on all $5{,}000$ validation images. Neither the CNN nor the ViT checkpoint is selected by minimum reconstruction validation loss since it is misleading for downstream performance.

\textbf{Masked} CNNs use validation-selected reconstruction checkpoints, whereas the main \textbf{masked} ViT comparisons use the final pretraining checkpoint. This selection favors the unmasked reference, whose reported score is a selected validation result rather than an independent test estimate.

\subsection{Mask renewal and image augmentation in CNNs}
\label{app:cnn-augmentation}

\Cref{fig:app-cnn-augmentation} complements \Cref{fig:masking-augmentation} using the CNN autoencoder on CIFAR-10 from \Cref{fig:static-vs-dynamic}(b). We compare static and dynamic masking with and without stochastic cropping and flipping, averaging over five seeds. Within each seed, the four conditions share the data split and initial autoencoder weights. The architecture, optimizer settings, validation-based checkpoint selection and stopping rule, and frozen linear-probing protocol follow \Cref{app:experimental-protocols}.

\begin{figure}[ht]
\centering
\includegraphics[width=0.50\textwidth]{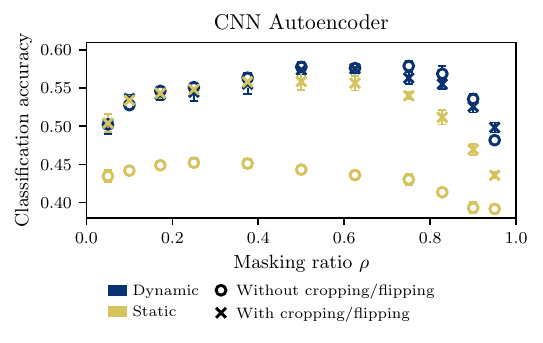}
\caption{\textbf{Mask renewal and image augmentation in CNNs.} Frozen linear-probe test top-1 accuracy on CIFAR-10 versus masking ratio $\rho$. Blue/yellow indicate \dynamicmask{dynamic}/\staticmask{static} masking; circles/crosses denote pretraining without/with stochastic cropping and flipping. Each marker reports the mean over five seeds, with data splits and autoencoder initializations shared across conditions within each seed. Linear probes use no image augmentation.}
\label{fig:app-cnn-augmentation}
\end{figure}

Cropping and flipping substantially improve static masking across the sweep, bringing its mean accuracy close to dynamic masking at lower masking ratios. At larger masking ratios, dynamic masking retains an advantage even with augmentation. The CNN comparison therefore supports the same interpretation as the ViT experiment: geometric transformations can partly compensate for mask reuse by changing the content to be predicted. The remaining gap shows that the extent of this compensation depends on the masking ratio and experimental setting.

\subsection{CNN training budgets and validation selection}
\label{app:cnn-training-diagnostics}
We examine the CNN mask-diversity comparison of \Cref{fig:K-interpolation}(b) at $\rho=0.625$, with $K\in\{1,2,4,8,16,32\}$ and dynamic masking. Both experiments use the CIFAR-10 architecture, AdamW settings, and the frozen linear-probing protocol of \Cref{app:experimental-protocols}, without cropping or flipping. Each condition has five seeds (different pretraining/validation splits). Encoders are selected by validation reconstruction MSE followed by downstream accuracy measured on a separate test set.

\Cref{fig:app-cnn-training-original} shows the original training and validation trajectories with the \textit{plateau learning-rate} (LR) scheduler. Finite-mask runs continue reducing training error after validation error levels off or rises (\textit{overfitting}), motivating validation-based checkpoint selection. Since the original scheduler operates in epochs whose update counts depend on $K$, we also test whether the results persist under a common schedule measured in optimizer updates.

\begin{figure}[!htbp]
\centering
\includegraphics[width=\textwidth]{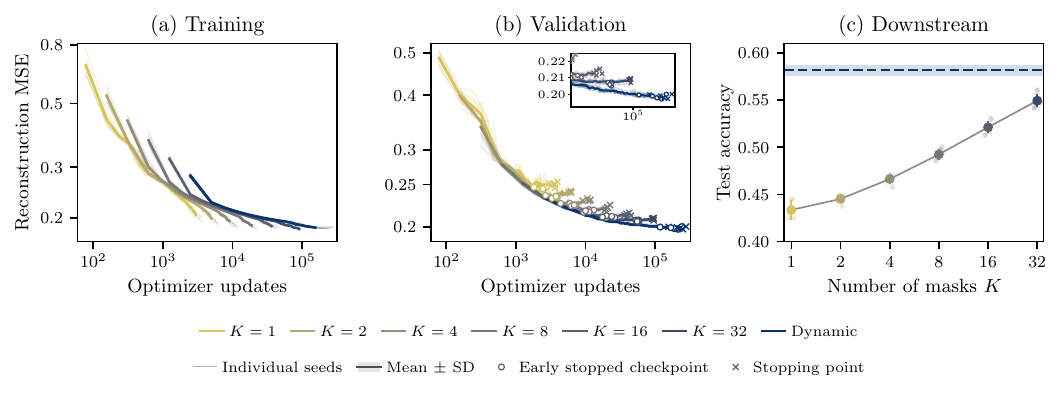}
\caption{\textbf{CNN training under the original protocol.} \textbf{(a,b)} Training and validation reconstruction MSE versus optimizer updates. Thin traces show individual seeds; thick curves and shading show mean $\pm1$ SD while all five runs are observed. Circles and crosses mark early stopped and final checkpoints, respectively. \textbf{(c)} Frozen linear-probe test accuracy of early-stopped encoders, with $\pm1$ SD. The horizontal dashed line indicates dynamic masking.}
\label{fig:app-cnn-training-original}
\end{figure}

To assess the role of the scheduler, we repeat the experiments with the same data splits, random seeds, and initial weights, using a fixed learning rate (LR) of $10^{-3}$ and a longer training number of optimizer updates (\Cref{fig:app-cnn-training-fixedlr}). All conditions now use $2{,}500$ updates between validation evaluations, a stopping patience of $125{,}000$ updates without improvement, and a total maximum of $1{,}150{,}000$ updates.

\Cref{fig:app-cnn-protocol-overlay}(a) compares validation-selected (early-stopped checkpoints) encoders under the two protocols. Their mean test accuracies remain close across masking conditions, and the downstream improvement with mask diversity persists under the fixed-learning-rate control. The extended runs also exhibit a reconstruction-generalization gap.

\Cref{fig:app-cnn-protocol-overlay}(b) instead selects the encoder checkpoint saved at step $125{,}000$ from each fixed learning rate run (\Cref{fig:app-cnn-training-fixedlr}) and evaluates it using the same frozen linear-probing protocol. This changes the checkpoint used for downstream evaluation, not the training or stopping procedure. Thus, the downstream mask-diversity trend persists both when selecting the encoder by validation reconstruction MSE and when selecting its checkpoint at a common training step.

\begin{figure}[!htbp]
\centering
\includegraphics[width=\textwidth]{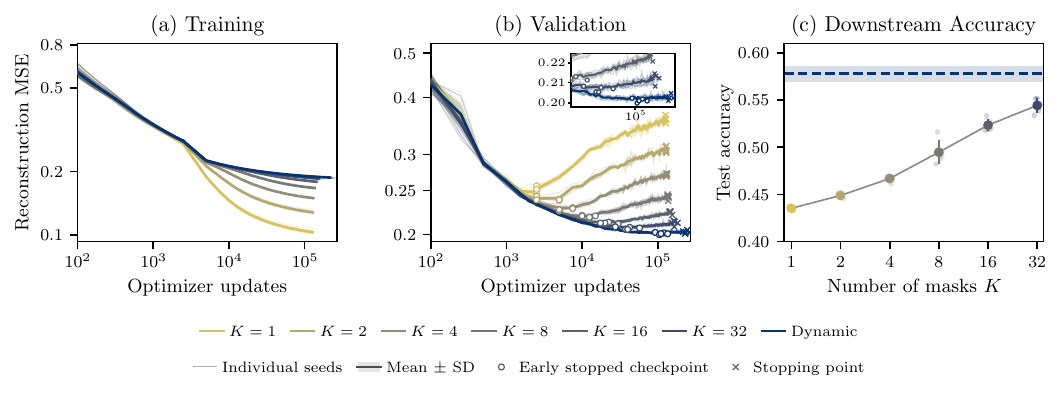}
\caption{\textbf{CNN control with a fixed learning rate and extended training.} Panels and uncertainty conventions follow \Cref{fig:app-cnn-training-original}.}
\label{fig:app-cnn-training-fixedlr}
\end{figure}

\begin{figure}[!htbp]
\centering
\includegraphics[width=1.0\textwidth]{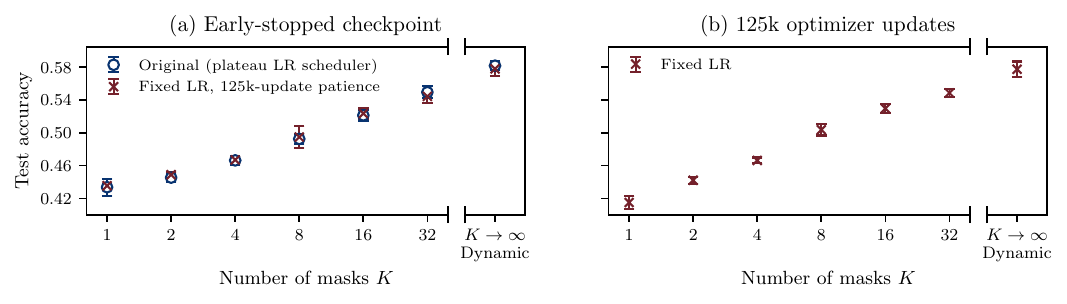}
\caption{\textbf{The downstream mask-diversity trend persists across training protocols.} Test accuracies from \textbf{(a)} are taken from linear probing the early-stopped checkpoints from \Cref{fig:app-cnn-training-original,fig:app-cnn-training-fixedlr}. \textbf{(b)} corresponds to linear probing the checkpoints after $125$k optimizer updates in the setting of \Cref{fig:app-cnn-training-fixedlr}. Markers and bars show mean $\pm1$ SD over five seeds.}
\label{fig:app-cnn-protocol-overlay}
\end{figure}

\subsection{Numerical assessment of finite-mask threshold ratios}
\label{app:threshold-ratio-assessment}

How well does the weak-signal mask-reuse law describe recovery at finite signal strength? \Cref{fig:app-threshold-ratios} compares it with thresholds computed from the full linear equations at zero ridge. We solve the uninformative closures in \eqref{eq:fixed-beta-vgamma}--\eqref{eq:fixed-beta-vGamma-closure} and locate the gain crossing in \eqref{eq:fixed-beta-instability-condition}; \Cref{app:experimental-metadata} gives the numerical procedure.

\paragraph{Variance laws and normalization.} The two rows of \Cref{fig:app-threshold-ratios} use homogeneous noise, $\gamma=1$, and the bounded law
\begin{equation}
 \gamma=1+c\left(\frac94 B-\frac34\right),\qquad B\sim\operatorname{Beta}(1,2),\qquad c\simeq0.5743390142.
 \label{eq:app-threshold-bounded-law}
\end{equation}

This bounded law differs from the example in \Cref{fig:validation}(b).

\begin{figure}[!htbp]
\centering
\includegraphics[width=\textwidth]{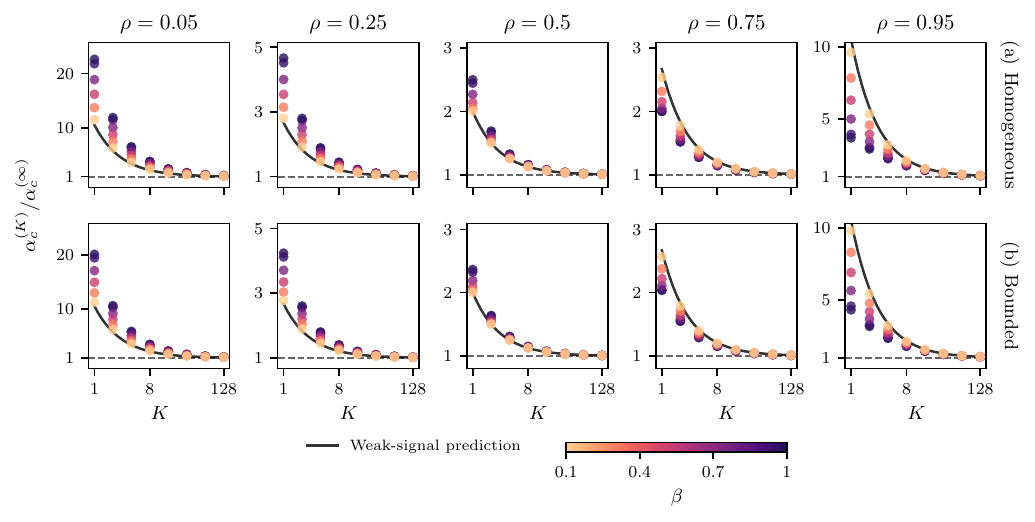}
\caption{\textbf{Finite-mask threshold ratios under homogeneous and bounded noise.} Rows show \textbf{(a)} homogeneous variance distribution $\gamma=1$ and \textbf{(b)} the bounded law \eqref{eq:app-threshold-bounded-law}; columns vary the masking ratio $\rho$. Colors indicate signal strength $\beta$. Markers are numerical teacher-overlap instability thresholds divided by their corresponding dynamic threshold ($K\rightarrow \infty$). The black curve is the weak-signal ratio prediction $1+K^{-1}([2\rho(1-\rho)]^{-1}-1)$; the dashed line marks unity. }
\label{fig:app-threshold-ratios}
\end{figure}

Across the displayed parameters, increasing mask diversity lowers the threshold toward the dynamic reference. The finite-signal deviations from the weak-signal curve depend on $\rho$, while the diminishing cost of mask reuse remains visible.

\subsection{Synthetic experiments and numerical predictions}
\label{app:experimental-metadata}
\paragraph{Data and masking-ratio sweeps.} The synthetic autoencoder simulations use $P_u=P_\lambda=\N(0,1)$ and $\beta=1$. The masked results in \Cref{fig:static-vs-dynamic}(a) and \Cref{fig:K-interpolation}(a) use linear activation, $d=2000$, $\alpha=8$, and $\gamma=\frac14+\frac94 B$ with $B\sim\operatorname{Beta}(1,2)$, as in \Cref{fig:validation}(b). Each condition uses $30$ seeds and $4000$ full-batch Adam updates. \Cref{fig:app-reconstruction-errors} uses the same experiments. Bars show $\pm1$ SD. The red PCA references in \Cref{fig:static-vs-dynamic}(a) and \Cref{fig:K-interpolation}(a) show the asymptotic chance accuracy $1/2$ for this law (\Cref{prop:sep-compact-pca,subsec:sep-example}).

\paragraph{Sample-complexity sweeps.} \Cref{fig:validation} uses the three variance laws stated in its caption, $d=2000$, $\sigma(x)=x$, $\rho=0.75$, $20$ seeds per point, and $4000$ Adam updates for each autoencoder. The no masking points are found by extracting the first principal component of the pretraining set and then linear probing following the procedure detailed in \Cref{sec:main-model}.

In \Cref{fig:mask-sample-efficiency}, the four activations use the bounded law $\gamma=\frac14+\frac94 B$, $B\sim\operatorname{Beta}(1,2)$, as in \Cref{fig:validation}(b), with $d=2000$, $\beta=1$, and $\rho=0.75$. Each autoencoder condition uses $30$ seeds and $4000$ full-batch Adam updates. Both rows report means with $\pm1$ SD across seeds. AMP dimensions and seeds are given in \Cref{app:experimental-bayes}.

\paragraph{Optimization and downstream evaluation.} Adam uses learning rate $0.05$, no schedule or weight decay, and one initialization with iid $\N(0,1)$ coordinates, without warm starts across conditions. Each update averages the reconstruction loss over all hidden entries in the fixed mask collection, or over one fresh mask per sample for dynamic training. The masked autoencoders in \Cref{fig:static-vs-dynamic}(a) and \Cref{fig:K-interpolation}(a), the bounded autoencoders in \Cref{fig:validation}(b), and all masked autoencoders use $\laR=10^{-4}$. Sampled coordinate variances are not renormalized or truncated. Frozen scalar probes use $n_{\rm ds}=100$ independently labeled examples and $10{,}000$ independent test examples from the same quenched spike, with labels $\operatorname{sign}(\lambda)$. For every activation, probing uses the raw, unmasked preactivation $\bw^\top\bx/\sqrt d$. The scalar logistic coefficient has no intercept or regularization and is fitted with L-BFGS for at most $100$ iterations.

\paragraph{Numerical predictions.} The RS learning curves use the small ridge $\laR=10^{-4}$, separately from the zero-ridge threshold calculation in \Cref{fig:app-threshold-ratios}. For the masking-ratio sweeps in \Cref{fig:static-vs-dynamic}(a), \Cref{fig:K-interpolation}(a), and \Cref{fig:app-reconstruction-errors}, we use $200$-node Gauss--Legendre quadrature over the full bounded support $[1/4,5/2]$ and analytic linear output averages.

For \Cref{fig:mask-sample-efficiency}, variance integrals use $200$-node Gauss--Legendre quadrature over the full bounded support $[1/4,5/2]$; the Bayes reference uses $40$ variance nodes. The RS curves use a $120$-point grid over $\alpha\in[0.09,11]$, with denser sampling near the transitions. Linear output averages are analytic. Antithetic scrambled Sobol field integration uses $2^{18}$ samples for tanh at $K=1$ and ReLU at $K=2,4$, and $2^{14}$ for ELU at $K=1,2,4$ and tanh at $K=2,4$. ReLU at $K=1$ and dynamic masking uses Gauss--Hermite field quadrature of order $120$. Dynamic tanh and ELU use order $64$ for both field and mask quadrature. Finite-$K$ ReLU uses its piecewise-quadratic proximal; tanh and ELU use numerical proximal minimization. Fixed-point updates use damping $0.3$ and tolerance $10^{-6}$, except tanh at $K=1$ ($10^{-7}$) and dynamic tanh ($10^{-8}$). Downstream curves include finite-label orientation uncertainty through \eqref{eq:downstream-finite-error}, evaluated with $10^4$ Monte Carlo draws for the masking-ratio sweeps and $10^5$ for the four-activation comparison.

\paragraph{Finite-mask threshold ratios.} For \Cref{fig:app-threshold-ratios}, we use linear activation and $\lambda_R = 0$, evaluate the output channel analytically, and follow the uninformative saddle on a logarithmic grid in $\alpha$. The first crossing of $\Lambda_K=1$ in \eqref{eq:fixed-beta-gain} is refined by logarithmic bisection. Each finite-$K$ threshold is normalized by the dynamic threshold at the same $\beta$, $\rho$, and variance law. The bounded distribution is defined in \eqref{eq:app-threshold-bounded-law} and differs from the law used in \Cref{fig:validation}(b).

\paragraph{Matched-moment noise-adaptation examples.} We evaluate the finite-support expectations $I_{j,\ell}(\tau)$ exactly as weighted sums. For a fixed $\beta$, parameterize the critical curve by $\tau$ using \eqref{eq:dyn-bound-critical-reduction}, then recover $\rho$ from \eqref{eq:dyn-bound-rho-identity}. Follow the regular continuation through $\tau=0$ and solve for the specified masking ratio. Reconstructing the full order parameters verifies the susceptibility, variance, and teacher-response closures. For the three laws in \Cref{tab:mask-adaptation-matched}, the homogeneous law, and $\tfrac12\delta_{1/2}+\tfrac12\delta_{3/2}$, these checks at $\rho\in\{0.8,0.95,0.99\}$ give absolute critical closure residuals below $2\times10^{-13}$. Solutions of \eqref{eq:dyn-bound-informative-reduction} at sample ratios $1.001$, $1.01$, and $1.05$ times each critical ratio have positive alignment and satisfy the full closures to absolute tolerance $10^{-7}$. These checks support the continuous informative continuation used for the reported thresholds.

\subsection{Mask diversity in BERT}
\label{app:bert-masking}

We compare one fixed mask per sequence ($K=1$), a fixed collection of ten masks ($K=10$), and dynamic mask resampling in a BERT-Medium pretraining pilot. Each condition uses three pretraining seeds and two fine-tuning seeds per pretraining seed and task. The corpus, architecture, and pretraining budget are shared across masking conditions.

\paragraph{Implementation and data.} We train a BERT-style masked language model \citep{devlin2019bert} using our adaptation of DinkyTrain\footnote{\url{https://github.com/princeton-nlp/DinkyTrain}.} \citep{wettig2023should}, built on fairseq \citep{ott2019fairseq} and its RoBERTa implementation \citep{liu2019roberta}. DinkyTrain's efficient pretraining recipe builds on \citet{izsak2021train}. The pre-LayerNorm encoder has eight layers, hidden dimension $512$, feed-forward dimension $2048$, eight attention heads, and approximately $51.5$ million pretraining parameters. We use the repository's preprocessed Wikipedia--BookCorpus data \citep{zhu2015aligning}, with GPT-2 byte-pair encoding and a $50{,}265$-token model vocabulary. The pilot uses one fixed shard, retaining $3{,}756{,}032$ contiguous sequences of length $128$. Token blocks are constructed before mask views are assigned; the underlying text remains fixed across presentations.

\paragraph{Masking protocol.} We select $15\%$ of token positions for prediction. Every selected token is replaced by \texttt{[MASK]}; neither random-token substitution nor unchanged selected tokens are used. For finite $K$, the corruption depends on the mask seed, sequence identifier, and view identifier, and is fixed across presentations. Dynamic masking draws a fresh corruption at each presentation. Each sequence is presented $109$ times in every condition: the single $K=1$ mask is reused throughout, the ten $K=10$ masks are reused approximately eleven times each, and dynamic masking draws $109$ masks. Pretraining seeds are $0,1,2$, with the training and mask seed both set to the same value. Seeds control initialization, data order, dropout, and mask draws. Held-out validation uses $1{,}941$ sequences.

\paragraph{Pretraining budget and optimization.} Each run completes $99{,}953$ optimizer updates with an effective batch size of $4{,}096$. Adam uses a peak learning rate of $2\times10^{-3}$, $1{,}999$ warmup updates followed by cosine decay, coefficients $(0.9,0.98)$, $\epsilon=10^{-6}$, weight decay $0.01$, and no gradient clipping. Dropout and attention dropout are both $0.1$. Pretraining checkpoints are selected using held-out masked-language-model validation loss.

\begin{table}[htbp]
\centering
\small
\caption{\textbf{BERT downstream results.} Development scores ($\times100$, $\uparrow$), reported as mean $\pm$ SD across three pretraining seeds after averaging two fine-tuning runs per pretraining seed. MCC denotes Matthews correlation; STS-B uses Spearman correlation. MNLI-m and MNLI-mm are the two splits of the same task.}
\label{tab:bert-glue-means}
\begin{tabular}{llccc}
\toprule
Task & Metric & $K=1$ & $K=10$ & Dynamic \\
\midrule
MNLI-m & Accuracy & $72.85\pm0.33$ & $76.27\pm0.42$ & $76.83\pm0.18$ \\
MNLI-mm & Accuracy & $74.47\pm0.29$ & $77.48\pm0.13$ & $77.65\pm0.35$ \\
QNLI & Accuracy & $82.55\pm0.87$ & $84.40\pm0.61$ & $85.45\pm0.38$ \\
QQP & Accuracy & $85.92\pm0.10$ & $87.02\pm0.16$ & $87.22\pm0.16$ \\
SST-2 & Accuracy & $87.63\pm0.60$ & $88.08\pm1.03$ & $89.40\pm0.46$ \\
STS-B & Spearman & $79.76\pm1.18$ & $81.32\pm1.04$ & $82.46\pm1.67$ \\
MRPC & MCC & $46.27\pm1.94$ & $55.50\pm2.63$ & $55.55\pm5.38$ \\
CoLA & MCC & $37.23\pm1.80$ & $43.30\pm2.82$ & $47.17\pm3.40$ \\
RTE & MCC & $9.43\pm3.29$ & $10.42\pm1.95$ & $7.28\pm4.27$ \\
\bottomrule
\end{tabular}
\end{table}

\paragraph{Downstream optimization.} We fine-tune the full model on eight GLUE tasks \citep{wang2018glue} for ten epochs and evaluate after the last epoch on the development split. The maximum input length is $512$ tokens. Learning rate and batch size are $10^{-5}$ and $32$ for MNLI, QNLI, QQP, and SST-2; $2\times10^{-5}$ and $16$ for STS-B; and $5\times10^{-5}$ and $16$ for MRPC, CoLA, and RTE.

For MRPC, CoLA, and RTE, a preliminary learning-rate sweep resolved the majority-class predictions observed with the initial fine-tuning settings. We selected $5\times10^{-5}$ using the $K=10$ model with pretraining seed~0 from $\{5\times10^{-5},10^{-4},2\times10^{-4},4\times10^{-4}\}$ and applied it across all masking conditions and seeds. This value is included in the released fine-tuning grid accompanying the compact BERT models of \citet{turc2019well}.\footnote{\url{https://github.com/google-research/bert}.} Adam uses coefficients $(0.9,0.98)$, $\epsilon=10^{-6}$, weight decay $0.1$, and no gradient clipping. The learning rate warms up for $6\%$ of the update budget and then decays linearly. Dropout and attention dropout are both $0.1$. Classification tasks use cross-entropy; STS-B uses mean-squared-error regression.

\paragraph{Downstream performance.} In \Cref{tab:bert-glue-means}, $K=10$ improves the mean score over $K=1$ on all eight tasks. Dynamic masking has the highest mean on seven of eight tasks, including both MNLI splits. The additional gain over $K=10$ depends on the task: the MRPC means are nearly equal, while RTE remains weak and variable. The results support a benefit of mask diversity across independent pretraining runs.

\subsection{Bayes curves and AMP estimates}
\label{app:experimental-bayes}

The Bayes-optimal lines in \Cref{fig:validation,fig:mask-sample-efficiency,fig:app-heterogeneous-pca} show the state-evolution solution of \eqref{eq:BO-complete-scalar-root}.

Green diamonds use the alternating, Fisher-averaged AMP from \Cref{subsec:BO-AMP}, specializing the bipartite Low-RAMP of \citet{lesieur2017constrained}. Inference uses the known signal strength and Gaussian priors, with coordinate variances estimated by the empirical second moments in \eqref{eq:BO-empirical-gamma}.

\paragraph{Initialization and convergence.} The plotted runs use weak teacher-informed initialization
\begin{equation}
\widehat{\bu}^{\,0}=\sqrt p\,\bu^\star+\sqrt{1-p}\,\bz,\qquad \bz\sim\N(0,I_d),
\end{equation}
In \Cref{fig:validation,fig:mask-sample-efficiency,fig:app-heterogeneous-pca}, AMP uses $d=4000$, $p=0.05$, and ten seeds per $\alpha$ value. Initial sample-amplitude estimates are Gaussian with variance $0.1$, with damping factor $\eta=0.75$.

\paragraph{Matched downstream comparison.} The lower row of \Cref{fig:mask-sample-efficiency} applies the same unweighted raw-feature logistic probe to AMP and autoencoder estimates, using $100$ labels and $10{,}000$ test examples. Its green curve combines the Bayes pretraining solution with the raw-feature SNR \eqref{eq:downstream-AMP-snr} and the finite-label logistic-probe formulas in \Cref{subsec:downstream-classification}.

\end{document}

%% file: authors.tex
\author{
\normalsize Jorge Medina Moreira\thanks{Corresponding author. jorge.medinamoreira@epfl.ch}, Lorenzo Bardone, Lenka Zdeborová\\
\small\itshape Statistical Physics of Computation Laboratory,\\
\small\itshape École Polytechnique Fédérale de Lausanne (EPFL), CH-1015 Lausanne, Switzerland
}
\hypersetup{
  pdftitle={The Hidden Advantage of Mask Resampling: A Theory of Masked Autoencoders},
  pdfauthor={Jorge Medina Moreira, Lorenzo Bardone, Lenka Zdeborová}
}

%% file: appendices/appendix_dynamic_linear_bound.tex
\subsection{Dynamic linear masking at fixed signal strength}
\label{app:dynamic-linear-bound}

For linear activation, the fixed-signal equations give both the recovery threshold and the feature alignment above it. We first identify when dynamic masking favors cleaner coordinates, then recover the homogeneous PCA learning curve. The comparison with static masking is developed in \Cref{app:mask-noise-adaptation}. Throughout, $\sig(h)=h$ and $\laR=0$.

\subsubsection{Quadratic output and the critical equations}
\label{subsec:dyn-bound-critical}
Write $Q=q_I$, $C=q_\Gamma$, and $\Delta=v_\Gamma$. The exact dynamic effective loss for the linear activation is
\begin{equation}
    \ell_{\rm dyn}^{\rm lin}(h;Q,C)=a h^2+bC, \qquad a=\rho(1-\rho)((1-\rho)Q-2),\qquad b=\rho(1-\rho)(\rho Q+2).
    \label{eq:dyn-bound-loss}
\end{equation}
For $D=1+2a\Delta>0$, the output proximal and force are
\begin{equation}
    h_\star=h_0/D,\qquad F=\mathcal{X} h_0,\qquad \mathcal{X}=-2a/D,\qquad h_0=\sqrt\beta M\lambda+\sqrt C z.
    \label{eq:dyn-bound-prox}
\end{equation}
Set $T=\beta M^2+C$. The explicit derivatives of the loss give
\begin{align}
    U&=\rho(1-\rho)\left((1-\rho)T/D^2+\rho C\right), &H&=b-\mathcal{X}/2,\nonumber\\
    \widehat m&=\alpha\beta\mathcal{X} M, &\widehat q(\gamma)&=\alpha\gamma\mathcal{X}^2T, &\widehat v(\gamma)&=2\alpha(U+\gamma H).
    \label{eq:dyn-bound-hats}
\end{align}
In particular, $H$ includes the explicit derivative $\partial_C\ell=b$; discarding it would impose the negative variance coefficient characteristic of the full-reconstruction susceptibility. On a nontrivial regular saddle, $U>0$ and
\begin{equation}
    v(\gamma)=\frac{1}{2\alpha U(1+\tau\gamma)},\qquad \tau=H/U,\qquad 1+\tau\gamma>0\quad P_\gamma\text{-a.s.}
    \label{eq:dyn-bound-v}
\end{equation}
Define, whenever the indicated moments are finite,
\begin{equation}
    I_{j,\ell}(\tau):=\E_\gamma \left[\frac{\gamma^j}{(1+\tau\gamma)^\ell}\right].
    \label{eq:dyn-bound-moments}
\end{equation}
At a regular continuous critical point with $M=0$ and $C>0$, the variance closure and the magnetization instability condition are
\begin{equation}
    1=\alpha\mathcal{X}^2\E_\gamma[\gamma^2v(\gamma)^2], \qquad 1=\alpha\beta\mathcal{X}\E_\gamma[v(\gamma)].
    \label{eq:dyn-bound-two-closures}
\end{equation}
Consequently $\mathcal{X}>0$, $0<D<1$, and
\begin{equation}
\begin{aligned}
    \alpha_{c,\mathrm{RS}}^{\rm dyn} &=\frac{I_{2,2}(\tau_c)}{\beta^2 I_{0,1}(\tau_c)^2} =\frac{\mathcal R_\gamma(\tau_c)}{\beta^2},\\
    \frac{Q}{C}&=\frac{I_{1,2}}{I_{2,2}},\qquad D=\left(1+\frac{I_{1,1}}{\sqrt{\alpha I_{2,2}}}\right)^{-1}.
\end{aligned}
    \label{eq:dyn-bound-critical-reduction}
\end{equation}
This is the fixed-signal factorization with output-efficiency factor one; no weak-signal expansion is involved.

\subsubsection{Upper bound and the masking condition}
\label{subsec:dyn-bound-rho}
A positive response parameter favors cleaner coordinates and lowers the threshold below the second-moment scale $\kappa_2/\beta^2$. We now relate this improvement to the masking rate. The monotonicity argument in \cref{app:fixed-beta-comparison} gives
\begin{equation}
    \frac{\dd}{\dd\tau}\log\mathcal R_\gamma(\tau) =-2\frac{I_{3,3}}{I_{2,2}}+2\frac{I_{1,2}}{I_{0,1}} \leq0,\qquad \mathcal R_\gamma(0)=\kappa_2.
    \label{eq:dyn-bound-monotone}
\end{equation}
For finite $\kappa_2$, continuity at zero extends this monotonicity bound even when $\kappa_3=\infty$. When the inverse moments are finite, the same calculation gives
\begin{equation}
    \frac{1}{\beta^2\E_\gamma[\gamma^{-2}]} \leq\frac{1}{\beta^2\E_\gamma[\gamma^{-1}]^2} \leq\alpha_{c,\mathrm{RS}}^{\rm dyn} \leq\frac{\kappa_2}{\beta^2},\qquad \tau_c\geq0.
    \label{eq:dyn-bound-sandwich}
\end{equation}
The last inequality needs only the finite second moment.

We now determine the sign of $\tau_c$. At $M=0$, write $A=1-\rho+\rho D^2$. Using $H=\rho(1-\rho)Q-\mathcal{X}(1-D)/2$, $U=\rho(1-\rho)CA/D^2$, and \eqref{eq:dyn-bound-two-closures}, one obtains
\begin{equation}
    \tau I_{2,2}+D I_{1,1}=\frac{D^2I_{1,2}}{A},\qquad \frac{I_{1,2}}{I_{1,1}}=\frac{(1+D)A}{A+D^2}.
    \label{eq:dyn-bound-rho-identity}
\end{equation}
Here $I_{1,1}-I_{1,2}=\tau I_{2,2}$. Since $0<D<1$,
\begin{equation}
    \tau\geq0 \quad\Longleftrightarrow\quad A\leq D \quad\Longleftrightarrow\quad \rho\geq\frac{1}{1+D}.
    \label{eq:dyn-bound-sign-D}
\end{equation}
At fixed $\alpha$, $D(\tau)$ in \eqref{eq:dyn-bound-critical-reduction} is nonincreasing, because
\begin{equation}
    \frac{\dd}{\dd\tau}\log\frac{I_{1,1}}{\sqrt{I_{2,2}}} =\frac{I_{1,1}I_{3,3}-I_{2,2}^2}{I_{1,1}I_{2,2}}\geq0.
\end{equation}
The numerator is nonnegative by Cauchy--Schwarz. At zero, $D(0)=\sqrt{\alpha\kappa_2}/(1+\sqrt{\alpha\kappa_2})$. Comparing the positive and negative sides of zero in \eqref{eq:dyn-bound-sign-D} therefore yields
\begin{equation}
    \tau_c\geq0 \quad\Longleftrightarrow\quad \rho\geq\frac{1+\sqrt{\alpha_c\kappa_2}} {1+2\sqrt{\alpha_c\kappa_2}}.
    \label{eq:dyn-bound-rho-alpha}
\end{equation}
For a nondegenerate law, strict monotonicity of $\mathcal R_\gamma$ implies that $\tau_c>0$ corresponds to $\alpha_c<\kappa_2/\beta^2$, and $\tau_c<0$ to the reverse inequality. The right-hand side of \eqref{eq:dyn-bound-rho-alpha} decreases with $\alpha_c$. Evaluating it at $\kappa_2/\beta^2$ gives the equivalent fixed-signal statement
\begin{equation}
    \tau_c\geq0\quad\Longleftrightarrow\quad \rho\geq\frac{\beta+\kappa_2}{\beta+2\kappa_2} \quad\Longrightarrow\quad \alpha_{c,\mathrm{RS}}^{\rm dyn}\leq\frac{\kappa_2}{\beta^2}.
    \label{eq:dyn-bound-final-rho}
\end{equation}
Equality holds at $\tau_c=0$. For the homogeneous law the threshold is $\beta^{-2}$ for every regular critical solution, irrespective of the sign of $\tau_c$.

For unbounded support, positivity in \eqref{eq:dyn-bound-v} rules out $\tau_c<0$, so the masking condition also bounds the domain of this continuous-onset calculation. It does not imply that recovery is impossible at smaller masking rates: under the Gaussian model with variances bounded away from zero and finite second moment, \Cref{thm:unbounded-dynamic-recovery} guarantees recovery at sufficiently large sample ratio for every $\rho\in(0,1)$. The transition mechanism at those smaller rates remains to be determined.

\subsubsection{Feature alignment above threshold and the homogeneous benchmark}
\label{subsec:dyn-bound-informative}
The full linear dynamic equations can also be reduced without linearizing in $M$. To avoid confusing the cosine $\theta_{\bu^\star}$ with a ratio of second moments, define $z_M=M^2/C$ and $r_Q=Q/C$. For $M\ne0$, the magnetization equation gives $\mathcal{X}/(2U)=1/(\beta I_{0,1})$, hence
\begin{equation}
    q(\gamma)=\frac{M^2+\gamma T/(\alpha\beta^2)} {I_{0,1}^2(1+\tau\gamma)^2}.
\end{equation}
The remaining equations reduce to
\begin{align}
    z_M&=\frac{I_{0,1}^2-I_{2,2}/(\alpha\beta^2)} {I_{1,2}+I_{2,2}/(\alpha\beta)},\nonumber\\
    r_Q&=\frac{z_MI_{0,2}+(1+\beta z_M)I_{1,2}/(\alpha\beta^2)} {I_{0,1}^2},\nonumber\\
    D&=\left(1+\frac{I_{1,1}}{\alpha\beta I_{0,1}}\right)^{-1},\nonumber\\
    \Phi(\tau,\alpha)&:=\tau- \frac{r_Q}{\rho+(1-\rho)(1+\beta z_M)D^{-2}} +\frac{1-D}{\beta I_{0,1}}=0.
    \label{eq:dyn-bound-informative-reduction}
\end{align}
All $I_{j,\ell}$ are evaluated at $\tau$; $\beta,\rho,P_\gamma$ are fixed in the notation $\Phi(\tau,\alpha)$. A root must have $z_M>0$ and satisfy the regularity conditions. Its scale is recovered from
\begin{equation}
\begin{aligned}
    C&=\frac{2}{(1-\rho) r_Q+ [\rho+(1-\rho)(1+\beta z_M)D^{-2}]D/(\beta I_{0,1})},\\
    Q&=r_Q C,\qquad M^2=z_M C.
\end{aligned}
\end{equation}
The squared alignment is $\theta_{\bu^\star}^2=z_M/r_Q$. For $P_\gamma=\delta_1$, all susceptibility factors cancel, $r_Q=1$, and this becomes
\begin{equation}
    \theta_{\bu^\star}^2 =\frac{\alpha\beta^2-1}{\alpha\beta^2+\beta}, \qquad \alpha>\beta^{-2},
    \label{eq:dyn-bound-homogeneous}
\end{equation}
which recovers the PCA learning curve throughout the informative regime, consistent with the agreement observed in \Cref{fig:validation}(a).

%% file: appendices/appendix_mask_adaptation.tex
\subsection{Mask diversity changes how the learner uses heterogeneous noise}
\label{app:mask-noise-adaptation}

The weak-signal law in \eqref{eq:mask-diversity-threshold} describes an additional sample requirement caused by mask reuse. At finite signal strength, the replica equations reveal a further distinction: static and dynamic masking can use the coordinate noise profile differently when learning first becomes possible. Static masking has a noise-independent response to an emerging feature at its regular linear recovery transition. Dynamic masking can instead give cleaner coordinates a larger response, lowering the threshold through an adaptation that the static critical solution lacks. Thus the protocols differ not only in how many samples they require, but also in how they extract the feature from those samples.

We develop this interpretation using the fixed-signal factorization in \cref{subsec:fixed-beta-factorization} and the dynamic linear equations in \cref{app:dynamic-linear-bound}. Throughout this subsection, $\sig(h)=h$, $\laR=0$, $\E_\gamma\gamma=1$, and $\kappa_2=\E_\gamma\gamma^2<\infty$; the concrete examples have finite positive variance support and Gaussian latent priors. We evaluate the response at the continuous recovery onset defined above.

\subsubsection{Coordinate response separates noise adaptation from mask reuse}
\label{subsec:mask-adaptation-response}

The coordinate susceptibility $v_c(\gamma)$ describes the response to an infinitesimal informative field near recovery. In particular, the input equation gives $m(\gamma)=\widehat m v(\gamma)$, with the same $\widehat m$ at every coordinate variance. Ratios of susceptibilities therefore determine where the emerging feature alignment is concentrated. Using \eqref{eq:fixed-beta-v-tau}, for two variances $0<\gamma_1<\gamma_2$,
\begin{equation}
    \frac{v_c(\gamma_1)}{v_c(\gamma_2)} =\frac{1+\tau_c\gamma_2}{1+\tau_c\gamma_1}.
    \label{eq:mask-adaptation-response-ratio}
\end{equation}
For $\tau_c>0$, this ratio exceeds one: the response favors cleaner coordinates. At $\tau_c=0$ it is uniform, while negative $\tau_c$ favors noisier coordinates on its admissible domain. This is a statement about the response to the shared feature, not equality or ordering of the realized coordinate weights. The noise levels are not supplied to the autoencoder; their influence on the response emerges from the reconstruction objective and the data.

The general finite-$K$ factorization \eqref{eq:fixed-beta-factorization-R} separates two contributions: $\eta_{K,c}^2$ measures output efficiency, while $\mathcal R_\gamma(\tau_{K,c})$ measures how coordinate reweighting changes the effective noise burden. Both depend on mask diversity through the full equations. The static and dynamic endpoints below show how the response itself changes, beyond the sample penalty captured by the weak-signal law.

\subsubsection{Why static masking cannot adapt its response at the transition}
\label{subsec:static-noise-response}

For $K=1$, the regular uninformative equations force $\mathcal B_{1,0}=0$. We give the short elimination because this identity is the source of the distinction, rather than an imposed approximation. Write $Q=q_{I,0}$, $C=q_{\Gamma,0}$, and $\Delta=v_{\Gamma,0}$, and use the visible and hidden fields $y,t$ of the $K=1$ channel. At $M=0$, their Gaussian centers $y_0,t_0$ are independent, with variances $(1-\rho)C$ and $\rho C$. The quadratic loss is $\rho Qy^2-2yt$. Eliminating $t$ in the proximal gives
\begin{equation}
    y_\star=\frac{y_0/[(1-\rho)\Delta]+2t_0}{\mathcal D_{\rm stat}}, \qquad \mathcal D_{\rm stat} =\frac1{(1-\rho)\Delta}+2\rho(Q-2\Delta)>0.
    \label{eq:mask-adaptation-static-prox}
\end{equation}
The last inequality is the remaining proximal-curvature condition. Evaluating the definitions in \eqref{eq:fixed-beta-A}, \eqref{eq:fixed-beta-B}, and \eqref{eq:fixed-beta-E} gives
\begin{align}
    \mathcal A_{1,0} &=\frac{\rho C}{\mathcal D_{\rm stat}^2} \left[\frac1{(1-\rho)\Delta^2}+4\rho\right],\nonumber\\
    \mathcal B_{1,0} &=\frac1{2\Delta}-\frac1{2\mathcal D_{\rm stat}} \left[\frac1{(1-\rho)\Delta^2}+4\rho\right],\nonumber\\
    \mathcal E_{1,0} &=\left[\frac1{(1-\rho)\Delta^2}+4\rho\right] \frac{\mathcal A_{1,0}}{\rho} +\frac C{\Delta^2} -\frac{2C}{(1-\rho)\Delta^3\mathcal D_{\rm stat}}.
    \label{eq:mask-adaptation-static-output}
\end{align}
The input equations \eqref{eq:fixed-beta-vgamma}--%
\eqref{eq:fixed-beta-vGamma-closure} imply
\begin{equation}
    \mathcal E_{1,0}\Delta =2\mathcal A_{1,0}Q+2\mathcal B_{1,0}C.
    \label{eq:mask-adaptation-static-input}
\end{equation}
Substituting the explicit output expressions into the difference of the two sides yields the identity
\begin{equation}
    \mathcal E_{1,0}\Delta-2\mathcal A_{1,0}Q-2\mathcal B_{1,0}C =-\frac{4C\mathcal B_{1,0}} {(1-\rho)\Delta\mathcal D_{\rm stat}}.
    \label{eq:mask-adaptation-static-cancellation}
\end{equation}
Since $C,\Delta,\mathcal D_{\rm stat}>0$ and $0<\rho<1$, the input closure forces $\mathcal B_{1,0}=0$. Consequently,
\begin{equation}
    \tau_{1,c}=0,\qquad v_0(\gamma)=\Delta,\qquad Q=4\Delta,\qquad C=\kappa_2Q.
    \label{eq:mask-adaptation-static-flat}
\end{equation}
Here the constant susceptibility follows from $\E_\gamma\gamma=1$; the relations for $Q,C$ follow from \eqref{eq:mask-adaptation-static-output} and the variance closures. In fact $q_0(\gamma)=4\gamma\Delta$, so noisy coordinates still contribute more to the uninformative weight fluctuations.

The remaining closure gives
\begin{equation}
    \Delta=\frac1{2\sqrt{\rho(1-\rho)(2\alpha\kappa_2-1)}}, \qquad 2\alpha\kappa_2>1.
    \label{eq:mask-adaptation-static-delta}
\end{equation}
Evaluating the teacher-response gain \eqref{eq:fixed-beta-instability-condition} then gives
\begin{equation}
    \alpha\beta\mathcal X_{1,0}\E_\gamma v_0 =\frac\beta{\kappa_2} \left[\sqrt{\rho(1-\rho)(2\alpha\kappa_2-1)} +\frac{2\rho-1}{2}\right].
    \label{eq:mask-adaptation-static-gain}
\end{equation}
Setting this gain to one yields the regular static threshold
\begin{equation}
    \alpha_c^{\rm stat} =\frac1{2\kappa_2} \left[1+\frac{(\kappa_2/\beta+1/2-\rho)^2}{\rho(1-\rho)}\right], \qquad \frac{\kappa_2}{\beta}+\frac12-\rho>0.
    \label{eq:mask-adaptation-static-threshold}
\end{equation}
Thus two variance laws with the same mean and second moment have the same regular static threshold, even if their clean coordinates occur with very different frequencies. This insensitivity is specific to the uninformative saddle and its recovery transition: it does not imply that static masking cannot develop noise-dependent feature alignment after recovery.

\subsubsection{Dynamic masking exploits cleaner coordinates at recovery onset}
\label{subsec:dynamic-noise-response}

Mask averaging changes the output channel. In the linear case, its explicit dependence on $C$ contributes to $H$ in \eqref{eq:dyn-bound-hats}; the static cancellation above no longer applies. The resulting susceptibility and threshold are already derived in \eqref{eq:dyn-bound-v} and \eqref{eq:dyn-bound-critical-reduction}. Their interpretation follows from the monotonicity proved in \cref{subsec:fixed-beta-R-bounds}: positive $\tau_c$ downweights high-noise coordinates and strictly lowers $\mathcal R_\gamma$ below its neutral value $\kappa_2$ for a nondegenerate variance law.

The masking condition \eqref{eq:dyn-bound-final-rho} makes the relevant regime explicit:
\begin{equation}
    \rho>\frac{\beta+\kappa_2}{\beta+2\kappa_2} \quad\Longrightarrow\quad \tau_c>0,\qquad \alpha_c^{\rm dyn}<\frac{\kappa_2}{\beta^2},
    \label{eq:mask-adaptation-dynamic-improvement}
\end{equation}
whenever the regular critical saddle exists and the variance law is nonconstant. The learner can therefore exploit the heterogeneous reliability of coordinates while the informative direction is only beginning to emerge. Static masking, whose response is uniform at this point, cannot obtain the same threshold reduction by such reweighting.

For $0<\beta<2\kappa_2$, optimizing \eqref{eq:mask-adaptation-static-threshold} gives
\begin{equation}
    \min_\rho\alpha_c^{\rm stat}=\frac{2\kappa_2}{\beta^2}, \qquad \rho_{\rm stat}^\star=\frac12+\frac\beta{4\kappa_2}.
    \label{eq:mask-adaptation-static-optimum}
\end{equation}
Hence a dynamic critical saddle with $\tau_c>0$ requires less than half the sample ratio of even the best regular static crossing in this regime. The factor-of-two weak-signal comparison is therefore not the whole advantage: finite-signal noise adaptation can increase it. The signal range above is the domain of this optimized-static comparison. The dynamic masking boundary also depends on $\beta$; its advantage at fixed $\rho$ therefore need not increase with signal strength.

\subsubsection{Matched second moments expose the different mechanisms}
\label{subsec:mask-adaptation-examples}

Consider the three mean-one laws
\begin{align}
    P_\gamma^{\rm balanced} &=\tfrac12\delta_{1-\sqrt{0.45}} +\tfrac12\delta_{1+\sqrt{0.45}},\nonumber\\
    P_\gamma^{\rm rare\ high} &=\tfrac{320}{329}\delta_{71/80}+\tfrac9{329}\delta_5,\nonumber\\
    P_\gamma^{\rm very\ clean} &=\tfrac{180}{541}\delta_{1/20} +\tfrac{361}{541}\delta_{28/19}.
    \label{eq:mask-adaptation-matched-laws}
\end{align}
Each has $\kappa_2=29/20$, but the distributions of reliable and unreliable coordinates are different. The balanced law divides its mass equally between two noise levels. The second assigns a small fraction of coordinates a high variance. The third contains a substantial fraction of particularly clean coordinates.

\begin{table}[!htbp]
\centering
\caption{\textbf{Identical static thresholds, different dynamic adaptation.} Regular linear replica thresholds for the laws in \eqref{eq:mask-adaptation-matched-laws} at $\beta=1.5$, $\rho=0.95$, and zero ridge. All three laws have unit mean and $\kappa_2=1.45$. The response ratio compares the lowest- and highest-variance coordinates using \eqref{eq:mask-adaptation-response-ratio}. Entries are computed numerically from the replica equations.}
\label{tab:mask-adaptation-matched}
\begin{tabular}{lrrrr}
\toprule
Variance law & $\alpha_c^{\rm stat}$ & $\alpha_c^{\rm dyn}$
 & $\tau_c^{\rm dyn}$ & $v_c(\gamma_{\min})/v_c(\gamma_{\max})$\\
\midrule
Balanced & 2.28272 & 0.52222 & 0.19380 & 1.24442\\
Rare high variance & 2.28272 & 0.43843 & 0.26078 & 1.87089\\
Very clean component & 2.28272 & 0.54817 & 0.17987 & 1.25380\\
\bottomrule
\end{tabular}
\end{table}

\Cref{tab:mask-adaptation-matched} shows that matching the static threshold does not match dynamic learning. In this example, resampling is particularly effective for the rare-high-variance law, for which the response most strongly suppresses the noisy component. The common optimized regular static threshold is $1.28889$, so the dynamic benefit also survives optimizing the static masking ratio. The comparison isolates dependence on the shape of $P_\gamma$ beyond its second moment, rather than a change in average noise power or signal strength.

The entries are obtained from the linear critical equations, with numerical checks reported in \Cref{app:experimental-metadata}.

The rare-high-variance law even has $\alpha_c^{\rm dyn}=0.43843<1/\beta^2=0.44444$, the homogeneous threshold at the same mean variance and signal strength. Conversely, the very clean component does not give the smallest dynamic threshold: the benefit depends on how the reconstruction objective uses the entire noise profile. Dynamic adaptation still differs from fully noise-aware inference, as quantified by the inverse-moment bounds in \eqref{eq:dyn-bound-sandwich} and the Bayes comparison in \Cref{app:heteroPCAs}.

\newpage

%% file: appendices/appendix_linear_separation.tex
\subsection{PCA obstructions and bounded-noise masking recovery}
\label{app:linear-separation}

This subsection establishes the PCA obstructions and bounded-noise masking recovery used in \cref{thm:main-separation}; \cref{app:dynamic-unbounded} extends masking recovery beyond bounded variance support. The PCA statements concern the leading empirical principal component. The masking statements concern every global minimizer of the exact mask-averaged objective, with no replica assumption and no claim about convergence of a particular optimization algorithm. The sufficient sample ratio is deliberately non-sharp.

\subsubsection{Model, PCA reduction, and the homogeneous benchmark}
\label{subsec:sep-pca}
Use \eqref{eq:data-model} with independent standard Gaussian $u_i^\star,\lambda_\mu,z_{\mu i}$, independent positive $\gamma_i^\star\iid P_\gamma$, and $n/d\to\alpha\in(0,\infty)$. Define
\begin{equation}
    \bm e=\frac{\bu^\star}{\|\bu^\star\|},\qquad \beta_d=\frac{\beta\|\bu^\star\|^2}{d},\qquad \Sigma_d=\Gamma^\star+\beta_d\bm e\bm e^\top,\qquad \widehat\Sigma=\frac1n\sum_\mu\bx_\mu\bx_\mu^\top.
    \label{eq:sep-model}
\end{equation}
Conditionally on the quenched variables $\bu^\star,\Gamma^\star$, the samples are independent $\N(0,\Sigma_d)$ vectors. With the sample-first matrices $X,Z$ of \cref{app:notation}, write
\begin{equation}
    Y:=\frac{X}{\sqrt n} =N+\sqrt\beta\,\frac{\bm\lambda}{\sqrt n} \frac{(\bu^\star)^\top}{\sqrt d}, \qquad N=\frac1{\sqrt n}Z(\Gamma^\star)^{1/2}.
    \label{eq:sep-data-matrix}
\end{equation}
Both $Y$ and $N$ are $n\times d$, and $\widehat\Sigma=Y^\top Y$. The PCA direction $\widehat{\br}_{\rm PCA}$ is the leading right singular vector of $Y$; its squared overlap is $|\bm e^\top\widehat{\br}_{\rm PCA}|^2$.

For a linear tied autoencoder, $\bw=\sqrt{dQ}\,\br$, $\|\br\|=1$, produces $\widehat\bx=Q\br\br^\top\bx$. Therefore \eqref{eq:main-full-objective}, divided by $n$ and with the target-only constant $n^{-1}\sum_\mu\|\bx_\mu\|^2$ subtracted, is exactly
\begin{equation}
    R_{\rm full}(Q,\br) =(Q^2-2Q)\br^\top\widehat\Sigma\br.
    \label{eq:sep-linear-pca-identity}
\end{equation}
Since the largest sample eigenvalue is positive, the global minimum has $Q=1$ and $\br$ in its top eigenspace. This proves the finite-dimensional equivalence with PCA, including the radial optimization.

For $P_\gamma=\delta_1$, the rectangular rank-one singular-vector transition \citep{benaych2012singular} gives
\begin{equation}
    |\bm e^\top\widehat{\br}_{\rm PCA}|^2\longrightarrow
    \begin{cases}
       0,&\alpha\beta^2\leq1,\\[2pt]
       \displaystyle\frac{1-1/(\alpha\beta^2)}{1+1/(\alpha\beta)}, &\alpha\beta^2>1,
    \end{cases}
    \label{eq:sep-homogeneous-pca}
\end{equation}
in probability. \Cref{subsec:dyn-bound-informative} derives the matching regular dynamic linear RS formula.

\subsubsection{PCA failure for unbounded coordinate noise}
\label{subsec:sep-unbounded}
\begin{proposition}[Unbounded variance support]
\label[proposition]{prop:sep-unbounded-pca}
Under \eqref{eq:sep-model}, if $P_\gamma$ has unbounded support and $\E_\gamma\gamma=1$, then for every fixed $\beta>0$ and $\alpha>0$, $|\bm e^\top\widehat{\br}_{\rm PCA}|^2\to0$ in probability.
\end{proposition}
\begin{proof}
Let $i_d$ maximize $\gamma_i^\star$. Since $P_\gamma$ is unbounded, $\gamma_{i_d}^\star\to\infty$ in probability. The selected column of $Z$ is independent of this selection, and its squared norm divided by $n$ converges in probability to one. Thus
\begin{equation}
    \|N\|_{\mathrm{op}}\geq\sqrt{\gamma_{i_d}^\star}\, \frac{\|Z_{:,i_d}\|}{\sqrt n} \longrightarrow\infty.
\end{equation}
The signal term in \eqref{eq:sep-data-matrix} has operator norm $\sqrt\beta\|\bu^\star\|\|\bm\lambda\|/\sqrt{dn}=O_{\mathbb P}(1)$. Hence the top singular value $s_1(Y)$ also diverges in probability. On the other hand,
\begin{equation}
    \E\|Y\bm e\|^2 =\E[\bm e^\top\Sigma_d\bm e]=1+\beta,
\end{equation}
because $\bm e$ is independent of $\Gamma^\star$ and $\E\beta_d=\beta$. $\|Y\bm e\|$ is consequently bounded in probability. Let $\widehat{\bm v}_1\in\R^n$ be the leading left singular vector paired with $\widehat{\br}_{\rm PCA}$. Then $Y^\top\widehat{\bm v}_1=s_1(Y)\widehat{\br}_{\rm PCA}$ gives
\begin{equation}
    |\bm e^\top\widehat{\br}_{\rm PCA}| =\frac{|\widehat{\bm v}_1^\top Y\bm e|}{s_1(Y)}\le\frac{\|\widehat{\bm v}_1^\top\|\|Y\bm e\|}{s_1(Y)}\longrightarrow0.
\end{equation}
This argument needs no assumptions about localization of the top eigenvector and no finite second moment of $\gamma$.
\end{proof}

\subsubsection{A bounded-noise obstruction at every sample ratio}
\label{subsec:sep-compact-pca}
The following lemma identifies when the population covariance has an eigenvalue separated from the variance support.

\begin{lemma}[Population outlier equation]
\label[lemma]{lem:sep-population-outlier}
Under \eqref{eq:sep-model}, suppose $P_\gamma$ is supported in $[a,\gamma_+]\subset(0,\infty)$, with $\gamma_+$ its essential upper endpoint. Let $m_+=\E_\gamma[(\gamma_+-\gamma)^{-1}]\in(0,\infty]$. If $\beta m_+>1$, the largest eigenvalue of $\Sigma_d$ converges in probability to the unique $t>\gamma_+$ satisfying
\begin{equation}
    1=\beta\E_\gamma\frac1{t-\gamma}.
    \label{eq:sep-population-outlier}
\end{equation}
If $\beta m_+\leq1$, the largest eigenvalue instead converges in probability to $\gamma_+$.
\end{lemma}
\begin{proof}
Write $t_d=\lambda_{\max}(\Sigma_d)$. Almost surely, $t_d>\max_i\gamma_i^\star$: the Rayleigh quotient at a coordinate maximizing $\gamma_i^\star$ is strictly larger than that maximum, since every Gaussian teacher coordinate is nonzero almost surely. For a corresponding eigenvector $\bv\ne0$,
\[
 (t_d\id-\Gamma^\star)\bv =\beta_d\bm e(\bm e^\top\bv).
\]
The matrix on the left is invertible, and $\bm e^\top\bv\ne0$; otherwise $\bv=0$. Multiplying by $\bm e^\top(t_d\id-\Gamma^\star)^{-1}$ and cancelling gives the exact finite-dimensional equation
\[
 1=\beta_d\sum_{i=1}^d\frac{e_i^2}{t_d-\gamma_i^\star} =\frac\beta d\sum_{i=1}^d \frac{(u_i^\star)^2}{t_d-\gamma_i^\star}.
\]
For this proof, define the random function
\[
 F_d(t)=\frac\beta d\sum_{i=1}^d \frac{(u_i^\star)^2}{t-\gamma_i^\star}, \qquad t>\max_i\gamma_i^\star,
\]
and its deterministic limit
\[
 F(t)=\beta\E_\gamma\frac1{t-\gamma}, \qquad t>\gamma_+.
\]
Thus $F_d(t_d)=1$, and $F_d$ is almost surely strictly decreasing on its domain. For each fixed deterministic $t>\gamma_+$, independence and the law of large numbers yield
\[
 F_d(t)\xrightarrow{\mathbb P}F(t).
\]
Indeed, the summands are i.i.d. and bounded above by $(u_i^\star)^2/(t-\gamma_+)$, which is integrable, and $\E[(u_i^\star)^2]=1$. The function $F$ is continuous and strictly decreasing on $(\gamma_+,\infty)$, with limits $\beta m_+$ at the left endpoint and zero at infinity. Thus it equals one at a unique point in this interval exactly when $\beta m_+>1$.

It remains to pass from convergence at deterministic arguments to convergence of the random solution $t_d$ of $F_d(t_d)=1$; the pointwise law of large numbers alone does not justify substituting the random argument $t_d$. Monotonicity allows us to locate this solution by evaluating $F_d$ at deterministic arguments on either side of its proposed limit.

Suppose first that $\beta m_+>1$, and denote the unique solution of $F(t)=1$ by $t_\star>\gamma_+$. Fix $0<\varepsilon<t_\star-\gamma_+$ and set $t_-=t_\star-\varepsilon$ and $t_+=t_\star+\varepsilon$. These two deterministic arguments lie above $\gamma_+$ and satisfy
\[
 F(t_-)>1>F(t_+).
\]
Convergence in probability at each of them implies $\Pr(F_d(t_-)\leq1)\to0$ and $\Pr(F_d(t_+)\geq1)\to0$. Whenever $F_d(t_-)>1>F_d(t_+)$, strict monotonicity and $F_d(t_d)=1$ force $t_-<t_d<t_+$. Consequently, by a union bound,
\[
 \Pr\bigl(|t_d-t_\star|\geq\varepsilon\bigr) \leq \Pr\bigl(F_d(t_-)\leq1\bigr) +\Pr\bigl(F_d(t_+)\geq1\bigr) \longrightarrow0.
\]
Since this holds for every sufficiently small $\varepsilon>0$, it proves $t_d\xrightarrow{\mathbb P}t_\star$.

If $\beta m_+\leq1$, then $F(\gamma_++\varepsilon)<1$ for every $\varepsilon>0$, including when $\beta m_+=1$, by strict monotonicity. The law of large numbers and the same monotonicity argument give
\[
 \Pr(t_d\geq\gamma_++\varepsilon) \leq\Pr\bigl(F_d(\gamma_++\varepsilon)\geq1\bigr) \longrightarrow0.
\]
For the lower bound, $t_d\geq\max_i\gamma_i^\star$, and this maximum converges in probability to $\gamma_+$ because the variances are i.i.d. with essential upper endpoint $\gamma_+$. Combining these upper and lower bounds proves $t_d\xrightarrow{\mathbb P}\gamma_+$.
\end{proof}

We next provide the eigenvector argument needed when this condition fails.

\begin{proposition}[Compact variance laws with a regular upper sample edge]
\label[proposition]{prop:sep-compact-pca}
Suppose $\E_\gamma\gamma=1$, $\operatorname{supp}(P_\gamma)\subset[a,\gamma_+]$, $\gamma_+$ is the upper support endpoint, and
\begin{equation}
    \beta\E_\gamma\frac1{\gamma_+-\gamma}<1,\qquad \E_\gamma\frac{\gamma^2}{(\gamma_+-\gamma)^2}=\infty.
    \label{eq:sep-compact-conditions}
\end{equation}
Then PCA has zero asymptotic squared overlap with $\bm e$ for every fixed $\alpha\in(0,\infty)$.
\end{proposition}
\begin{proof}
The companion Stieltjes transform of $NN^\top$ satisfies the Silverstein--Pastur equation \citep{SilversteinChoi1995}, with convention $\underline m(z)=\int(\lambda-z)^{-1}\underline F(\dd\lambda)$:
\begin{equation}
    \underline m(z)=-\left(z-\frac1\alpha \E_\gamma\frac{\gamma}{1+\gamma\underline m(z)}\right)^{-1}.
\end{equation}
On the upper outside-support branch, setting $t=-1/\underline m(z)>\gamma_+$ gives
\begin{equation}
    z=\psi_\alpha(t):=t\left(1+\frac1\alpha \E_\gamma\frac{\gamma}{t-\gamma}\right), \quad \psi'_\alpha(t)=1-\frac1\alpha \E_\gamma\frac{\gamma^2}{(t-\gamma)^2}.
    \label{eq:sep-pastur-map}
\end{equation}
The second condition in \eqref{eq:sep-compact-conditions} ensures a unique $t_e>\gamma_+$ with $\psi'_\alpha(t_e)=0$ for every finite $\alpha>0$. Furthermore $\psi''_\alpha(t_e)=2\alpha^{-1}\E_\gamma[\gamma^2/(t_e-\gamma)^3]>0$. The support characterization of \citet{SilversteinChoi1995} then gives a regular upper covariance edge $b_\alpha^2=\psi_\alpha(t_e)$ with a square-root density. The empirical noise spectrum converges to this law, and $s_1(N)\to b_\alpha$: the bounded population variances have no outliers beyond their support, so the no-eigenvalues-outside-the-support theorem applies \citep{BaiSilverstein1998}. A bounded Gaussian matrix norm controls the remaining unbounded part of the spectral axis.

For a singular-value argument $s>b_\alpha$, let $g_d(s)$ and $g_n(s)$ be the limits of the normalized traces of $s(s^2I_d-N^\top N)^{-1}$ and $s(s^2I_n-NN^\top)^{-1}$, respectively. Their product is the rectangular $D$-transform in \citet{benaych2012singular}; using both traces automatically includes the zero singular values on the larger side. At $s^2=\psi_\alpha(t)$,
\begin{equation}
    g_d(s)=\frac{t}{s}\E_\gamma\frac1{t-\gamma},\qquad g_n(s)=\frac{s}{t},\qquad D_\alpha(s)=g_d(s)g_n(s)=\E_\gamma\frac1{t-\gamma}.
    \label{eq:sep-rectangular-transform}
\end{equation}
In particular,
\begin{equation}
    \beta D_\alpha(b_\alpha^+) =\beta\E_\gamma\frac1{t_e-\gamma} <\beta m_+<1.
\end{equation}
The perturbation in \eqref{eq:sep-data-matrix} is exactly the independent Gaussian rank-one i.i.d. spike model of \citet{benaych2012singular}, with strength $\sqrt\beta$. Its singular-vector transition theorem gives vanishing projections onto the spike vectors below the threshold, provided the singular resolvent derivative diverges at the edge. The square-root density just verified implies this derivative condition. Both left and right projections vanish, so the conclusion also holds when $\alpha>1$ and the matrix is transposed to place the smaller dimension first. Normalizing $\bu^\star/\sqrt d$ to $\bm e$ changes the overlap by a factor converging to one. This proves the eigenvector conclusion, rather than merely absence of a separated eigenvalue.
\end{proof}

For reference, if \eqref{eq:sep-population-outlier} has a root $t>\gamma_+$, define $J_j=\E_\gamma[\gamma^j/(t-\gamma)^2]$ and $B=\E_\gamma[\gamma/(t-\gamma)]$. On the regular separated sample branch, \eqref{eq:sep-pastur-map} requires $\alpha>J_2$, and the same singular-vector formula gives
\begin{equation}
    \alpha_c^{\rm PCA}=J_2,\qquad \theta_{\bu^\star,\rm PCA}^2 =\frac{1-J_2/\alpha}{\beta^2J_0(1+B/\alpha)}, \qquad\alpha>J_2.
    \label{eq:sep-general-pca-overlap}
\end{equation}
Explicitly, that formula is $-2g_d(s)/(\beta D_\alpha'(s))$ at the outlier. Using $D_\alpha'(s)=-2sJ_0/\psi'_\alpha(t)$ and $g_d(s)=t/(\beta s)$ gives \eqref{eq:sep-general-pca-overlap}. Substitution of $P_\gamma=\delta_1$, $t=1+\beta$ recovers \eqref{eq:sep-homogeneous-pca}.

\subsubsection{Exact mask average and radial reduction}
\label{subsec:sep-exact-mask}
For the masking proof assume only that $P_\gamma$ is supported in $[a,\gamma_+]\subset(0,\infty)$ and has mean one. Fix $\rho\in(0,1)$. For a generic mask, suppress the sample and view indices and write $\chi_i\iid\operatorname{Bernoulli}(\rho)$, with $\chi_i=1$ marking a hidden coordinate. Define $\br_{\vis}=((1-\chi_i)r_i)_i$, $\br_{\mask}=(\chi_i r_i)_i$, and $m_{\mask}=\|\br_{\mask}\|^2$. The hidden predictions from the linear encoder are $Q\br_{\mask}(\br_{\vis}^\top\bx)$. Expanding the squared loss and averaging over masks gives exactly
\begin{equation}
\begin{aligned}
    \widehat R_{\rm dyn}(Q,\br) &=\widehat H(\br)Q^2-2\widehat{\mathcal C}(\br)Q,\\
    \widehat H(\br)&=\E_\chi[m_{\mask}\br_{\vis}^\top\widehat\Sigma\br_{\vis}],\qquad \widehat{\mathcal C}(\br)=\E_\chi[\br_{\vis}^\top\widehat\Sigma\br_{\mask}],
\end{aligned}
    \label{eq:sep-exact-radial}
\end{equation}
after subtracting the mask-averaged target-only constant. For $\widehat H>0$, optimizing $Q\geq0$ yields
\begin{equation}
    Q_*(\br)=\frac{[\widehat{\mathcal C}(\br)]_+}{\widehat H(\br)},\qquad \inf_{Q\geq0}\widehat R_{\rm dyn}(Q,\br) =-\widehat S(\br),\qquad \widehat S(\br)=\frac{[\widehat{\mathcal C}(\br)]_+^2}{\widehat H(\br)}.
    \label{eq:sep-empirical-score}
\end{equation}
Coordinate directions have $\widehat H=\widehat{\mathcal C}=0$ and score zero. On the positive-definite covariance event used below, these are the only directions with zero denominator. No replacement of a mask mass by its mean is used in this reduction.

\subsubsection{Population score gap, including localized directions}
\label{subsec:sep-population-gap}
Let $H,\mathcal C,S$ be the conditional population versions at fixed $\bu^\star,\Gamma^\star$, obtained by replacing $\widehat\Sigma$ with $\Sigma_d$. Set
\begin{equation}
    c=\bm e^\top\br,\quad I=\sum_i r_i^4,\quad J=\sum_i r_i^2e_i^2,\quad \delta_d=\max_i e_i^2.
\end{equation}
The disjoint visible and hidden sets make $\br_{\vis}^\top\Gamma^\star\br_{\mask}=0$ exactly. Thus
\begin{equation}
    \mathcal C(\br)=\rho(1-\rho)\beta_d(c^2-J),\qquad H(\br)\geq a\E_\chi[m_{\mask}(1-m_{\mask})]=a\rho(1-\rho)(1-I).
    \label{eq:sep-population-coefficients}
\end{equation}
The first identity isolates the cross-coordinate signal; the noise remains in the positive quartic coefficient $H$.

\begin{lemma}[Control of localization]
\label[lemma]{lem:sep-localization}
For unit vectors $\br,\bm e$, with $c,I,J$ as above and $\delta=\max_i e_i^2$,
\begin{equation}
    \frac{[c^2-J]_+^2}{1-I}\leq8(c^2+\delta),
    \label{eq:sep-localization}
\end{equation}
where the left side is defined to be zero at coordinate directions.
\end{lemma}
\begin{proof}
Only $D_1=c^2-J>0$ needs consideration. If $m=\max_i r_i^2\leq1/2$, then $I\leq1/2$ and $D_1^2/(1-I)\leq2c^4\leq2c^2$. Otherwise let $r_k^2=m>1/2$ and $R=1-m$. Then $1-I\geq2mR\geq R$. Writing $z=r_ke_k$, $y=\sum_{i\ne k}r_ie_i$ gives $D_1=2zy+y^2-\sum_{i\ne k}r_i^2e_i^2$ and $D_1\leq2\sqrt{\delta R}+R$. If $R\geq c^2$, use $D_1\leq c^2$ to obtain $D_1^2/(1-I)\leq c^2$. If $R<c^2$, use $D_1^2/(1-I)\leq(2\sqrt\delta+\sqrt R)^2 \leq8\delta+2R\leq8(\delta+c^2)$. At $R=0$ the direction is a coordinate vector and $D_1=0$.
\end{proof}

Combining \eqref{eq:sep-population-coefficients} and \eqref{eq:sep-localization} gives the uniform estimate
\begin{equation}
    S(\br)\leq\frac{8\beta_d^2\rho(1-\rho)}{a}(c^2+\delta_d).
    \label{eq:sep-small-overlap-score}
\end{equation}
For the Gaussian teacher, $\beta_d\to\beta$, $\delta_d\to0$, and $\bm e^\top\Gamma^\star\bm e\to1$ in probability. The teacher score is bounded away from zero: its cross coefficient tends to $\rho(1-\rho)\beta$, and its quartic coefficient has noise part
\begin{equation}
    \rho(1-\rho)\left(\bm e^\top\Gamma^\star\bm e -\sum_i\gamma_i^\star e_i^4\right)\to\rho(1-\rho).
\end{equation}
For the signal part, $\bm e_{\vis}^\top\bm e=1-m_{\mask}$ and $\operatorname{Var}_\chi(m_{\mask})=\rho(1-\rho)\sum_i e_i^4\to0$, so that $\beta_d\E_\chi[m_{\mask}(1-m_{\mask})^2]\to\beta\rho(1-\rho)^2$. Consequently
\begin{equation}
    S(\bm e)\longrightarrow g:=\frac{\rho(1-\rho)\beta^2}{1+\beta(1-\rho)}>0.
    \label{eq:sep-teacher-score}
\end{equation}

\subsubsection{Uniform empirical comparison and an explicit sample bound}
\label{subsec:sep-finite-sample}
Define the relative covariance error
\begin{equation}
    \varepsilon_d =\|\Sigma_d^{-1/2}\widehat\Sigma\Sigma_d^{-1/2}-\id\|_{\mathrm{op}}.
    \label{eq:sep-relative-error}
\end{equation}
Conditionally on the teacher and variances, the whitened covariance is standard Gaussian Wishart. Its extreme eigenvalue limits \citep{BaiYin1993} give, for $\alpha>1$,
\begin{equation}
    \varepsilon_d\longrightarrow \varepsilon_\alpha=2/\sqrt\alpha+1/\alpha \quad\text{in probability}.
    \label{eq:sep-wishart-error}
\end{equation}

On $\varepsilon_d<1$, the relative quadratic-form bound immediately gives
\begin{equation}
    (1-\varepsilon_d)H(\br)\leq\widehat H(\br) \leq(1+\varepsilon_d)H(\br).
\end{equation}
For $\Delta_\Sigma=\widehat\Sigma-\Sigma_d$, the corresponding bilinear inequality is
\begin{equation}
    |\bm x^\top\Delta_\Sigma\bm y| \leq\varepsilon_d \sqrt{\bm x^\top\Sigma_d\bm x} \sqrt{\bm y^\top\Sigma_d\bm y}.
\end{equation}
Apply it to $\br_{\vis},\br_{\mask}$. For $m_{\mask}>0$, rewrite the product as $\sqrt{m_{\mask}\br_{\vis}^\top\Sigma_d\br_{\vis}} \sqrt{\br_{\mask}^\top\Sigma_d\br_{\mask}/m_{\mask}}$. The second factor is at most $\sqrt{\|\Sigma_d\|_{\mathrm{op}}}$; when $m_{\mask}=0$ the bilinear form is zero. Cauchy--Schwarz over masks therefore yields
\begin{equation}
    |\widehat{\mathcal C}(\br)-\mathcal C(\br)| \leq\varepsilon_d\sqrt{H(\br)\|\Sigma_d\|_{\mathrm{op}}}.
\end{equation}
These estimates hold simultaneously for all directions and imply
\begin{align}
    \sqrt{\widehat S(\br)} &\leq\frac{\sqrt{S(\br)}+\varepsilon_d\sqrt{\|\Sigma_d\|_{\mathrm{op}}}} {\sqrt{1-\varepsilon_d}},\nonumber\\
    \sqrt{\widehat S(\br)} &\geq\frac{\sqrt{S(\br)}-\varepsilon_d\sqrt{\|\Sigma_d\|_{\mathrm{op}}}} {\sqrt{1+\varepsilon_d}} \quad\text{if }\mathcal C(\br)>0\text{ and the numerator is positive}.
    \label{eq:sep-uniform-score-comparison}
\end{align}

\begin{theorem}[Global dynamic linear recovery at finite sample ratio]
\label{thm:sep-dynamic-global}
For every mean-one law supported in $[a,\gamma_+]\subset(0,\infty)$, fixed $\beta>0$, and $\rho\in(0,1)$, define
\begin{equation}
\begin{aligned}
    g&=\frac{\rho(1-\rho)\beta^2}{1+\beta(1-\rho)}, &\eta&=\frac{a}{512[1+\beta(1-\rho)]},\\
    \varepsilon_*&=\min\left\{\frac14, \frac{\sqrt g}{16\sqrt{\gamma_++\beta+1}}\right\}, &A_{\rm suff}&=\frac{1}{(\sqrt{1+\varepsilon_*}-1)^2}.
\end{aligned}
\label{eq:sep-explicit-bound}
\end{equation}
If $n/d\to\alpha>A_{\rm suff}$, then with probability tending to one the exact ridgeless dynamic linear objective has a global minimizer, and every global minimizer has $|\bm e^\top\widehat{\br}|^2\geq\eta$.
\end{theorem}
\begin{proof}
Put $C_0=8\beta^2\rho(1-\rho)/a$, so that $\eta=g/(64C_0)$. The population estimates imply, with probability tending to one,
\begin{equation}
    \sup_{c^2\leq\eta}S(\br)\leq g/16,\qquad S(\bm e)\geq3g/4,\qquad \|\Sigma_d\|_{\mathrm{op}}\leq \gamma_++\beta+1.
\end{equation}
Solving $2/\sqrt\alpha+1/\alpha<\varepsilon_*$ gives exactly $\alpha>A_{\rm suff}$. By \eqref{eq:sep-wishart-error}, on an event of probability tending to one, $\varepsilon_d<\varepsilon_*$ as well. The uniform comparisons then yield
\begin{equation}
    \sup_{c^2\leq\eta}\sqrt{\widehat S(\br)} \leq\frac{(1/4+1/16)\sqrt g}{\sqrt{3/4}} <\frac{(\sqrt{3/4}-1/16)\sqrt g}{\sqrt{5/4}} \leq\sqrt{\widehat S(\bm e)}.
    \label{eq:sep-empirical-gap}
\end{equation}
Thus no maximizing direction can have squared overlap at most $\eta$.

To verify attainment, for large $d$ every coordinate direction has squared teacher overlap at most $\delta_d<\eta$. On the compact set $\{\|\br\|=1:c^2\geq\eta\}$ the positive lower bound for $H$ and the relative comparison ensure $\widehat H>0$. The score is continuous there and attains its maximum. The strict gap \eqref{eq:sep-empirical-gap} makes this a maximum over all directions. Its score is positive, so its radial optimum in \eqref{eq:sep-empirical-score} is finite and positive. The zero vector and coordinate directions cannot be global minimizers. The radial identity proves the conclusion for every global weight minimizer.
\end{proof}

This proof uses only a population score gap and relative covariance control. It does not require the exact population maximizing direction, uniqueness of a resolvent equation for that direction, or a uniform delocalization assumption on candidate empirical weights. The factor $\|\Sigma_d\|_{\mathrm{op}}$ in \eqref{eq:sep-uniform-score-comparison} explains why this sufficient bound does not extend directly to unbounded variances. \Cref{app:dynamic-unbounded} removes that factor by normalizing the off-diagonal error with the actual masking penalty, and proves recovery for variance laws bounded away from zero with finite second moment.

\begin{proof}[Proof of \cref{thm:main-separation}]
Under \ref{as:model} and \ref{as:noise}, \cref{thm:unbounded-dynamic-recovery} applies whether or not the variance support is bounded. It gives existence of global minimizers and the positive-overlap guarantee with constants depending only on $a,\beta,\rho$, as stated in the main theorem. For bounded support, \cref{thm:sep-dynamic-global} also supplies the alternative sufficient constants in \eqref{eq:sep-explicit-bound}.

If the variance support is unbounded, \cref{prop:sep-unbounded-pca} gives zero PCA overlap for every fixed $\alpha>0$. If instead the bounded-noise conditions \ref{as:pca-fail} hold, \cref{prop:sep-compact-pca} gives the same conclusion. The exact identity \eqref{eq:sep-linear-pca-identity} transfers this failure to full linear reconstruction. Combining the applicable PCA obstruction with masked recovery proves the separation on the same Gaussian data model and variance law.
\end{proof}

\subsubsection{Examples within the bounded-support class}
\label{subsec:sep-example}
Choose $P_\gamma$ from the family
\begin{equation}
\begin{aligned}
    P_\gamma(\dd\gamma) &=\frac{r+1}{(\gamma_+-a)^{r+1}}(\gamma_+-\gamma)^r \one_{\{a<\gamma<\gamma_+\}}\,\dd\gamma,\\
    0<a<1,\quad 0<r\leq1,&\qquad \gamma_+=(r+2)-(r+1)a.
\end{aligned}
    \label{eq:sep-family}
\end{equation}
Direct integration gives $\E_\gamma\gamma=a+(\gamma_+-a)/(r+2)=1$ and
\begin{equation}
    m_+=\frac{r+1}{r(\gamma_+-a)},\qquad \E_\gamma\frac{\gamma^2}{(\gamma_+-\gamma)^2}=\infty.
\end{equation}
The latter follows from the endpoint power $(\gamma_+-\gamma)^{r-2}$. Consequently, if
\begin{equation}
    0<\beta<\frac{r(r+2)(1-a)}{r+1},
    \label{eq:sep-family-signal}
\end{equation}
\cref{prop:sep-compact-pca} proves PCA failure at every fixed sample ratio, while \cref{thm:sep-dynamic-global} proves dynamic masking recovery for every fixed $\rho\in(0,1)$ at sufficiently large finite sample ratio. This is a class result; no particular density is needed for the dynamic theorem itself.

For a concrete numerical illustration, take $a=1/4$, $r=1$, $\gamma_+=5/2$, and $\beta=1$. The resulting law is
\begin{equation}
    P_\gamma(\dd\gamma)=\frac{32}{81}\left(\frac52-\gamma\right) \one_{\{1/4<\gamma<5/2\}}\,\dd\gamma.
    \label{eq:sep-concrete-law}
\end{equation}
This variance law is the bounded-noise example in Figure~\ref{fig:validation}(b); the general theorem does not depend on this choice. Its moments are
\begin{equation}
    \E_\gamma\gamma=1,\qquad \kappa_2=41/32,\qquad \kappa_3=617/320,\qquad m_+=8/9.
\end{equation}
The strict signal condition is $1<9/8$, so this law is one instance of \cref{thm:main-separation}. For an entirely explicit sufficient choice, take $\rho=1/2$. Then $g=1/6$, $\eta=1/3072$, and $\varepsilon_*=1/(48\sqrt3)$. Every $\alpha\geq30000$ satisfies
\begin{equation}
    \varepsilon_\alpha\leq\frac1{50\sqrt3}+\frac1{30000} <\frac1{48\sqrt3}=\varepsilon_*.
\end{equation}
These conservative constants certify a finite proportional sample size.

The failure criterion is not universal over heterogeneous laws. For example, an atom at the upper endpoint makes $m_+=\infty$, so a mean-one two-point law has a population outlier for every $\beta>0$ and PCA eventually recovers as $\alpha$ increases. The separation concerns ordinary PCA and the specified full-reconstruction architecture, rather than spectral estimators designed to correct heterogeneous noise.

%% file: appendices/appendix_dynamic_unbounded.tex
\subsection{Rigorous dynamic linear recovery with unbounded coordinate noise}
\label{app:dynamic-unbounded}

The upper bound on the variance support in \cref{thm:sep-dynamic-global} can be removed. The essential improvement is to control the off-diagonal sampling error relative to the noise penalty already present in the exact masking objective. This avoids the factor $\|\Sigma_d\|_{\mathrm{op}}$ in \eqref{eq:sep-uniform-score-comparison}. All results in this subsection concern the original tied linear architecture and the exact mask-averaged empirical objective; no clipping, reweighting, or restriction on the candidate weights is introduced.

\subsubsection{Statement and assumptions}
\label{subsec:unbounded-statement}
Use the Gaussian model and notation of \eqref{eq:sep-model}. Suppose
\begin{equation}
    \gamma_i^\star\iid P_\gamma,\qquad P_\gamma([a,\infty))=1\quad\text{for some }a>0,\qquad \E_\gamma\gamma=1,\qquad \E_\gamma\gamma^2<\infty.
    \label{eq:unbounded-assumptions}
\end{equation}
The variance law is fixed as $d$ grows, is independent of the Gaussian teacher and sample variables, and may have unbounded support. Fix $\beta>0$, $\rho\in(0,1)$, and $\lambda_R=0$.

\begin{theorem}[Recovery without an upper variance bound]
\label{thm:unbounded-dynamic-recovery}
Under \eqref{eq:unbounded-assumptions}, there is a finite deterministic $A(a,\beta,\rho)$ such that, whenever $n/d\to\alpha>A(a,\beta,\rho)$, with probability tending to one the exact dynamic linear objective has a global minimizer and every global minimizer satisfies
\begin{equation}
    \theta_{\bu^\star}^2\geq \eta:=\frac{a}{512[1+\beta(1-\rho)]}>0.
    \label{eq:unbounded-overlap-guarantee}
\end{equation}
If $P_\gamma$ also has unbounded support, full linear reconstruction and PCA have squared overlap converging to zero for every fixed $\alpha>0$ and $\beta>0$.
\end{theorem}

The second assertion follows from \cref{prop:sep-unbounded-pca}. The first is proved below. The sufficient sample ratio is non-sharp and the result does not assert convergence of an optimization algorithm. Neither the law nor its support needs to change with $\alpha$, $\beta$, or $\rho$. The constants in the asymptotic sufficient bound below do not depend on an upper support endpoint; the dimension required to reach its high-probability regime can depend on the variance law.

\subsubsection{A variance-sensitive lower bound on the score denominator}
\label{subsec:unbounded-denominator}
Recall the exact radial reduction in \eqref{eq:sep-exact-radial}--\eqref{eq:sep-empirical-score}: writing $\bw=\sqrt{dQ}\,\br$ with $\|\br\|=1$ and minimizing over $Q\geq0$ reduces the recovery problem to maximizing the directional score
\[
    \widehat S(\br) =\frac{[\widehat{\mathcal C}(\br)]_+^2}{\widehat H(\br)}.
\]
Here $\widehat{\mathcal C}$ and $\widehat H$ are the empirical coefficients defined in \eqref{eq:sep-exact-radial}; their conditional population counterparts $\mathcal C,H$ are obtained by replacing $\widehat\Sigma$ by $\Sigma_d$, with the teacher and variances held fixed. The score is defined to be zero at coordinate directions. These exact identities do not require an upper bound on the coordinate variances.

To transfer the population score gap to the empirical score, we will control $\widehat H$ relative to $H$ and the numerator error $|\widehat{\mathcal C}-\mathcal C|/\sqrt H$; the square root appears because we compare $\sqrt{\widehat S}$ with $\sqrt S$. The purpose of this subsection is to retain a lower bound on $H$ that grows with the variances in the candidate direction. This will allow large sampling errors to be controlled relative to the score denominator.

Write $\operatorname{off}(M)=M-\operatorname{diag}(M)$. For a unit direction $\br$, let
\begin{equation}
    D_\Gamma(\br):=\sum_i\gamma_i^\star r_i^2(1-r_i^2),\qquad I(\br):=\sum_i r_i^4.
    \label{eq:unbounded-denominator}
\end{equation}
This noise penalty satisfies $D_\Gamma\geq a(1-I)$ and is positive except at coordinate directions. The exact mask identities give
\begin{equation}
    \widehat{\mathcal C}(\br)=\rho(1-\rho)\,\br^\top \operatorname{off}(\widehat\Sigma)\br, \qquad H(\br)\geq\rho(1-\rho) D_\Gamma(\br).
    \label{eq:unbounded-mask-identities}
\end{equation}
The second inequality retains the actual coordinate variances rather than replacing them by their common lower bound. Its noise contribution is exactly $\rho(1-\rho)D_\Gamma$; the signal contribution to $H$ is nonnegative. Thus $D_\Gamma$ is a lower bound on the population score denominator, up to the factor $\rho(1-\rho)$, rather than the denominator itself. It suffices to bound the off-diagonal covariance error after division by $\sqrt{D_\Gamma(\br)}$. The next subsection establishes this bound for the noise contribution; \eqref{eq:unbounded-cross-comparison} then gives the required control of $|\widehat{\mathcal C}-\mathcal C|/\sqrt H$ for the full covariance.

\subsubsection{Uniform control of the noise sampling error}
\label{subsec:unbounded-noise-control}
We first study the noise-noise part of the off-diagonal covariance error. Use the $n\times d$ matrix $Z=(z_{\mu i})$ from \cref{app:notation} and set
\begin{equation}
    E_d:=\operatorname{off}\!\left( (\Gamma^\star)^{1/2}Z^\top Z(\Gamma^\star)^{1/2}/n\right), \qquad \gamma_{+,d}=\max_i\gamma_i^\star,\qquad s_d=\frac{\operatorname{tr}\Gamma^\star}{n}.
    \label{eq:unbounded-noise-error}
\end{equation}
Here $\gamma_{+,d}$ is the finite-sample maximum, distinct from the deterministic support endpoint $\gamma_+$, which may be infinite. The population off-diagonal noise covariance is zero. Finite second moment and the law of large numbers imply
\begin{equation}
    s_d\to\alpha^{-1},\qquad \operatorname{tr}[(\Gamma^\star)^2]=O_{\mathbb P}(d),\qquad \gamma_{+,d}=o_{\mathbb P}(\sqrt d),\qquad \frac{\gamma_{+,d}\log d}{n}\to0.
    \label{eq:unbounded-tail-control}
\end{equation}
For the third assertion, for each fixed $\varepsilon>0$, $\Pr(\gamma_{+,d}>\varepsilon\sqrt d)\leq d\Pr(\gamma>\varepsilon\sqrt d)\to0$, since $t^2\Pr(\gamma>t)\to0$ when $\E_\gamma\gamma^2<\infty$.

\begin{lemma}[Noise error normalized by the masking penalty]
\label[lemma]{lem:unbounded-normalized-noise}
There is an absolute finite constant $C_1$ such that, for each fixed $\alpha>1$,
\begin{equation}
    \sup_{\|\br\|=1,\ D_\Gamma(\br)>0} \frac{|\br^\top E_d\br|}{\sqrt{D_\Gamma(\br)}} \leq C_1\left(\alpha^{-1/2} +\frac{\alpha^{-1}}{\sqrt a}\right) +o_{\mathbb P}(1).
    \label{eq:unbounded-normalized-noise}
\end{equation}
The supremum includes data-dependent and localized directions.
\end{lemma}

\begin{proof}
We first establish a bound using $\br^\top\Gamma^\star\br$ and then treat directions near a coordinate axis separately.

\paragraph{A uniform quadratic-form bound.} For $t>0$, define the diagonal matrices
\begin{equation}
    P_t=(\Gamma^\star+t\id)^{-1/2},\qquad A_t=P_t\Gamma^\star P_t, \qquad \|A_t\|_{\mathrm{op}}\leq1,\quad \operatorname{tr}A_t\leq\operatorname{tr}\Gamma^\star/t.
\end{equation}
Fix $t>0$ and condition on $\Gamma^\star$. For $\bz_\mu=(z_{\mu1},\ldots,z_{\mu d})^\top$, the transformed samples $\bm y_\mu^{(t)}=P_t(\Gamma^\star)^{1/2}\bz_\mu$ are conditionally i.i.d. $\N(0,A_t)$. Their sample covariance $\widehat A_t=n^{-1}\sum_\mu\bm y_\mu^{(t)}(\bm y_\mu^{(t)})^\top$ therefore satisfies
\begin{equation}
    \widehat A_t-A_t =P_t\left((\Gamma^\star)^{1/2}\frac{Z^\top Z}{n} (\Gamma^\star)^{1/2}-\Gamma^\star\right)P_t.
    \label{eq:unbounded-preconditioned-error}
\end{equation}
By Jensen's inequality, $\bigl(\E[\|\bm y_\mu^{(t)}\|\mid\Gamma^\star]\bigr)^2 \leq\E[\|\bm y_\mu^{(t)}\|^2\mid\Gamma^\star]=\operatorname{tr}A_t$. Thus the effective rank in Corollary~2 of \citet{koltchinskii2017concentration} is at most $\operatorname{tr}(A_t)/\|A_t\|_{\mathrm{op}}$. Applying that corollary and bounding its maximum by a sum gives, for every confidence parameter $\omega\geq1$, with conditional probability at least $1-e^{-\omega}$,
\begin{equation}
\begin{aligned}
 \|\widehat A_t-A_t\|_{\mathrm{op}}\leq C\Big[ &\sqrt{\|A_t\|_{\mathrm{op}}\operatorname{tr}(A_t)/n} +\operatorname{tr}(A_t)/n \\ &+\|A_t\|_{\mathrm{op}}\sqrt{\omega/n}
 +\|A_t\|_{\mathrm{op}}\omega/n\Big].
\end{aligned}
\end{equation}
Here $C$ is universal. Using the bounds on $\|A_t\|_{\mathrm{op}}$ and $\operatorname{tr}A_t$ above, together with \eqref{eq:unbounded-preconditioned-error}, yields
\begin{equation}
\begin{aligned}
 &\left\|P_t\left((\Gamma^\star)^{1/2}Z^\top Z(\Gamma^\star)^{1/2}/n -\Gamma^\star\right)P_t\right\|_{\mathrm{op}}\\
 &\hspace{3em}\leq C\left[\sqrt{s_d/t}+s_d/t +\sqrt{\omega/n}+\omega/n\right].
\end{aligned}
    \label{eq:unbounded-preconditioned-covariance}
\end{equation}
For a symmetric matrix $M$, $\|\operatorname{off}(M)\|_{\mathrm{op}}\leq2\|M\|_{\mathrm{op}}$. Because $P_t$ and $\Gamma^\star$ are diagonal, $\operatorname{off}(\widehat A_t-A_t)=P_tE_dP_t$, so the same bound with $2C$ controls $\|P_tE_dP_t\|_{\mathrm{op}}$.

To choose the preconditioning scale according to a possibly data-dependent direction, we first make the bound simultaneous over a finite grid. Use the grid $t_j=a2^j$, from $a$ up to its first value at least $\gamma_{+,d}$, and take $\omega=4\log d$. On $\gamma_{+,d}\leq d$, which has probability tending to one, this grid has $O(\log d)$ points. The grid depends only on $\Gamma^\star$, which is fixed under the conditioning; it does not depend on $Z$ or on the candidate direction. Its conditional failure probability is at most $O(\log d)d^{-4}$. A conditional union bound makes \eqref{eq:unbounded-preconditioned-covariance} simultaneous on the grid with probability tending to one. Averaging over $\Gamma^\star$ and adding $\Pr(\gamma_{+,d}>d)=o(1)$ gives the same unconditional conclusion. On this simultaneous event, for any unit $\br$, put $g_{\br}=\br^\top\Gamma^\star\br\in[a,\gamma_{+,d}]$ and choose a grid point $g_{\br}\leq t_j\leq2g_{\br}$. Then $\br^\top(\Gamma^\star+t_j\id)\br\leq3g_{\br}$, and therefore
\begin{equation}
    \frac{|\br^\top E_d\br|}{\sqrt{g_{\br}}} \leq6C\left[ \sqrt{s_d}+\frac{s_d}{\sqrt a} +\sqrt{\gamma_{+,d}}\left(\sqrt{\omega/n}+\omega/n\right)\right].
    \label{eq:unbounded-uniform-g}
\end{equation}
By \eqref{eq:unbounded-tail-control}, the last term vanishes in probability. Consequently, with
\begin{equation}
    T_d:=\sup_{\|\br\|=1} \frac{|\br^\top E_d\br|}{\sqrt{\br^\top\Gamma^\star\br}},
\end{equation}
we have $T_d\leq6C(\alpha^{-1/2}+a^{-1/2}\alpha^{-1})+o_{\mathbb P}(1)$. By homogeneity, every vector $\bv$ obeys
\begin{equation}
    |\bv^\top E_d\bv|\leq T_d\|\bv\| \sqrt{\bv^\top\Gamma^\star\bv}.
    \label{eq:unbounded-homogeneous-g}
\end{equation}

\paragraph{A simultaneous row bound.} Let $\bm f_i$ denote coordinate vector $i$ and let $\bz_i=Z_{:,i}\in\R^n$ be column $i$ of $Z$. Conditional on $\Gamma^\star,\bz_i$, the variables $\bz_i^\top\bz_j/\|\bz_i\|$, $j\ne i$, are independent standard Gaussians. Hence
\begin{equation}
    \frac{\|E_d\bm f_i\|^2}{\gamma_i^\star} \ \overset{\mathrm{law}}{=}\ \frac{\|\bz_i\|^2}{n}\, \frac1n\sum_{j\ne i}\gamma_j^\star\xi_j^2,
    \label{eq:unbounded-row-law}
\end{equation}
where the equality in law is conditional on $\Gamma^\star,\bz_i$. For independent standard Gaussians $\xi_j$, their moment-generating function gives
\begin{equation}
 \Pr\!\left\{ \sum_j\gamma_j^\star(\xi_j^2-1) >2\sqrt{\omega\sum_j(\gamma_j^\star)^2}+2\gamma_{+,d}\omega \,\middle|\,\Gamma^\star\right\}\leq e^{-\omega}.
    \label{eq:unbounded-weighted-chi-square}
\end{equation}
Indeed, for $0<s<(2\gamma_{+,d})^{-1}$, the logarithm of the centered moment-generating function is at most $s^2\sum_j(\gamma_j^\star)^2/(1-2s\gamma_{+,d})$; Chernoff's inequality gives the displayed bound. Apply it to each row with $\omega=4\log d$ and combine it with the analogous bound for $\|\bz_i\|^2$. With conditional probability at least $1-2de^{-\omega}$, simultaneously for all rows,
\begin{equation}
 W_d^2\leq \left(1+2\sqrt{\omega/n}+2\omega/n\right) \left(s_d+\frac{2\sqrt{\omega\operatorname{tr}[(\Gamma^\star)^2]}}n +\frac{2\gamma_{+,d}\omega}n\right).
    \label{eq:unbounded-row-finite-bound}
\end{equation}
Here $W_d$ is defined in the next display. This conditional union bound and \eqref{eq:unbounded-tail-control} yield
\begin{equation}
    W_d:=\max_i\frac{\|E_d\bm f_i\|}{\sqrt{\gamma_i^\star}} \leq\alpha^{-1/2}+o_{\mathbb P}(1).
    \label{eq:unbounded-row-bound}
\end{equation}
No independence between different rows' events is required.

\paragraph{Including directions concentrated on a coordinate.} If $\max_i r_i^2\leq1/2$, then $D_\Gamma(\br)\geq g_{\br}/2$ and the normalized error is at most $\sqrt2T_d$. Otherwise, let $r_i^2=m>1/2$ and write $\br=r_i\bm f_i+\bv$, where $v_i=0$ and $R=\|\bv\|^2=1-m$. Put $g_{\bv}=\bv^\top\Gamma^\star\bv$. Since $1-v_j^2\geq m$ for $j\ne i$,
\begin{equation}
    D_\Gamma(\br)\geq\gamma_i^\star mR+mg_{\bv}.
\end{equation}
The diagonal of $E_d$ is zero, so
\begin{equation}
    |\br^\top E_d\br| \leq2W_d\sqrt{\gamma_i^\star mR}+T_d\sqrt{Rg_{\bv}}.
\end{equation}
Dividing by $\sqrt{D_\Gamma(\br)}$ bounds this by $2W_d+T_d$, because $R<m$. At coordinate directions the numerator is exactly zero. Thus the supremum is at most $2T_d+2W_d$. The preceding bounds prove \eqref{eq:unbounded-normalized-noise}, for example with $C_1=12C+2$.
\end{proof}

\subsubsection{Adding the signal and transferring the score gap}
\label{subsec:unbounded-score-gap}
The sample covariance decomposes exactly as
\begin{equation}
\begin{aligned}
 \widehat\Sigma-\Sigma_d &=(\Gamma^\star)^{1/2}(Z^\top Z/n-\id)(\Gamma^\star)^{1/2} +\sqrt{\beta_d}(\bm e\bm t^\top+\bm t\bm e^\top) +\beta_d(q_\lambda-1)\bm e\bm e^\top,\\
 q_\lambda&=\|\bm\lambda\|^2/n,\qquad \bm t=(\Gamma^\star)^{1/2}Z^\top\bm\lambda/n.
\end{aligned}
    \label{eq:unbounded-signal-decomposition}
\end{equation}
Here $q_\lambda\to1$, $\beta_d\to\beta$, and $\|\bm t\|^2\to\alpha^{-1}$ in probability. To see the last assertion, condition on $\Gamma^\star,\bm\lambda$: $\bm t$ is Gaussian with covariance $q_\lambda\Gamma^\star/n$. Its squared norm has conditional mean $q_\lambda s_d\to\alpha^{-1}$ and conditional variance $2q_\lambda^2\operatorname{tr}[(\Gamma^\star)^2]/n^2\to0$.

For unit $\br$, the Frobenius-norm identity is $\|\operatorname{off}(\br\br^\top)\|_{\mathrm F}=\sqrt{1-I(\br)}$. Frobenius Cauchy--Schwarz and $D_\Gamma\geq a(1-I)$ give
\begin{equation}
\begin{aligned}
 \frac{|\br^\top\operatorname{off}(\bm e\bm t^\top+\bm t\bm e^\top)\br|} {\sqrt{D_\Gamma(\br)}}&\leq\frac{2\|\bm t\|}{\sqrt a},\\
 \frac{|\br^\top\operatorname{off}(\bm e\bm e^\top)\br|} {\sqrt{D_\Gamma(\br)}}&\leq\frac1{\sqrt a}.
\end{aligned}
\end{equation}
Combining this with \cref{lem:unbounded-normalized-noise} proves
\begin{equation}
\begin{aligned}
 &\sup_{\|\br\|=1,\ D_\Gamma>0} \frac{|\br^\top\operatorname{off}(\widehat\Sigma-\Sigma_d)\br|} {\sqrt{D_\Gamma(\br)}} \leq B_\alpha+o_{\mathbb P}(1),\\
 &B_\alpha:=C_1\left(\alpha^{-1/2} +a^{-1/2}\alpha^{-1}\right) +2\sqrt{\frac{\beta}{a\alpha}}.
\end{aligned}
    \label{eq:unbounded-full-error}
\end{equation}
In particular, $B_\alpha\to0$ as $\alpha\to\infty$, with no upper variance bound. The mask identities now imply, simultaneously over all non-coordinate directions,
\begin{equation}
    \frac{|\widehat{\mathcal C}(\br)-\mathcal C(\br)|}{\sqrt{H(\br)}} \leq\sqrt{\rho(1-\rho)}\,[B_\alpha+o_{\mathbb P}(1)].
    \label{eq:unbounded-cross-comparison}
\end{equation}

The relative Wishart bound \eqref{eq:sep-wishart-error} remains valid: conditionally on $\bu^\star,\Gamma^\star$, the samples are Gaussian with covariance $\Sigma_d$. Hence $(1-\varepsilon_d)H\leq\widehat H\leq(1+\varepsilon_d)H$, with $\varepsilon_d\to2/\sqrt\alpha+1/\alpha$. Define the actual uniform cross-coefficient error
\begin{equation}
 b_d:=\sup_{\|\br\|=1,\ H(\br)>0} \frac{|\widehat{\mathcal C}(\br)-\mathcal C(\br)|}{\sqrt{H(\br)}} \leq\sqrt{\rho(1-\rho)}B_\alpha+o_{\mathbb P}(1).
\end{equation}
The score comparison is therefore
\begin{equation}
\begin{aligned}
 \sqrt{\widehat S(\br)} &\leq\frac{\sqrt{S(\br)}+b_d}{\sqrt{1-\varepsilon_d}},\\
 \sqrt{\widehat S(\bm e)} &\geq\frac{\sqrt{S(\bm e)}-b_d}{\sqrt{1+\varepsilon_d}},
\end{aligned}
    \label{eq:unbounded-score-comparison}
\end{equation}
where the lower bound is used when its numerator is positive.

\begin{proof}[Proof of \cref{thm:unbounded-dynamic-recovery}]
The population localization bound \eqref{eq:sep-small-overlap-score} uses only $\Gamma^\star\succeq a\id$, so it remains valid. The teacher score limit \eqref{eq:sep-teacher-score} also remains valid: the law of large numbers gives $\bm e^\top\Gamma^\star\bm e\to1$, and $\sum_i\gamma_i^\star e_i^4\leq(\max_i e_i^2)\bm e^\top\Gamma^\star\bm e\to0$. Thus, with $g=\rho(1-\rho)\beta^2/(1+\beta(1-\rho))$ and $\eta$ from \eqref{eq:unbounded-overlap-guarantee}, with probability tending to one,
\begin{equation}
    \sup_{|\bm e^\top\br|^2\leq\eta}S(\br)\leq g/16, \qquad S(\bm e)\geq3g/4.
    \label{eq:unbounded-population-gap}
\end{equation}
For a finite sufficient bound, define
\begin{equation}
\begin{aligned}
 C_{a,\beta}&=C_1(1+a^{-1/2})+2\sqrt{\beta/a},\\
 A(a,\beta,\rho)&=\max\left\{100, \frac{1024C_{a,\beta}^2[1+\beta(1-\rho)]}{\beta^2}\right\}.
\end{aligned}
    \label{eq:unbounded-sample-bound}
\end{equation}
For $\alpha>A$, one has $\varepsilon_\alpha<1/4$ and $B_\alpha\leq C_{a,\beta}/\sqrt\alpha <\beta/[32\sqrt{1+\beta(1-\rho)}]$. Consequently, with probability tending to one, $\varepsilon_d<1/4$ and $b_d<\sqrt g/16$. Equations \eqref{eq:unbounded-score-comparison} and \eqref{eq:unbounded-population-gap} yield the strict empirical gap
\begin{equation}
 \sup_{|\bm e^\top\br|^2\leq\eta}\sqrt{\widehat S(\br)} \leq\frac{(1/4+1/16)\sqrt g}{\sqrt{3/4}} <\frac{(\sqrt{3/4}-1/16)\sqrt g}{\sqrt{5/4}} \leq\sqrt{\widehat S(\bm e)}.
\end{equation}
This excludes every low-overlap direction, including directions chosen adaptively from the data.

For attainment, $\max_i e_i^2<\eta$ with probability tending to one. The compact set $\{\|\br\|=1:|\bm e^\top\br|^2\geq\eta\}$ then contains no coordinate direction. On this set, $D_\Gamma>0$, and the relative covariance event makes $\widehat H>0$. The continuous score attains its maximum there. The strict empirical gap makes this a global maximum with positive score. Its radial optimizer $Q_*=[\widehat{\mathcal C}]_+/\widehat H$ is finite and positive by \eqref{eq:sep-empirical-score}. Coordinate directions and the zero weight have zero centered improvement and cannot minimize globally. This proves existence and the overlap assertion for every global minimizer.
\end{proof}

The sufficient sample ratio in \eqref{eq:unbounded-sample-bound} has, up to universal multiplicative constants, the dependence $A(a,\beta,\rho)\asymp 1+(1+\beta)[1+\beta(1-\rho)]/(a\beta^2)$, since $a\leq1$. Holding the other parameters fixed, this bound scales as $\beta^{-2}$ for weak signal and as $a^{-1}$ when the lower variance bound vanishes; it remains bounded as $\beta\to\infty$. These describe the chosen sufficient bound, not a sharp recovery threshold. Its conservatism reflects uniform control over all candidate directions, lower bounds on the masking penalty, and the numerical margins used to preserve the score gap. The variance law's second moment enters the vanishing remainder terms rather than this leading bound; the dimension needed for the high-probability guarantee can therefore depend on the law.
\paragraph{One fixed unbounded example}
\label{subsec:unbounded-example}
Take
\begin{equation}
    \gamma=(1+\xi_{\rm exp})/2,\qquad \xi_{\rm exp}\sim\operatorname{Exp}(1).
    \label{eq:unbounded-example-law}
\end{equation}
Then $a=1/2$, $\E_\gamma\gamma=1$, $\kappa_2=5/4$, and $\kappa_3=2$. This one variance law has unbounded support and satisfies all assumptions of \cref{thm:unbounded-dynamic-recovery}. For every fixed $\beta>0$ and $\rho\in(0,1)$, dynamic linear masking recovers at every sufficiently large fixed $\alpha$, whereas PCA and full linear reconstruction fail at every fixed $\alpha>0$. The required sufficient $\alpha$ may depend on $\beta$ and $\rho$; it is not uniform as $\beta\downarrow0$.

\newpage

%% file: appendices/appendix_bayes_amp.tex
\subsection{Explicit heteroskedastic bipartite AMP}
\label{subsec:BO-AMP}

We specialize the bipartite Low-RAMP equations (90)--(97) of \citet{lesieur2017constrained}, Section~II.4, using the sample-first observations $x_{\mu i}$, $d$ coordinate variables, and $n$ sample variables with $v_\mu^\star=\lambda_\mu$. The variance labels $\gamma_i^\star$ are held fixed and known in the matched benchmark. Define the coordinate-by-sample coupling and likelihood curvature
\begin{equation}
 J_{i\mu}:=\frac{\sqrt\beta\,x_{\mu i}}{\sqrt d\,\gamma_i^\star}, \qquad T_i:=\frac{\beta}{d\gamma_i^\star}.
 \label{eq:BO-AMP-couplings}
\end{equation}
The Gaussian likelihood contributes exactly $J_{i\mu}u_iv_\mu-T_i u_i^2v_\mu^2/2$ to the log posterior. In the score notation of \eqref{eq:BO-score-Gaussian} and \eqref{eq:BO-R-Gaussian}, $S_{i\mu}/\sqrt d=J_{i\mu}$ and $R_{i\mu}/d=J_{i\mu}^2-T_i$. These identities fix both the normalization and the signs of the precision updates below.

Let $\widehat u_i,\widehat v_\mu$ denote AMP means and $c_{u,i},c_{v,\mu}$ their scalar posterior variances. A sweep first updates the coordinate side at time $t$, then the sample side to time $t+1$. Given $(\widehat v^t,c_v^t,\widehat u^{t-1})$, set
\begin{equation}
\begin{aligned}
 B_{U,i}^t &=\sum_{\mu=1}^nJ_{i\mu}\widehat v_\mu^t -\widehat u_i^{t-1}\sum_{\mu=1}^nJ_{i\mu}^2c_{v,\mu}^t,\\
 A_{U,i}^t &=\sum_{\mu=1}^n\left[ T_i(\widehat v_\mu^t)^2 -(J_{i\mu}^2-T_i)c_{v,\mu}^t\right],\\
 \widehat u_i^t&=f_{\rm in}^U(A_{U,i}^t,B_{U,i}^t),\qquad c_{u,i}^t=\partial_B f_{\rm in}^U(A_{U,i}^t,B_{U,i}^t).
\end{aligned}
 \label{eq:BO-AMP-U}
\end{equation}
Using these newly computed coordinate estimates, set
\begin{equation}
\begin{aligned}
 B_{V,\mu}^t &=\sum_{i=1}^dJ_{i\mu}\widehat u_i^t -\widehat v_\mu^t\sum_{i=1}^dJ_{i\mu}^2c_{u,i}^t,\\
 A_{V,\mu}^t &=\sum_{i=1}^d\left[ T_i(\widehat u_i^t)^2 -(J_{i\mu}^2-T_i)c_{u,i}^t\right],\\
 \widehat v_\mu^{t+1}&=f_{\rm in}^V(A_{V,\mu}^t,B_{V,\mu}^t),\qquad c_{v,\mu}^{t+1}=\partial_B f_{\rm in}^V(A_{V,\mu}^t,B_{V,\mu}^t).
\end{aligned}
 \label{eq:BO-AMP-V}
\end{equation}
In particular, the two Onsager terms multiply $\widehat u^{t-1}$ and $\widehat v^t$, respectively. This is the alternating schedule of the cited equations. These are AMP approximations to the posterior; the exactness of the Gaussian likelihood expansion does not make its finite-dimensional marginals factorized.

\paragraph{Fisher-averaged implementation.} Under the regular state-evolution assumptions, the squared couplings in the summed corrections can be replaced by their null-channel means $\E_0[J_{i\mu}^2\mid\gamma_i^\star]=T_i$; the centered curvature terms $J_{i\mu}^2-T_i$ self-average. This is a site-dependent version of the Bayes simplification in Section~II.3 of \citet{lesieur2017constrained}. With $\alpha_d=n/d$, define
\begin{equation}
\begin{aligned}
 Q_v^t&=\frac1n\sum_\mu(\widehat v_\mu^t)^2, &V_v^t&=\frac1n\sum_\mu c_{v,\mu}^t,\\
 Q_{u,-1}^t&=\frac1d\sum_i\frac{(\widehat u_i^t)^2}{\gamma_i^\star}, &V_{u,-1}^t&=\frac1d\sum_i\frac{c_{u,i}^t}{\gamma_i^\star}.
\end{aligned}
 \label{eq:BO-AMP-empirical-moments}
\end{equation}
The resulting implementation uses
\begin{equation}
\begin{aligned}
 A_{U,i}^t&=\frac{\alpha_d\beta Q_v^t}{\gamma_i^\star}, &B_{U,i}^t&=\sum_\mu J_{i\mu}\widehat v_\mu^t -\frac{\alpha_d\beta V_v^t}{\gamma_i^\star}\widehat u_i^{t-1},\\
 A_{V,\mu}^t&=\beta Q_{u,-1}^t, &B_{V,\mu}^t&=\sum_i J_{i\mu}\widehat u_i^t -\beta V_{u,-1}^t\widehat v_\mu^t,
\end{aligned}
 \label{eq:BO-AMP-Fisher}
\end{equation}
followed by the same denoisers and the same update order. The factors $1/\gamma_i^\star$ remain inside the coordinate-side coefficients and sample-side averages; replacing them by a single mean variance changes the algorithm. The two finite-dimensional versions \eqref{eq:BO-AMP-U}--\eqref{eq:BO-AMP-V} and \eqref{eq:BO-AMP-Fisher} need not give identical iterates.

For $P_u=P_\lambda=\mathcal N(0,1)$, both denoisers are explicit:
\begin{equation}
\begin{aligned}
 \widehat u_i^t&=\frac{B_{U,i}^t}{1+A_{U,i}^t}, &c_{u,i}^t&=\frac1{1+A_{U,i}^t},\\
 \widehat v_\mu^{t+1}&=\frac{B_{V,\mu}^t}{1+A_{V,\mu}^t}, &c_{v,\mu}^{t+1}&=\frac1{1+A_{V,\mu}^t}.
\end{aligned}
 \label{eq:BO-AMP-Gaussian}
\end{equation}
The Fisher-averaged precisions are nonnegative, so these denominators are at least one. The unsimplified iteration requires $1+A>0$ for each Gaussian denoiser. For a Rademacher sample prior, only the sample-side denoiser changes to $\tanh B$ with variance $1-\tanh^2B$; it must not be used for Gaussian sample amplitudes.

\paragraph{Connection to the existing state evolution.} Let $M_v^t,Q_v^t$ be the limiting overlap and squared norm on the sample side, and $M_u^t(\gamma),Q_u^t(\gamma)$ their conditional coordinate-side counterparts. The formal Onsager-corrected scalar fields corresponding to the alternating AMP have
\begin{equation}
\begin{aligned}
 \mathcal A_U^t(\gamma)&=\frac{\alpha\beta Q_v^t}{\gamma}, &\mathcal B_U^t(\gamma) &=\frac{\alpha\beta M_v^t}{\gamma}u^\star +\sqrt{\mathcal A_U^t(\gamma)}\,Z_U,\\
 \mathcal A_V^t&=\beta\E_\gamma\frac{Q_u^t(\gamma)}\gamma, &\mathcal B_V^t &=\beta\E_\gamma\frac{M_u^t(\gamma)}\gamma\,v^\star +\sqrt{\mathcal A_V^t}\,Z_V,
\end{aligned}
 \label{eq:BO-AMP-SE-fields}
\end{equation}
where the Gaussian noises are independent of the respective scalar teachers. These are effective cavity laws, not independence assertions between finite-dimensional iterates and their training data. On a matched Nishimori trajectory, $Q_v^t=M_v^t$ and $Q_u^t(\gamma)=M_u^t(\gamma)$. Applying the scalar denoisers to \eqref{eq:BO-AMP-SE-fields} then gives exactly \eqref{eq:BO-Mu-gamma-Gaussian}--\eqref{eq:BO-Mv-update-Gaussian}. This verifies the normalization of the explicit AMP against the existing Bayes equations. It does not supply a general convergence theorem from an arbitrary initialization or a uniform state-evolution theorem for all unbounded variance laws.

The formulas condition on the realized variances. Replacing them by \eqref{eq:BO-empirical-gamma} defines a plug-in AMP, whose performance requires the stability analysis discussed in \cref{subsec:BO-gamma-identifiability}. When complete training samples are available, independent masking carries no additional information about the planted factors. The full-data Bayes/AMP benchmark therefore does not depend on the masking ratio $\rho$ or the number of masks $K$.